\documentclass[10pt,twocolumn,letterpaper]{article}

\usepackage[pagenumbers]{cvpr} %

\usepackage[dvipsnames]{xcolor}

\usepackage{fp} %
\usepackage{mathtools}
\usepackage{arydshln}  %
\usepackage{tikz}  %
\usepackage{makecell}  %
\usepackage{nicematrix}  %
\usepackage{rotating}  %
\usepackage{multirow}

\usepackage{algorithm2e}

\usepackage[export]{adjustbox} %

\usepackage[most]{tcolorbox}  %

\definecolor{mygray}{HTML}{E9E9E9}
\tcbset{
    examplebox/.style={
        colback=mygray,    %
        colframe=black,     %
        coltitle=white,     %
        fonttitle=\bfseries,%
        boxrule=0.25mm,      %
        rounded corners,    %
        width=\linewidth,   %
        top=5pt,           %
        bottom=5pt,         %
        left=5pt,         %
        right=5pt         %
    }
}

\DeclareRobustCommand{\circlelabel}[3]{%
  \tikz[baseline=(char.base)]{
    \node[shape=circle, draw=#2, fill=#1, text=black, inner sep=1pt] (char) {#3};
  }%
}

\DeclareRobustCommand{\rectlabel}[2]{%
  \tikz[baseline=(char.base)]{
    \node[shape=rectangle, fill=#1, text=black, inner sep=2pt] (char) {#2};
  }%
}

\definecolor{fig1blue}{HTML}{DAE8FC}
\definecolor{fig1bluedark}{HTML}{6C8EBF}
\definecolor{fig1orange}{HTML}{FFE6CC}
\definecolor{fig1orangedark}{HTML}{D79B00}
\definecolor{fig1purple}{HTML}{E1D5E7}
\definecolor{fig1purpledark}{HTML}{9673A6}
\definecolor{fig1green}{HTML}{D5E8D4}
\definecolor{fig1greendark}{HTML}{82B366}
\definecolor{fig1red}{HTML}{F8CECC}
\definecolor{fig1reddark}{HTML}{B85450}

  {\color{red}\everymath{\color{red}}\ignorespaces}%
  {\unskip\ignorespacesafterend}
  
\newcommand\lft{\mathopen{}\left}
\newcommand\rgt{\aftergroup\mathclose\aftergroup{\aftergroup}\right}

\DeclarePairedDelimiter{\abs}{\lvert}{\rvert}
\DeclarePairedDelimiter{\norm}{\lVert}{\rVert}

\DeclarePairedDelimiterX{\inner}[2]{\langle}{\rangle}{#1,#2}

\makeatletter
\let\oldabs\abs
\def\abs{\@ifstar{\oldabs}{\oldabs*}}
\makeatother

\definecolor{tabfirst}{rgb}{1, 0.7, 0.7} %
\definecolor{tabsecond}{rgb}{1, 0.85, 0.7} %
\definecolor{tabthird}{rgb}{1, 1, 0.7} %

\definecolor{figred}{HTML}{c8526a}
\definecolor{figblue}{HTML}{599ec4}
\definecolor{figgreen}{HTML}{7bae72}

\usepackage{amsthm}
\usepackage{thmtools, thm-restate}

\makeatletter
\def\adl@drawiv#1#2#3{%
        \hskip.5\tabcolsep
        \xleaders#3{#2.5\@tempdimb #1{1}#2.5\@tempdimb}%
                #2\z@ plus1fil minus1fil\relax
        \hskip.5\tabcolsep}
\newcommand{\cdashlinelr}[1]{%
  \noalign{\vskip\aboverulesep
           \global\let\@dashdrawstore\adl@draw
           \global\let\adl@draw\adl@drawiv}
  \cdashline{#1}
  \noalign{\global\let\adl@draw\@dashdrawstore
           \vskip\belowrulesep}}
\makeatother

\makeatletter
\renewcommand{\paragraph}{%
  \@startsection{paragraph}{4}%
  {\z@}{0ex \@plus 1ex \@minus .2ex}{-1em}%
  {\normalfont\normalsize\bfseries}%
}
\makeatother

\usepackage{amsmath,amsfonts,bm}

\def\eqref#1{equation~\ref{#1}}

\def\1{\bm{1}}

\def\rva{{A}}

\def\rvk{{K}}

\def\rvn{{N}}

\def\rvq{{Q}}
\def\rvr{{R}}
\def\rvs{{S}}

\def\rvx{{X}}
\def\rvy{{Y}}
\def\rvz{{Z}}

\def\vzero{{0}}

\def\vtheta{{\theta}}
\def\va{{a}}

\def\vf{{f}}

\def\vw{{w}}
\def\vx{{x}}
\def\vy{{y}}
\def\vz{{z}}

\def\mI{{\bm{I}}}

\DeclareMathAlphabet{\mathsfit}{\encodingdefault}{\sfdefault}{m}{sl}
\SetMathAlphabet{\mathsfit}{bold}{\encodingdefault}{\sfdefault}{bx}{n}

\newcommand{\E}{\mathbb{E}}

\newcommand{\R}{\mathbb{R}}

\usepackage{layouts}
\usepackage{colortbl}
\definecolor{cvprblue}{rgb}{0.21,0.49,0.74}
\usepackage[pagebackref,breaklinks,colorlinks,citecolor=cvprblue]{hyperref}

\def\paperID{11897} %
\def\confName{CVPR}
\def\confYear{2025}

\title{Proxies for Distortion and Consistency with \\ Applications for Real-World Image Restoration}

\author{Sean Man \qquad Guy Ohayon \qquad Ron Raphaeli \qquad Michael Elad \\
Technion -- Israel Institute of Technology\\
\texttt{\{sean.man,ohayonguy,ronraphaeli,elad\}@cs.technion.ac.il}
}

\begin{document}

\crefname{algocf}{alg.}{algs.}
\Crefname{algocf}{Algorithm}{Algorithms}

\maketitle

\begin{abstract}

Real-world image restoration deals with the recovery of images suffering from an unknown degradation.
This task is typically addressed while being given only degraded images, without their corresponding ground-truth versions.
In this hard setting, designing and evaluating restoration algorithms becomes highly challenging.
This paper offers a suite of tools that can serve both the design and assessment of real-world image restoration algorithms.
Our work starts by proposing a trained model that predicts the chain of degradations a given real-world measured input has gone through. 
We show how this estimator can be used to approximate the consistency -- the match between the measurements and any proposed recovered image.
We also use this estimator as a guiding force for the design of a simple and highly-effective plug-and-play real-world image restoration algorithm, leveraging a pre-trained diffusion-based image prior.
Furthermore, this work proposes no-reference proxy measures of MSE and LPIPS, which, without access to the ground-truth images, allow ranking of real-world image restoration algorithms according to their (approximate) MSE and LPIPS.
The proposed suite provides a versatile, first of its kind framework for evaluating and comparing blind image restoration algorithms in real-world scenarios.

\end{abstract}

\begin{figure}[tb]
    \centering
    \includegraphics[width=\linewidth]{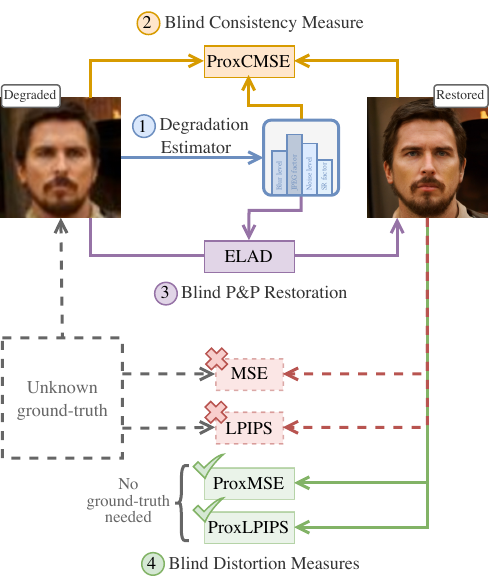}
    \caption{
    This work introduces several novel tools to help tackle the challenging task of real-world image restoration. We propose \circlelabel{fig1blue}{fig1bluedark}{1} an estimator that predicts the degradations a real-world corrupted measurement has gone through.
    Using this estimator, we \circlelabel{fig1orange}{fig1orangedark}{2} approximate the consistency of any reconstructed candidate with a given input measurement, and use such a measure to develop \circlelabel{fig1purple}{fig1purpledark}{3} a plug-and-play real-world restoration algorithm.
    Moreover, we propose \circlelabel{fig1green}{fig1greendark}{4} blind (no-reference) measures of distortion that mimic MSE and LPIPS, thus eliminating the need for ground-truth images when comparing the distortion of real-world image restoration algorithms.
    }
    \label{fig:hero}
\end{figure}

\section{Introduction}
\label{sec:intro}

Image restoration -- the task of reconstructing high-quality images from their degraded measurements (\eg, noisy, blurry) -- is one of the most extensively studied problems in imaging sciences.
The typical foundation for designing and testing image restoration algorithms is the assumption that the degradation process (forward model) is known and follows a specific mathematical form.
In image denoising, for example, the noise contamination is commonly assumed to be additive, white, and Gaussian with a known variance \cite{GaussianDenoiserResidual2017zhang,ImageDenoisingDeep2023elada,HighPerceptualQuality2021ohayona,StochasticImageDenoising2021kawarb,RestormerEfficientTransformer2022zamir,SwinIRImageRestoration2021liangb,ImageDenoisingSparse2007dabova,WeightedNuclearNorm2014gu,ResidualDenseUNet2021gurrola-ramos,ImageDenoisingSparse2006elad,TrainableNonlinearReaction2017chen}.
This enables synthesizing corresponding pairs of clean images and degraded measurements, which can then be used for supervised learning (\eg, minimizing the MSE~\cite{MeanSquaredError2009wang}) and/or evaluation (\eg, computing distortion measures such as PSNR, SSIM~\cite{ImageQualityAssessment2004wang}, and LPIPS~\cite{UnreasonableEffectivenessDeep2018zhanga}).
The mathematical expression of the degradation process can also be leveraged to design plug-and-play image restoration algorithms \cite{DenoisingDiffusionRestoration2022kawara,DiffusionPosteriorSampling2022chunga,VariationalPerspectiveSolving2023mardani,DenoisingDiffusionModels2023zhua,StochasticImageDenoising2021kawarb,LittleEngineThat2017romanoa,PlugPlayPriorsModel2013venkatakrishnan,ZeroShotImageRestoration2022wangb}.
Indeed, the gradient of the log-likelihood function $\log{p_{\rvy|\rvx}}(\vy|\vx)$ w.r.t. $x$, where $X$ and $Y$ are random vectors denoting the ground-truth image and the degraded measurement, respectively, is often used for guiding the sampling from the posterior distribution \cite{DenoisingDiffusionRestoration2022kawara,DiffusionPosteriorSampling2022chunga,VariationalPerspectiveSolving2023mardani,DenoisingDiffusionModels2023zhua,StochasticImageDenoising2021kawarb,ZeroShotImageRestoration2022wangb}.

As opposed to the above,  the degradation process is typically unknown  in real-world scenarios \cite{Ji_2020_CVPR_Workshops,MetricLearningBased2022mou,RealESRGANTrainingRealWorld2021wangb,DesigningPracticalDegradation2021zhang}, making it significantly more challenging to design and evaluate image restoration algorithms for everyday applications. This paper offers several focused contributions that aim to foster progress in such real-world scenarios, circumventing the above described difficulties. We start by introducing a highly-effective learnable model that predicts the chain of degradations a given real-world input has gone through~(\Cref{fig:hero}). As we show, this estimator can then be leveraged in three important ways: \emph{(i)} We rely on it for the development of a \emph{blind} consistency measure -- an approximation of the likelihood of any reconstructed candidate image given a real-world degraded input;
\emph{(ii)} Using this approximate-likelihood measure, we considerably improve the distortion and consistency of DifFace~\cite{DifFaceBlindFace2023yue}, by guiding the generated images to remain consistent with the inputs in a plug-and-play manner;
\emph{(iii)} Finally, we use the degradation estimator for synthesizing pairs of clean images and their corresponding degraded measurements, in a way that mimics any existing dataset of real-world inputs. %

Beyond all the above, we also propose \emph{blind} proxy measures of MSE and LPIPS, as illustrated in~\Cref{fig:hero}.
With no access to the ground-truth images of the given real-world degraded inputs, these measures can be used to rank real-world image restoration algorithms according to their (approximate) MSE and LPIPS.

This paper is organized as follows:
We begin with the problem's formulation and notations in \Cref{sec:background}.
Our degradation estimator and likelihood approximation approach are described in~\Cref{sec:ela}, along with their derivative applications.
In~\Cref{sec:metric} we introduce our blind (no-reference) proxy distortion measures, and use them to evaluate real-world restoration methods on blind face restoration datasets.
We conclude with a discussion on related work in~\Cref{sec:related-work}, and the limitations and possible future directions in \Cref{sec:conclusion}.

\section{Preliminaries}
\label{sec:background}

As in~\citep{PerceptionDistortionTradeoff2018blaua}, we consider a natural image $\vx$ to be a realization of a random vector $\rvx$ with probability density function $p_{\rvx}$.
A degraded measurement $\vy$ is also a realization of a random vector $\rvy$, which is related to $\rvx$ via the conditional probability density function $p_{\rvy|\rvx}$.
Generally speaking, an image restoration algorithm is some estimator $\smash{\hat{\rvx}}$ that generates reconstructions according to $p_{\hat{\rvx}|\rvy}$, where $\smash{\rvx\rightarrow\rvy\rightarrow\hat{\rvx}}$ ($\rvx$ and $\smash{\hat{\rvx}}$ are statistically independent given $\rvy$).

The degradation process is typically unknown in real-world scenarios.
Namely, a given real-world  measurement can result from many different variations over the possible degradation processes.
We denote by $A$ the random vector that represents all the possible degradation processes that $y$ could have gone through, and by $p_{A}$ its density.
Thus, we can generally say that the degraded measurement $y$ is sampled from the conditional density $p_{Y|X,A}(\cdot|x,a)$, where $x$ and $a$ are realizations of the random vectors $X$ and $A$, respectively, which are sampled jointly and independently.
In blind face image restoration (BFR), the degradation process is commonly modeled~\cite{RealWorldBlindFace2021wang,RobustBlindFace2022zhou,DifFaceBlindFace2023yue,DiffBIRBlindImage2024lin,VQFRBlindFace2022gu} as
\begin{align}
\label{eq:bfr_model}
    Y=\text{JPEG}_{\rvq} \lft( \lft( \rvk * \vx \rgt) {\downarrow_{\rvs}} + \rvn \rgt),
\end{align}
where $A=(\rvk,\rvs,\sigma_{\rvn},\rvq)^{\top}$ is a random vector: $\rvk$ is a Gaussian blur kernel of width $\sigma_{\rvk}$, $\downarrow_{\rvs}$ denotes the bilinear down-sampling operator with scale-factor $\rvs$, $\rvn\sim\mathcal{N}(0,\sigma_{\rvn}^2 I)$ is a white Gaussian noise with standard deviation $\sigma_{\rvn}$, and $\text{JPEG}_{\rvq}$ is JPEG compression-decompression algorithm with quality-factor $\rvq$.
Throughout the paper, we consider the BFR task for demonstrating our proposed tools, where  $\sigma_{\rvk}, \rvs, \sigma_{\rvn}$ and $\rvq$ in~\cref{eq:bfr_model} are sampled independently and uniformly from $[0.1,15], [1, 32], [0, 20/255], [30,100]$, respectively (as in DifFace~\citep{DifFaceBlindFace2023yue}), unless mentioned otherwise.

\section{Empirical likelihood approximation (ELA)}
\label{sec:ela}

The conditional probability density function $p_{\rvy|\rvx}$ leads to the definition of the \emph{log-likelihood function} $\ell(\vy,\vx) = \log p_{\rvy|\rvx}(\vy|\vx)$, where a larger value of $\ell(\vy,\vx)$ implies that the image $\vx$ is more \emph{consistent} with the measurement $\vy$.
Such a likelihood function has several practical benefits in image restoration.
For instance, it facilitates plug-and-play image restoration, where the outputs generated by a pre-trained diffusion model are enforced to be consistent with the given measurement~\cite{DenoisingDiffusionRestoration2022kawara,DiffusionPosteriorSampling2022chunga,VariationalPerspectiveSolving2023mardani,DenoisingDiffusionModels2023zhua,StochasticImageDenoising2021kawarb,LittleEngineThat2017romanoa,PlugPlayPriorsModel2013venkatakrishnan,ZeroShotImageRestoration2022wangb}.
However, $p_{Y|X}$ is typically unknown in real-world scenarios, so $\ell(\vy,\vx)$ cannot be used directly.
To overcome such a limitation, we introduce \emph{Empirical Likelihood Approximation} (ELA), an approach to approximate $\ell(\vy,\vx)$ based on a novel degradation estimator.
Utilizing the proposed likelihood function, we present a blind consistency measure for real-world problems, followed by a plug-and-play restoration method for blind face restoration that harnesses a pre-trained diffusion prior.

\subsection{Approximate log-likelihood}

Many of the algorithms that operate in non-blind settings (where $p_{Y|X}$ is known) focus on deterministic degradation operators with additive Gaussian noise.
This leads to the familiar log-likelihood expression
\begin{align}
\label{eq:gauss_likelihood}
    \ell(\vy,\vx) \propto -\norm{\vy - h(\vx)}_2^2,
\end{align}
where $h$ is some deterministic degradation operator (\eg, bi-cubic down-sampling).
Suppose that $Y=y$ is coupled with an appropriate $A=a$, namely $y$ is the degraded measurement that resulted from the degradation corresponding to $a$.
We approximate the log-likelihood function corresponding to such a degradation by
\begin{align}
    \ell(\vy,\vx) \approx -\norm{y-\mu_{\rvy}(\vx,a)}_{2}^{2},~\label{eq:nll-approx}
\end{align}
where $\mu_{Y}(\vx,a)\coloneqq \mathbb{E}[Y|X=\vx,A=a]$,
and the expectation is taken w.r.t. the remaining stochastic portions of the degradation $A=a$ (such as noise).
To design an appropriate log-likelihood function for an unknown degradation, 
we train a model to predict $A$ from $Y$, and then use the result in~\cref{eq:nll-approx}. Namely, 
\begin{align}
    \ell(\vy,\vx) \approx -\norm{y-\mu_{\rvy}(x,a_{\theta}(y))}_{2}^{2}\label{eq:nll-final-approx}
\end{align}
is our final approximation of the log-likelihood, where $a_{\theta}$ is a trained model which estimates $A$ from $Y$.

\subsection{Degradation estimator}
\label{sec:deg_est}

Given a measurement $\vy$, the task of the estimator $\va_\vtheta(\vy)$ is to predict the degradation that produced $\vy$, \ie $\va_\vtheta(\vy) \approx \va$.
For the task of blind face restoration, we train a regression model to estimate the degradation parameters described in \cref{eq:bfr_model}.
To train our model, we incorporate the standard squared error regression loss $\mathcal{L}_{\text{Main}}=\norm{a-a_{\theta}(y)}_2^2$.
Moreover, to ensure that the predicted solution aligns with $\mu_{Y}(x,a)$, we also add the regularization term $\mathcal{L}_{\text{Reg.}}=\norm{\mu_{\rvy}(\vx, \va) - \mu_{\rvy}(\vx, \va_\theta(\vy))}_2^2$.
Our final training loss is given by $0.25\smash{\mathcal{L}_{\text{Main}}+\mathcal{L}_{\text{Reg.}}}$.
\Cref{fig:deg_est_acc} demonstrates the high accuracy of the trained estimator. We defer the architecture and optimization details to \Cref{sec:imp_supp}.

While in this paper we focus on the degradation estimator's benefits for likelihood approximation, one may leverage this estimator for  additional applications.
As an example, we estimate the degradations' parameters in the real-world datasets LFW-Test~\cite{LabeledFacesWild2007huang}, WebPhoto-Test~\cite{RealWorldBlindFace2021wang}, and WIDER-Test~\cite{WIDERFACEFace2016yang,RobustBlindFace2022zhou} and approximate their distribution using Kernel Density Estimation (KDE).
Then, we synthesize degraded measurements from CelebA-Test~\cite{liu2015faceattributes,ProgressiveGrowingGANs2018karrasa,RobustBlindFace2022zhou}, a dataset of clean images, by sampling $A$ according to the predicted distribution corresponding to each real-world dataset.
The resulting synthetic datasets mimic the real-world ones, and therefore enable a more appropriate evaluation of real-world image restoration algorithms. This stands in contrast with prior works that sample the parameters of $A$ uniformly~\cite{RealWorldBlindFace2021wang,RobustBlindFace2022zhou,DifFaceBlindFace2023yue,DiffBIRBlindImage2024lin,VQFRBlindFace2022gu,RestoreFormerRealWorldBlind2023wang}.
Hence, we utilize our synthetic datasets (as well as uniform sampling) to evaluate the proposed tools presented in the next sections.
In~\Cref{fig:deg_est_real_world} we plot the approximated distributions of the real-world degradations' parameters, and in~\Cref{sec:datasets_supp} we compare the results visually.

\begin{figure}[tb]
    \centering
    \includegraphics[width=\linewidth]{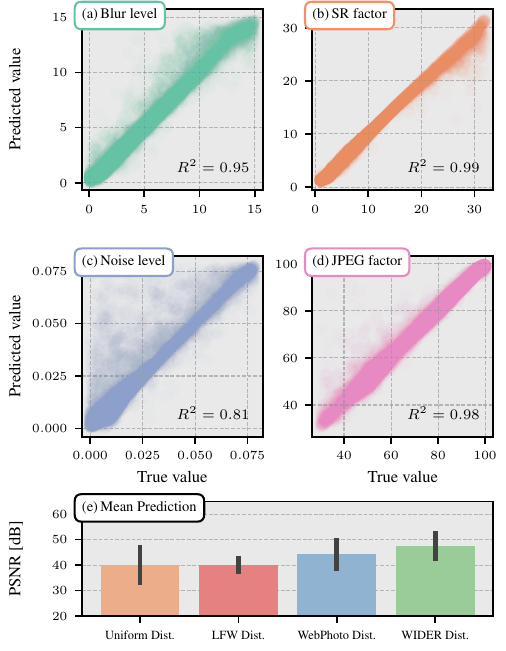}
    \caption{\textbf{Degradation estimator accuracy.}
    We test our degradation estimator on synthetic CelebA-Test datasets (\Cref{sec:deg_est}).
    (a-d) Scatter plots and $R^2$ scores of the true \vs the predicted values for each type of operator in \cref{eq:bfr_model}.
    (e) The mean and standard deviation of the PSNR between $\mu_{Y}(\vx, \va)$ and $\mu_{Y}(\vx, \va_\vtheta(\vy))$.
    The estimator demonstrates high prediction accuracy, as reflected by the high PSNR and $R^2$ scores.
    }
    \label{fig:deg_est_acc}
\end{figure}

\begin{figure}[tbh]
    \centering
    \includegraphics[width=\linewidth]{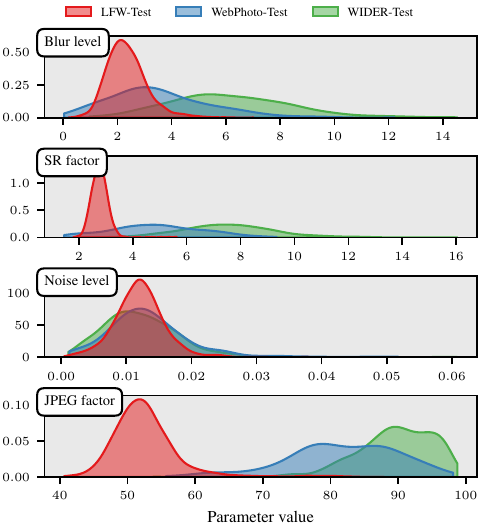}
    \caption{\textbf{Degradations in real-world BFR datasets.}
    Using our degradation estimator, we reveal the distribution of the degradations presented in real-world datasets.
    This information can be utilized to better analyze such datasets, as well as mimicking them.}
    \label{fig:deg_est_real_world}
\end{figure}

\subsection{Blind measure of consistency}
\label{sec:consistency}

An immediate application of \cref{eq:nll-approx} is a \emph{blind} consistency measure.
In the case where the true degradation process $\va$ is known, we define
\begin{align}
\label{eq:cmse}
    \text{CMSE}(\hat{\rvx})\coloneqq \E_{(\hat\vx, \vy) \sim p_{\hat\rvx,\rvy}} \lft[ \norm{\vy - \mu_{Y}(\hat{\vx}, \va)}_2^2 \rgt].
\end{align}
This is a direct generalization of the commonly used consistency measure, as practiced in~\cite{SRFlowLearningSuperResolution2020lugmayr,NTIRE2021Learning2021lugmayra,NTIRE2022Challenge2022lugmayr,ReasonsSuperiorityStochastic2023ohayon,HighPerceptualQualityJPEG2023man,ExplorableSuperResolution2020bahata,StochasticImageDenoising2021kawarb,HighPerceptualQuality2021ohayona}. 
In real-world scenarios where $a$ is unknown, we use the approximation of $\ell(\vy,\vx)$ in~\cref{eq:nll-final-approx} and define
\begin{eqnarray}
    \text{ProxCMSE}(\hat{\rvx})\coloneqq  \E_{(\hat\vx, \vy) \sim p_{\hat\rvx,\rvy}} \lft[ \norm{\vy - \mu_{Y}(\hat{\vx}, \va_\theta(\vy))}_2^2 \rgt].
\end{eqnarray}
Namely, if $\va_\theta(y)$ is trained appropriately, $\text{ProxCMSE}(\hat{\rvx})$ would be an approximation of $\text{CMSE}(\hat{\rvx})$.
In \Cref{fig:bfr_crmse} we show that ProxCMSE aligns well with the true CMSE on the datasets we synthesized in \Cref{sec:deg_est}. In \Cref{tab:elad,tab:bfr_rmse_real} we use ProxCMSE as another measure to evaluate algorithms on real-world datasets. %

\begin{figure}[tb]
    \centering
    \includegraphics[width=\linewidth]{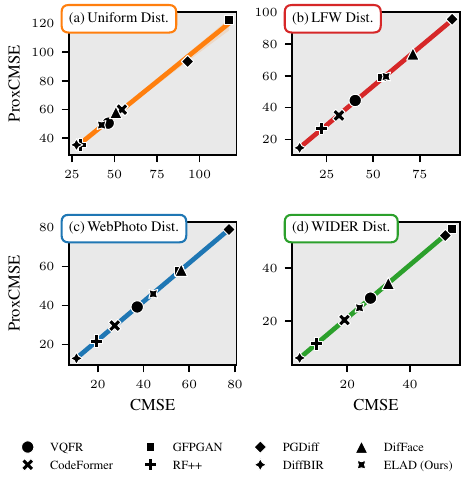}
    \caption{\textbf{Proxy consistency measure.}
    Each plot shows the CMSE versus ProxCMSE, evaluated on synthetic CelebA-Test datasets (\Cref{sec:deg_est}).
    The strong alignment of the two suggests that ProxCMSE is a trustworthy approximation for the CMSE when the degradation process is unknown.
    }
    \label{fig:bfr_crmse}
\end{figure}

\subsection{Plug-and-play real-world image restoration}
\label{sec:elad}

Another application of our degradation estimator and the derived likelihood approximation is the ability to construct plug-and-play blind restoration algorithms that harness pre-trained generative image priors, such as diffusion models.
To achieve valid and perceptually pleasing restorations in blind face restoration, prior plug-and-play diffusion methods use different types of heuristics.
PGDiff~\cite{PGDiffGuidingDiffusion2023yang} uses the guidance term $\nabla_{\hat\vx}\norm{\hat\vx - f(\vy)}$ at each step, where $f(\vy)$ is an MMSE predictor of $\rvx$ from $\rvy=\vy$.
DifFace~\cite{DifFaceBlindFace2023yue} starts from an intermediate diffusion timestep, initialized by $f(\vy)$ with no guidance term.
As seen in \Cref{fig:elad_consistency}, these heuristics tend to produce inconsistent outputs.

Here, we present a new approach: \emph{Empirical Likelihood Approximation with Diffusion prior} (ELAD) (\Cref{alg:elad_short}).
Our algorithm extends DifFace by guiding the diffusion process to take steps in the direction of the approximated log-likelihood score.
Thus, our approach is similar to DPS~\cite{DiffusionPosteriorSampling2022chunga}, where we use our approximation of the likelihood (\cref{eq:nll-final-approx}) instead of the true one, $\log{p_{Y|X}(y|x)}$ (which is unavailable in our setting). As seen in \Cref{fig:elad_consistency}, 
using such guidance promotes consistency with the measurement at each step without sacrificing image quality. 
We perform $100$ denoising steps like DifFace, starting from an intermediate timestep for accelerated sampling.
While we adapt DifFace's diffusion prior, note that our method is not limited only for such a prior.
Namely, the results could be further improved by using better priors.
See \Cref{sec:imp_supp} for more implementation details.

We defer the quantitative evaluation of ELAD on synthetic and real-world datasets to \Cref{sec:metric-exp}, immediately after we introduce our new proxy distortion measures. A qualitative comparison between ELAD and leading end-to-end methods on real-world datasets is presented in \Cref{fig:images}.

\begingroup
\newcolumntype{M}[1]{>{\centering\arraybackslash}m{#1}}
\newcommand{\vcentered}[1]{\begin{tabular}{@{}l@{}} #1 \end{tabular}}
\setlength{\tabcolsep}{0pt} %
\renewcommand{\arraystretch}{0} %

\newcommand{\centered}[1]{\begin{tabular}{l} #1 \end{tabular}}

\newcommand{\addimgcol}[8][1]{
	\centered{
	   	\begin{tikzpicture}[
	   		baseline=-2.45,
	   		spy using outlines={magnification=#3, circle, height=#8, width=#8, yellow, every spy on node/.append style={thick}, connect spies},
	   		]
			\node[inner sep=0pt]{\scalebox{#1}[1]{\adjincludegraphics[width=\hero, trim={#4}, clip]{images/ours/#7-bsds/compressed/#2}}};
			\spy on (#5) in node at (#6);
		\end{tikzpicture}}&
		
		\centered{
		\begin{tikzpicture}[
	   		baseline=-2.45,
	   		spy using outlines={magnification=#3, circle, height=#8, width=#8, yellow, every spy on node/.append style={thick}, connect spies},
	   		]
			\node[inner sep=0pt]{\scalebox{#1}[1]{\adjincludegraphics[width=\hero, trim={#4}, clip]{images/qgac/#7-bsds/#2}}};
			\spy on (#5) in node at (#6);
		\end{tikzpicture}}&
		
		\centered{
		\begin{tikzpicture}[
	   		baseline=-2.45,
	   		spy using outlines={magnification=#3, circle, height=#8, width=#8, yellow, every spy on node/.append style={thick}, connect spies},
	   		]
			\node[inner sep=0pt]{\scalebox{#1}[1]{\adjincludegraphics[width=\hero, trim={#4}, clip]{images/qgac-gan/#7-bsds/#2}}};
			\spy on (#5) in node at (#6);
		\end{tikzpicture}}&
		
		\centered{
	   	\begin{tikzpicture}[
	   		baseline=-2.45,
	   		spy using outlines={magnification=#3, circle, height=#8, width=#8, yellow, every spy on node/.append style={thick}, connect spies},
	   		]
			\node[inner sep=0pt]{\scalebox{#1}[1]{\adjincludegraphics[width=\hero, trim={#4}, clip]{images/bahat/#7-bsds/fake_0/#2}}};
			\spy on (#5) in node at (#6);
		\end{tikzpicture}}&

		\centered{
		\begin{tikzpicture}[
	   		baseline=-2.45,
	   		spy using outlines={magnification=#3, circle, height=#8, width=#8, yellow, every spy on node/.append style={thick}, connect spies},
	   		]
			\node[inner sep=0pt]{\scalebox{#1}[1]{\adjincludegraphics[width=\hero, trim={#4}, clip]{images/ours-p/#7-bsds/fake_0/#2}}};
			\spy on (#5) in node at (#6);
		\end{tikzpicture}}&
		
		\centered{
		\begin{tikzpicture}[
	   		baseline=-2.45,
	   		spy using outlines={magnification=#3, circle, height=#8, width=#8, yellow, every spy on node/.append style={thick}, connect spies},
	   		]
			\node[inner sep=0pt]{\scalebox{#1}[1]{\adjincludegraphics[width=\hero, trim={#4}, clip]{images/ours/#7-bsds/real/#2}}};
			\spy on (#5) in node at (#6);
		\end{tikzpicture}}
		
	    \\
}

\begin{figure*}[tb]
    \centering
    \begin{tabular}{M{0.2\linewidth} M{0.2\linewidth} M{0.2\linewidth} M{0.2\linewidth} M{0.2\linewidth}}
    
        \footnotesize{Degraded} &
        \footnotesize{Ground-Truth} &
        \footnotesize{ELAD (Ours)} &
        \footnotesize{DifFace~\cite{DifFaceBlindFace2023yue}} &
        \footnotesize{PGDiff~\cite{PGDiffGuidingDiffusion2023yang}}
        \\
	    
        \rule{0pt}{0.8ex}\\

        \multicolumn{5}{c}{\centered{\includegraphics[width=\linewidth]{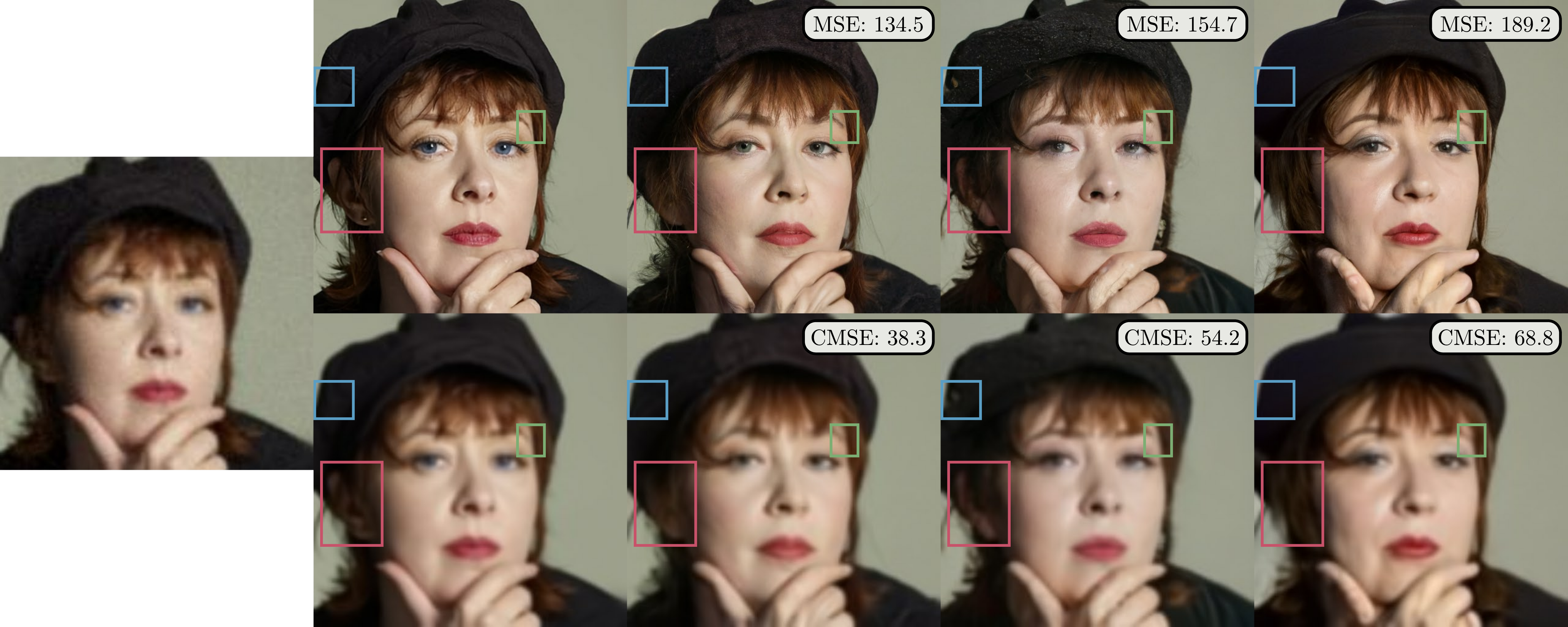}}} \\
        
        \rule{0pt}{0.8ex}\\

        &
        \footnotesize{$\mu_{Y}(\vx,a)$} &
        \footnotesize{$\mu_{Y}(\hat\vx_{\text{ELAD}},a)$} &
        \footnotesize{$\mu_{Y}(\hat\vx_{\text{DifFace}},a)$} &
        \footnotesize{$\mu_{Y}(\hat\vx_{\text{PGDiff}},a)$}
        
    \end{tabular}
    \caption{\textbf{Consistency of P\&P real-world restoration methods.}
    Left column: A synthetic degraded example from CelebA-Test.
    First row: The ground-truth image alongside the restorations of ELAD (our method), DifFace, and PGDiff.
    Second row: The mean of the likelihood defined using the true degradation process $\va$ and the corresponding image from the row above.
    For reference, we also report the MSE and CMSE of each method (\cref{eq:cmse}).
    As shown, ELAD attains better consistency with the measurement.
    PGDiff, for example, does not restore the ear (\rectlabel{fig1red}{red} rectangle), which is visible both in the degraded image and in $\mu_{Y}(\vx,a)$.
    Similarly, DifFace erases part of the brow (\rectlabel{fig1green}{green} rectangle) and creates inconsistent artifacts in the hat (\rectlabel{fig1blue}{blue} rectangle).
    In contrast, ELAD performs better by guiding the diffusion process to produce consistent reconstructions.
    As a positive side effect, achieving better consistency also leads to better distortion, which is apparent visually in the top row, and can be confirmed by the reported MSE and by~\Cref{tab:elad}.
    }
    \label{fig:elad_consistency}
\end{figure*}
\endgroup

\begingroup

\RestyleAlgo{ruled}

\SetKwComment{Comment}{// }{}
\newcommand\mycommfont[1]{\footnotesize\ttfamily\textcolor{cvprblue}{#1}}
\SetCommentSty{mycommfont}

\begin{algorithm}[hbt!]
\DontPrintSemicolon
\LinesNumbered
\caption{ELAD - Blind Restoration Diffusion Sampler (a full version is given in \cref{sec:imp_supp})}
\label{alg:elad_short}
\KwIn{measurement $\vy$, degradation estimator $\va_\theta$, MMSE regressor $f$, start time $T_0 \leq T$, step~sizes $\{\lambda_t\}_{t=1}^{T_0}$}
\KwOut{a restored image $\vx_0$}
$\vx_{T_0} = \text{AddNoise}(f(\vy), T_0)$\;
$\hat\va = \va_\theta(\vy)$ \Comment{predict degradation}
\For{$t=T_0$ \KwTo $1$}{
    $\hat\vx_0^t=\text{DenoiseStep}(\vx_t)$ \;
    \Comment{compute score likelihood}
    $g = \nabla_{\hat\vx_t} \norm{\vy - \mu(\hat\vx_0^t, \hat\va)}_2^2$ \;
    \Comment{perform likelihood step}
    $\hat\vx_0^t = \hat\vx_0^t - \lambda_t \cdot g.\text{clamp}(-1,1) $ \;
    $\vx_{t-1}=\text{DDIMStep}(\vx_t, \hat\vx_0^t)$ \;
}
\end{algorithm}

\endgroup

\begingroup
\newcolumntype{M}[1]{>{\centering\arraybackslash}m{#1}}
\newcommand{\vcentered}[1]{\begin{tabular}{@{}l@{}} #1 \end{tabular}}
\setlength{\tabcolsep}{0pt} %
\renewcommand{\arraystretch}{0} %

\def\columns{12}
\def\totalwidth{1.0}

\FPeval{\colwidth}{clip(\totalwidth/\columns)}
\FPeval{\doublecolwidth}{clip(2*\totalwidth/\columns)}
\FPeval{\imgwidth}{\totalwidth/\columns *\columns}

\newcommand{\centered}[1]{\begin{tabular}{l} #1 \end{tabular}}

\newcommand{\addimgcol}[8][1]{
	\centered{
	   	\begin{tikzpicture}[
	   		baseline=-2.45,
	   		spy using outlines={magnification=#3, circle, height=#8, width=#8, yellow, every spy on node/.append style={thick}, connect spies},
	   		]
			\node[inner sep=0pt]{\scalebox{#1}[1]{\adjincludegraphics[width=\hero, trim={#4}, clip]{images/ours/#7-bsds/compressed/#2}}};
			\spy on (#5) in node at (#6);
		\end{tikzpicture}}&
		
		\centered{
		\begin{tikzpicture}[
	   		baseline=-2.45,
	   		spy using outlines={magnification=#3, circle, height=#8, width=#8, yellow, every spy on node/.append style={thick}, connect spies},
	   		]
			\node[inner sep=0pt]{\scalebox{#1}[1]{\adjincludegraphics[width=\hero, trim={#4}, clip]{images/qgac/#7-bsds/#2}}};
			\spy on (#5) in node at (#6);
		\end{tikzpicture}}&
		
		\centered{
		\begin{tikzpicture}[
	   		baseline=-2.45,
	   		spy using outlines={magnification=#3, circle, height=#8, width=#8, yellow, every spy on node/.append style={thick}, connect spies},
	   		]
			\node[inner sep=0pt]{\scalebox{#1}[1]{\adjincludegraphics[width=\hero, trim={#4}, clip]{images/qgac-gan/#7-bsds/#2}}};
			\spy on (#5) in node at (#6);
		\end{tikzpicture}}&
		
		\centered{
	   	\begin{tikzpicture}[
	   		baseline=-2.45,
	   		spy using outlines={magnification=#3, circle, height=#8, width=#8, yellow, every spy on node/.append style={thick}, connect spies},
	   		]
			\node[inner sep=0pt]{\scalebox{#1}[1]{\adjincludegraphics[width=\hero, trim={#4}, clip]{images/bahat/#7-bsds/fake_0/#2}}};
			\spy on (#5) in node at (#6);
		\end{tikzpicture}}&

		\centered{
		\begin{tikzpicture}[
	   		baseline=-2.45,
	   		spy using outlines={magnification=#3, circle, height=#8, width=#8, yellow, every spy on node/.append style={thick}, connect spies},
	   		]
			\node[inner sep=0pt]{\scalebox{#1}[1]{\adjincludegraphics[width=\hero, trim={#4}, clip]{images/ours-p/#7-bsds/fake_0/#2}}};
			\spy on (#5) in node at (#6);
		\end{tikzpicture}}&
		
		\centered{
		\begin{tikzpicture}[
	   		baseline=-2.45,
	   		spy using outlines={magnification=#3, circle, height=#8, width=#8, yellow, every spy on node/.append style={thick}, connect spies},
	   		]
			\node[inner sep=0pt]{\scalebox{#1}[1]{\adjincludegraphics[width=\hero, trim={#4}, clip]{images/ours/#7-bsds/real/#2}}};
			\spy on (#5) in node at (#6);
		\end{tikzpicture}}
		
	    \\
}

\begin{figure*}[tb]
    \centering
    \begin{tabular}{cc cc cc cc cc cc}
    
	    \multicolumn{2}{c}{\footnotesize{Degraded}} &

	    \multicolumn{2}{c}{\footnotesize{RF++~\cite{RestoreFormerRealWorldBlind2023wang}}} &
	    \multicolumn{2}{c}{\footnotesize{CodeFormer~\cite{RobustBlindFace2022zhou}}} &
	    \multicolumn{2}{c}{\footnotesize{DiffBIR~\cite{DiffBIRBlindImage2024lin}}} &
	    \multicolumn{2}{c}{\footnotesize{PMRF~\cite{PosteriorMeanRectifiedFlow2024ohayon}}} &

	    \multicolumn{2}{c}{\footnotesize{ELAD (Ours)}}
	    \\
	    
        \rule{0pt}{0.8ex}\\

        \multicolumn{2}{c}{\centered{\includegraphics[width=\doublecolwidth\linewidth]{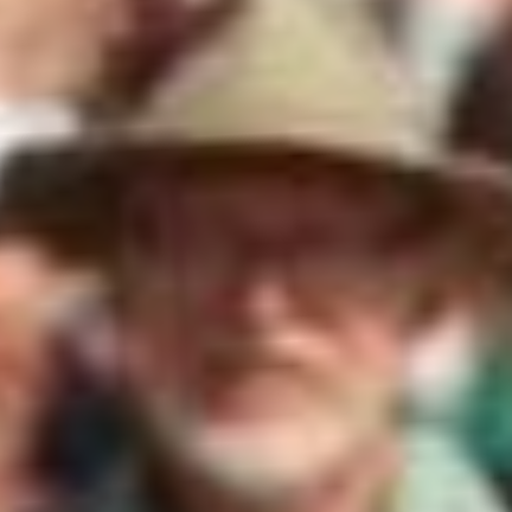}}} &
        \multicolumn{2}{c}{\centered{\includegraphics[width=\doublecolwidth\linewidth]{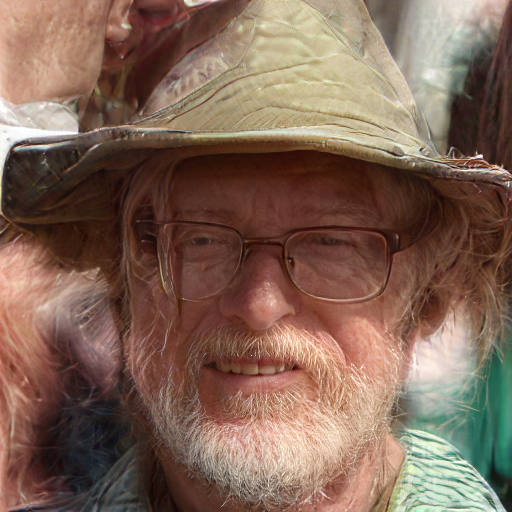}}} &
        \multicolumn{2}{c}{\centered{\includegraphics[width=\doublecolwidth\linewidth]{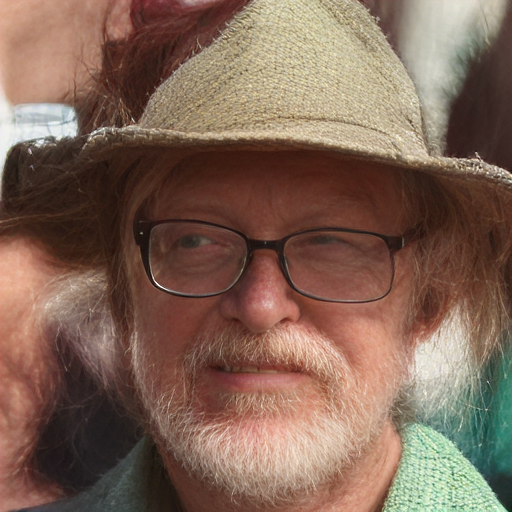}}} &
        \multicolumn{2}{c}{\centered{\includegraphics[width=\doublecolwidth\linewidth]{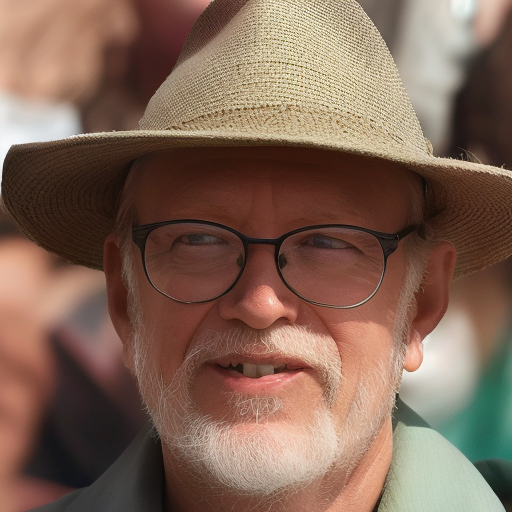}}} &
        \multicolumn{2}{c}{\centered{\includegraphics[width=\doublecolwidth\linewidth]{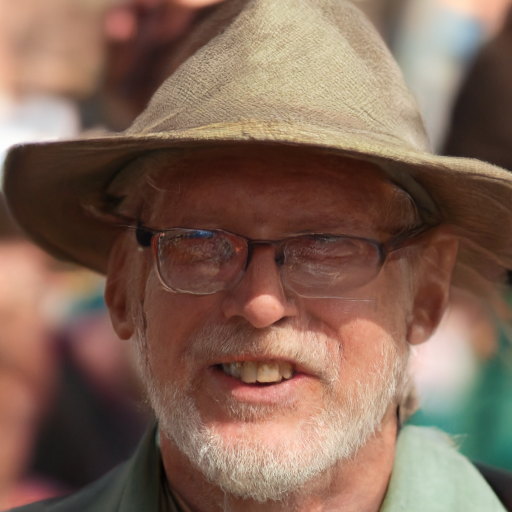}}} &
        \multicolumn{2}{c}{\centered{\includegraphics[width=\doublecolwidth\linewidth]{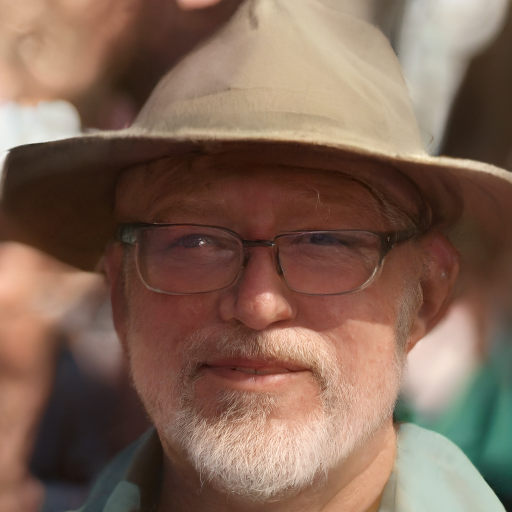}}}
	    \\
        \centered{\adjincludegraphics[width=\colwidth\linewidth, trim={{.2\width} {.375\height} {0.55\width} {.375\width}}, clip]{images/0011_wider_degraded.png}}&
        \centered{\adjincludegraphics[width=\colwidth\linewidth, trim={{.375\width} {.1\height} {0.375\width} {.65\height}}, clip]{images/0011_wider_degraded.png}}&
        \centered{\adjincludegraphics[width=\colwidth\linewidth, trim={{.2\width} {.375\height} {0.55\width} {.375\width}}, clip]{images/0011_wider_rf.png}}&
        \centered{\adjincludegraphics[width=\colwidth\linewidth, trim={{.375\width} {.1\height} {0.375\width} {.65\height}}, clip]{images/0011_wider_rf.png}}&
        \centered{\adjincludegraphics[width=\colwidth\linewidth, trim={{.2\width} {.375\height} {0.55\width} {.375\width}}, clip]{images/0011_wider_codeformer.png}}&
        \centered{\adjincludegraphics[width=\colwidth\linewidth, trim={{.375\width} {.1\height} {0.375\width} {.65\height}}, clip]{images/0011_wider_codeformer.png}}&
        \centered{\adjincludegraphics[width=\colwidth\linewidth, trim={{.2\width} {.375\height} {0.55\width} {.375\width}}, clip]{images/0011_wider_diffbir.png}}&
        \centered{\adjincludegraphics[width=\colwidth\linewidth, trim={{.375\width} {.1\height} {0.375\width} {.65\height}}, clip]{images/0011_wider_diffbir.png}}&
        \centered{\adjincludegraphics[width=\colwidth\linewidth, trim={{.2\width} {.375\height} {0.55\width} {.375\width}}, clip]{images/0011_wider_pmrf.png}}&
        \centered{\adjincludegraphics[width=\colwidth\linewidth, trim={{.375\width} {.1\height} {0.375\width} {.65\height}}, clip]{images/0011_wider_pmrf.png}}&
        \centered{\adjincludegraphics[width=\colwidth\linewidth, trim={{.2\width} {.375\height} {0.55\width} {.375\width}}, clip]{images/0011_wider_elad.png}}&
        \centered{\adjincludegraphics[width=\colwidth\linewidth, trim={{.375\width} {.1\height} {0.375\width} {.65\height}}, clip]{images/0011_wider_elad.png}}
        \\
        \multicolumn{2}{c}{\centered{\includegraphics[width=\doublecolwidth\linewidth]{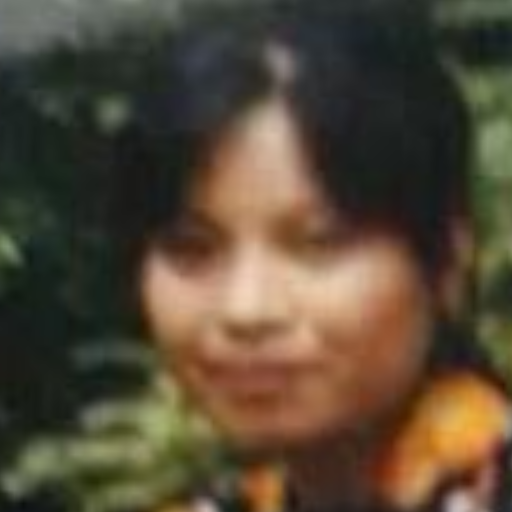}}} &
        \multicolumn{2}{c}{\centered{\includegraphics[width=\doublecolwidth\linewidth]{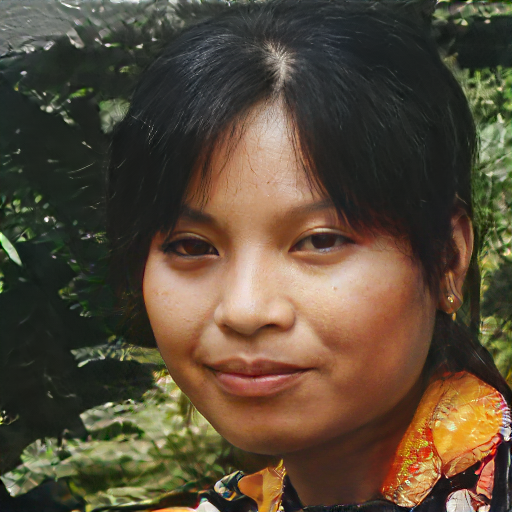}}} &
        \multicolumn{2}{c}{\centered{\includegraphics[width=\doublecolwidth\linewidth]{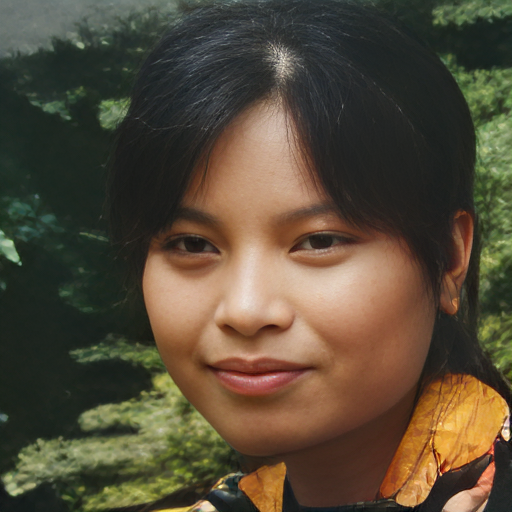}}} &
        \multicolumn{2}{c}{\centered{\includegraphics[width=\doublecolwidth\linewidth]{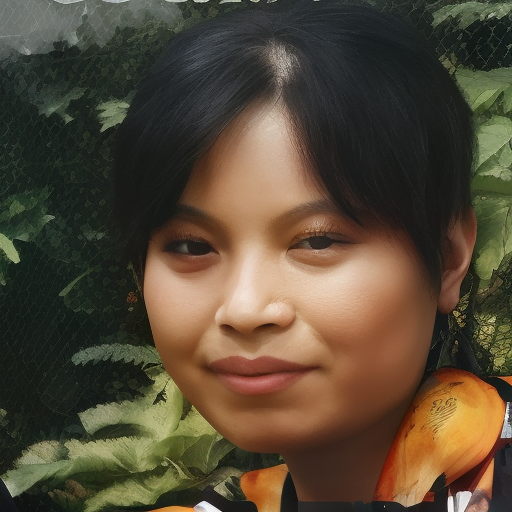}}} &
        \multicolumn{2}{c}{\centered{\includegraphics[width=\doublecolwidth\linewidth]{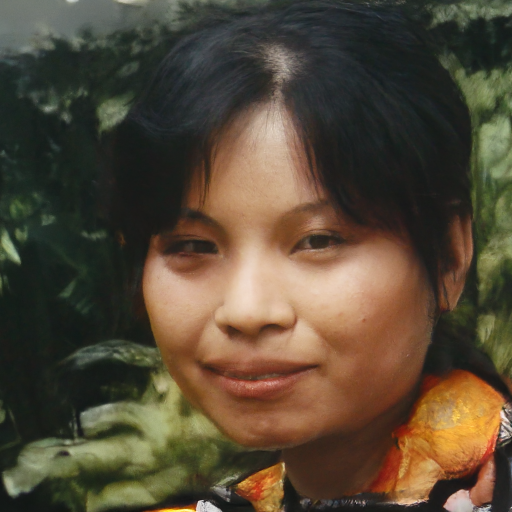}}} &
        \multicolumn{2}{c}{\centered{\includegraphics[width=\doublecolwidth\linewidth]{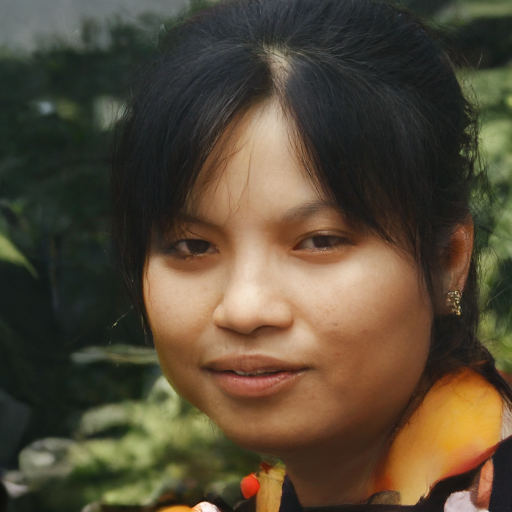}}}
        \\
        \centered{\adjincludegraphics[width=\colwidth\linewidth, trim={{0\width} {.75\height} {.75\width} {.0\width}}, clip]{images/00050_02_webphoto_degraded.png}}&
        \centered{\adjincludegraphics[width=\colwidth\linewidth, trim={{.225\width} {.375\height} {0.525\width} {.375\width}}, clip]{images/00050_02_webphoto_degraded.png}}&
        \centered{\adjincludegraphics[width=\colwidth\linewidth, trim={{0\width} {.75\height} {.75\width} {.0\width}}, clip]{images/00050_02_webphoto_rf.png}}&
        \centered{\adjincludegraphics[width=\colwidth\linewidth, trim={{.225\width} {.375\height} {0.525\width} {.375\width}}, clip]{images/00050_02_webphoto_rf.png}}&
        \centered{\adjincludegraphics[width=\colwidth\linewidth, trim={{0\width} {.75\height} {.75\width} {.0\width}}, clip]{images/00050_02_webphoto_codeformer.png}}&
        \centered{\adjincludegraphics[width=\colwidth\linewidth, trim={{.225\width} {.375\height} {0.525\width} {.375\width}}, clip]{images/00050_02_webphoto_codeformer.png}}&
        \centered{\adjincludegraphics[width=\colwidth\linewidth, trim={{0\width} {.75\height} {.75\width} {.0\width}}, clip]{images/00050_02_webphoto_diffbir.png}}&
        \centered{\adjincludegraphics[width=\colwidth\linewidth, trim={{.225\width} {.375\height} {0.525\width} {.375\width}}, clip]{images/00050_02_webphoto_diffbir.png}}&
        \centered{\adjincludegraphics[width=\colwidth\linewidth, trim={{0\width} {.75\height} {.75\width} {.0\width}}, clip]{images/00050_02_webphoto_pmrf.png}}&
        \centered{\adjincludegraphics[width=\colwidth\linewidth, trim={{.225\width} {.375\height} {0.525\width} {.375\width}}, clip]{images/00050_02_webphoto_pmrf.png}}&
        \centered{\adjincludegraphics[width=\colwidth\linewidth, trim={{0\width} {.75\height} {.75\width} {.0\width}}, clip]{images/00050_02_webphoto_elad.png}}&
        \centered{\adjincludegraphics[width=\colwidth\linewidth, trim={{.225\width} {.375\height} {0.525\width} {.375\width}}, clip]{images/00050_02_webphoto_elad.png}}
        \\
        \multicolumn{2}{c}{\centered{\includegraphics[width=\doublecolwidth\linewidth]{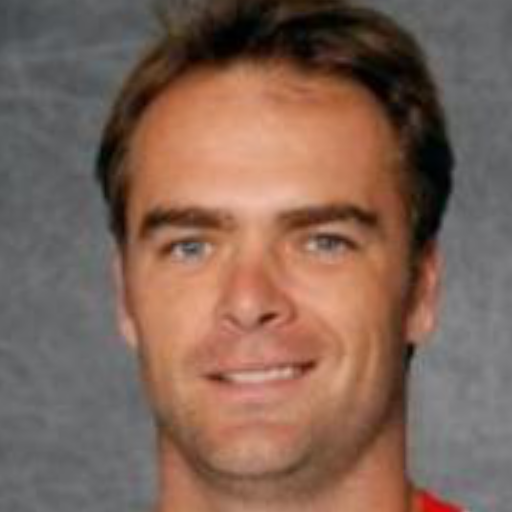}}} &
        \multicolumn{2}{c}{\centered{\includegraphics[width=\doublecolwidth\linewidth]{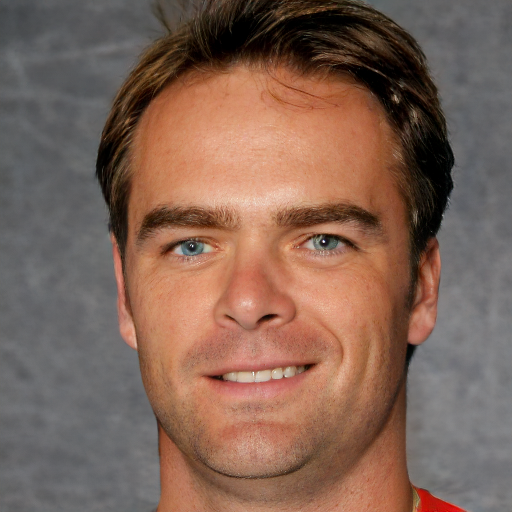}}} &
        \multicolumn{2}{c}{\centered{\includegraphics[width=\doublecolwidth\linewidth]{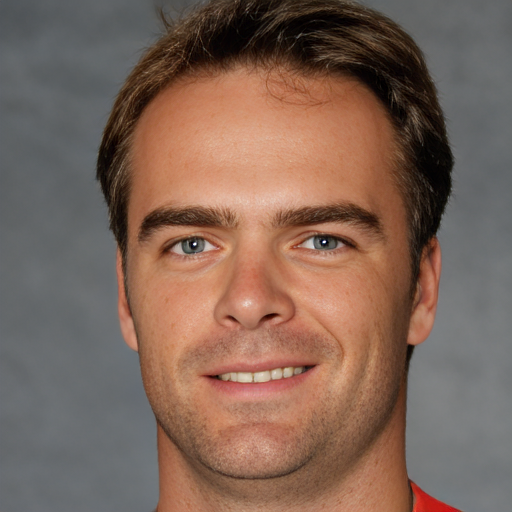}}} &
        \multicolumn{2}{c}{\centered{\includegraphics[width=\doublecolwidth\linewidth]{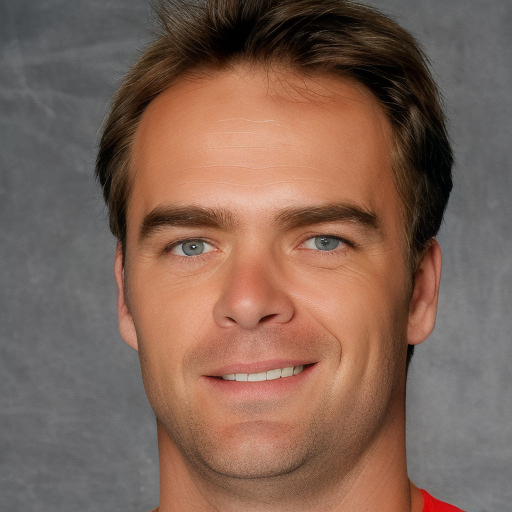}}} &
        \multicolumn{2}{c}{\centered{\includegraphics[width=\doublecolwidth\linewidth]{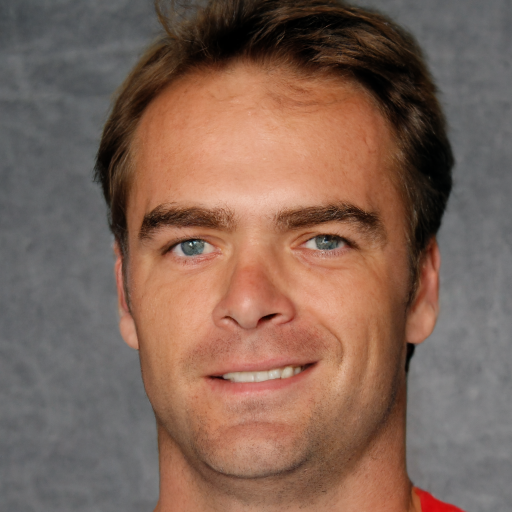}}} &
        \multicolumn{2}{c}{\centered{\includegraphics[width=\doublecolwidth\linewidth]{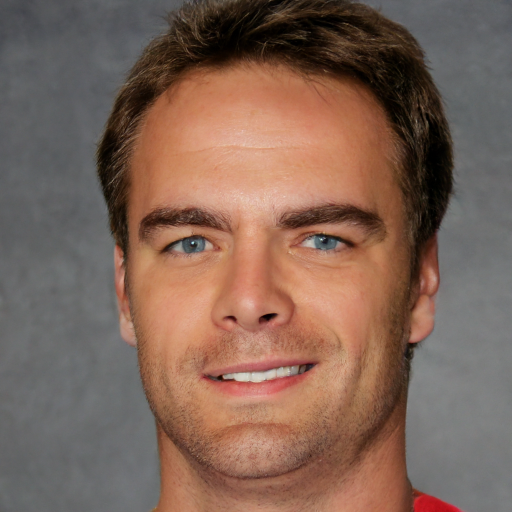}}}
        \\
        \centered{\adjincludegraphics[width=\colwidth\linewidth, trim={{.25\width} {.1\height} {.5\width} {.65\width}}, clip]{images/Curtis_Joseph_0001_00_lfw_degraded.png}}&
        \centered{\adjincludegraphics[width=\colwidth\linewidth, trim={{.25\width} {.4\height} {.5\width} {.35\width}}, clip]{images/Curtis_Joseph_0001_00_lfw_degraded.png}}&
        \centered{\adjincludegraphics[width=\colwidth\linewidth, trim={{.25\width} {.1\height} {.5\width} {.65\width}}, clip]{images/Curtis_Joseph_0001_00_lfw_rf.png}}&
        \centered{\adjincludegraphics[width=\colwidth\linewidth, trim={{.25\width} {.4\height} {.5\width} {.35\width}}, clip]{images/Curtis_Joseph_0001_00_lfw_rf.png}}&
        \centered{\adjincludegraphics[width=\colwidth\linewidth, trim={{.25\width} {.1\height} {.5\width} {.65\width}}, clip]{images/Curtis_Joseph_0001_00_lfw_codeformer.png}}&
        \centered{\adjincludegraphics[width=\colwidth\linewidth, trim={{.25\width} {.4\height} {.5\width} {.35\width}}, clip]{images/Curtis_Joseph_0001_00_lfw_codeformer.png}}&
        \centered{\adjincludegraphics[width=\colwidth\linewidth, trim={{.25\width} {.1\height} {.5\width} {.65\width}}, clip]{images/Curtis_Joseph_0001_00_lfw_diffbir.png}}&
        \centered{\adjincludegraphics[width=\colwidth\linewidth, trim={{.25\width} {.4\height} {.5\width} {.35\width}}, clip]{images/Curtis_Joseph_0001_00_lfw_diffbir.png}}&
        \centered{\adjincludegraphics[width=\colwidth\linewidth, trim={{.25\width} {.1\height} {.5\width} {.65\width}}, clip]{images/Curtis_Joseph_0001_00_lfw_pmrf.png}}&
        \centered{\adjincludegraphics[width=\colwidth\linewidth, trim={{.25\width} {.4\height} {.5\width} {.35\width}}, clip]{images/Curtis_Joseph_0001_00_lfw_pmrf.png}}&
        \centered{\adjincludegraphics[width=\colwidth\linewidth, trim={{.25\width} {.1\height} {.5\width} {.65\width}}, clip]{images/Curtis_Joseph_0001_00_lfw_elad.png}}&
        \centered{\adjincludegraphics[width=\colwidth\linewidth, trim={{.25\width} {.4\height} {.5\width} {.35\width}}, clip]{images/Curtis_Joseph_0001_00_lfw_elad.png}}
        \\
	    
    \end{tabular}
    \caption{Restoration examples of real-world images taken from WIDER-Test, WebPhoto-Test, and LFW-Test (top to bottom) for different methods. ELAD is highly competitive against current state-of-the-art end-to-end methods while being independent of the underlying prior.}
    \label{fig:images}
\end{figure*}
\endgroup

\section{Proxy distortion measures}
\label{sec:metric}

This section introduces the \emph{Proxy Mean Squared Error} (ProxMSE), a distortion measure that allows to practically \emph{rank} image restoration algorithms according to their MSE without access to the ground-truth images.
Using the same derivations, we present ProxLPIPS as a no-reference perceptual distortion measure.
We start by defining distortion measures in general and MSE and LPIPS in particular.
We continue by defining ProxMSE and ProxLPIPS, showing why they produce rankings that are faithful to the true MSE and LPIPS.
Then, we demonstrate that ProxMSE and ProxLPIPS are highly correlative to the true MSE and LPIPS by conducting controlled experiments on synthetic degradations.
Finally, we utilize ProxMSE and ProxLPIPS to rank several real-world face restoration algorithms.

\subsection{Background}

Non-blind image restoration algorithms are typically evaluated by full-reference distortion measures, which quantify the discrepancy between the reconstructed images and the ground-truth ones.
Formally, the average distortion of an estimator $\hat{\rvx}$ is defined by
\begin{align}
    \mathbb{E}_{(\vx, \hat\vx) \sim p_{\rvx, \hat\rvx}} [\Delta(\vx,\hat{\vx})],
\end{align}
where $\Delta(\vx,\hat{\vx})$ is some distortion measure (\eg, the squared error). %
Thus, measuring the average distortion in practice requires access to pairs of samples from $p_{\rvx,\hat{\rvx}}$.
The most common way to obtain such pairs is to degrade the given samples from $\smash{p_{\rvx}}$ according to $\smash{p_{\rvy|\rvx}}$ (\eg, add noise), and then reconstruct the results using $\smash{\hat{\rvx}}$.
Yet, in real-world scenarios, there is no access to $p_{\rvy|\rvx}$, so evaluating the distortion in such cases is impossible.

Two prominent examples of distortion measures are (1) the squared error $\Delta_\text{SE}(\vx,\hat{\vx}) = \norm{\vx-\hat{\vx}}_2^2$; and
(2) LPIPS~\cite{UnreasonableEffectivenessDeep2018zhanga}, a full-reference (perceptual) distortion measure that compares weighted features extracted from neural networks such as VGG~\cite{VeryDeepConvolutional2015simonyana}.
In \Cref{sec:lpips_supp} we show that LPIPS can be interpreted as a squared error in latent space.
Denoting by $\vz$ the feature vector corresponding to $\vx$, we have
$\Delta_\text{LPIPS}(\vx,\hat{\vx}) = \Delta_\text{SE}(\vz,\hat{\vz})$.

\subsection{Derivation}
\label{sec:proxmse}

Let $\hat{\rvx}$ be some estimator, and let $\rvx^*=\E \lft[ \rvx | \rvy \rgt]$ be the posterior mean (the MMSE estimator).
We define
\begin{align}
    \text{ProxMSE}(\hat{\rvx})\coloneqq\E_{(\hat\vx, \vx^*) \sim p_{\hat\rvx, \rvx^*}} \lft[ \norm{\hat{\vx} - \vx^*}_2^2 \rgt],
\end{align}
where the right hand side is the MSE between $\hat{\rvx}$ and $\rvx^{*}$.
Interestingly, ProxMSE satisfies the following appealing property (the proof is deferred to \Cref{sec:proofs}):
\begin{restatable}{proposition}{rankorder}
\label{prop:proxmse}
    The ProxMSE and the MSE of an estimator $\hat{X}$ are equal up to a constant which does not depend on $\hat\rvx$,
    \begin{align}
        \text{ProxMSE}(\hat{\rvx}) = \text{MSE}(\rvx,\hat{\rvx}) - d^*.\label{eq:prox-decomp}
    \end{align}
    Namely, the ranking order of estimators according to their ProxMSE is equivalent to that according to their MSE.
\end{restatable}
Note that~\cref{eq:prox-decomp} was originally proved by~\citet{TheoryDistortionPerceptionTradeoff2021freirich} (Lemma 2, Appendix B.1). \citet{PosteriorMeanRectifiedFlow2024ohayon} used $\text{ProxMSE}(\hat{\rvx})$ as a no-reference distortion measure, but did not assess its practical validity in experiments.

In practice, we do not have access to the true MMSE estimator, but rather to an approximation of it (typically a neural network trained to minimize the MSE).
In \Cref{sec:proxmse_bound} we develop a bound on this approximation.
Lastly, since LPIPS is just a MSE in latent space, we similarly define
\begin{align}
    \text{ProxLPIPS}(\hat{\rvx})\coloneqq\E_{(\hat\vz, \vz^*) \sim p_{\hat\rvz, \rvz^*}} \lft[ \norm{\hat{\vz} - \vz^*}_2^2 \rgt],
\end{align}
where $\rvz^*$ is the MMSE estimator \emph{in the latent space}, which again is approximated using a neural network trained to minimize the LPIPS loss.

\subsection{Experiments}
\label{sec:metric-exp}

\begin{figure}[tb]
    \centering
    \includegraphics[width=\linewidth]{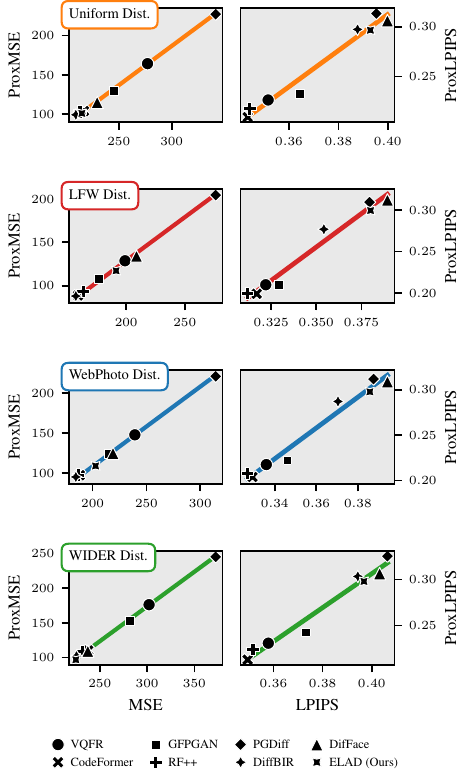}
    \caption{\textbf{Proxy distortion measures.}
    The plots compare the proxy measures with their true counterparts, for several state-of-the-art methods evaluated on the synthetic CelebA-Test datasets (\Cref{sec:deg_est}).
    A linear regression line is drawn for better clarity (for ProxMSE the slope equals one, following \Cref{prop:proxmse}).
    ProxMSE and ProxLPIPS rank methods similarly to the MSE and LPIPS measures without the need for ground-truth images.
    }
    \label{fig:bfr_rmse_id}
\end{figure}

\begin{table*}[tb]
    \centering
    \setlength{\tabcolsep}{3pt}
    \renewcommand\theadalign{bc}
    \small
    \begin{NiceTabular}{@{}lcccccccccccccc@{}}\toprule
    & \multicolumn{4}{c}{\thead{LFW-Test}} && \multicolumn{4}{c}{\thead{WebPhoto-Test}} && \multicolumn{4}{c}{\thead{WIDER-Test}} \\
\cmidrule{2-5} \cmidrule{7-10} \cmidrule{12-15}
\thead{Method} & \thead{FID $\downarrow$} & \thead{Proxy\\MSE} $\downarrow$ & \thead{Proxy\\LPIPS} $\downarrow$ & \thead{Proxy\\CMSE} $\downarrow$ && \thead{FID $\downarrow$} & \thead{Proxy\\MSE} $\downarrow$ & \thead{Proxy\\LPIPS} $\downarrow$ & \thead{Proxy\\CMSE} $\downarrow$ && \thead{FID $\downarrow$} & \thead{Proxy\\MSE} $\downarrow$ & \thead{Proxy\\LPIPS} $\downarrow$ & \thead{Proxy\\CMSE} $\downarrow$ \\
\midrule

\quad PGDiff~\cite{PGDiffGuidingDiffusion2023yang}    & \colorbox{tabfirst}{43.52} & 195.20 & \colorbox{tabsecond}{0.319} & 99.0 
          && 84.94 & 225.7 & 0.368 & 90.4 
          && 42.60 & 326.9 & 0.373 & 80.9 \\
\quad DifFace~\cite{DifFaceBlindFace2023yue}    & 47.39 & \colorbox{tabsecond}{131.0} & \colorbox{tabsecond}{0.319} & \colorbox{tabsecond}{76.7} 
          && \colorbox{tabfirst}{80.88} & \colorbox{tabsecond}{108.3} & \colorbox{tabsecond}{0.354} & \colorbox{tabsecond}{60.6} 
          && \colorbox{tabfirst}{37.03} & \colorbox{tabsecond}{113.6} & \colorbox{tabsecond}{0.334} & \colorbox{tabsecond}{49.7} \\
\quad ELAD (Ours)& \colorbox{tabsecond}{46.02} & \colorbox{tabfirst}{114.0} & \colorbox{tabfirst}{0.306} & \colorbox{tabfirst}{62.4} 
          && \colorbox{tabsecond}{81.89} & \colorbox{tabfirst}{92.7} & \colorbox{tabfirst}{0.343} & \colorbox{tabfirst}{48.3} 
          && \colorbox{tabsecond}{37.50} & \colorbox{tabfirst}{101.3} & \colorbox{tabfirst}{0.328} & \colorbox{tabfirst}{39.8} \\

\bottomrule

    \end{NiceTabular}
    \caption{Quantitative comparison between plug-and-play diffusion methods: PGDiff, DifFace, and ELAD on real-world BFR datasets.
    We highlight the \rectlabel{tabfirst}{first} and \rectlabel{tabsecond}{second} best-performing methods in each measure.
    ELAD is better in terms of distortion (ProxMSE, ProxLPIPS) and consistency (ProxCMSE), with minimal impact on perception (FID).}
    \label{tab:elad}
\end{table*}

We use the BFR regressor of DifFace~\cite{DifFaceBlindFace2023yue} as our approximate MMSE estimator.
This model was trained on the synthetic degradation process defined in~\cref{eq:bfr_model}.
For ProxLPIPS, we re-train the same model with an LPIPS loss: the VGG network is used to extract image features and the rest of the training scheme follows that of DifFace's regressor.

\paragraph{Synthetic datasets.}
In~\Cref{fig:bfr_rmse_id} we compare the proxy measures with their true counterparts (\eg, ProxMSE vs. MSE) on the synthetic datasets presented in \Cref{sec:deg_est}.
The figure shows that ProxMSE aligns well with MSE, implying that it is a reliable, no-reference proxy MSE measure.
Similarly, ProxLPIPS displays an excellent alignment with LPIPS, but it is not as accurate as ProxMSE.
We hypothesize that a better alignment might be possible with a better-trained LPIPS estimator. %

\paragraph{Real-world datasets.}
After establishing that ProxMSE and ProxLPIPS can serve as reliable proxy measures of MSE and LPIPS, respectively, we turn to evaluate existing real-world image restoration methods on datasets where MSE and LPIPS are impossible to compute.
\Cref{tab:bfr_rmse_real} in \Cref{sec:eval_supp} compares state-of-the-art \textbf{end-to-end} methods on the LFW-Test, WebPhoto-Test, and WIDER-Test datasets.
PMRF~\cite{PosteriorMeanRectifiedFlow2024ohayon} stands out in ProxMSE and ProxCMSE, which comes with no surprise as the method is trained to achieve the best MSE possible under a perfect perceptual index constraint. Regarding perception alone, DiffBIR~\cite{DiffBIRBlindImage2024lin} and RestorFormer++~\cite{RestoreFormerRealWorldBlind2023wang} generally lead.
Note that such a comparison is impossible with prior no-reference quality measures, as those typically do not correlate with the distortion.

\paragraph{ELAD evaluation.}
Both in synthetic (\Cref{fig:bfr_rmse_id,fig:bfr_crmse}) and real-world datasets (\Cref{tab:elad}), ELAD improves the consistency over DifFace~\cite{DifFaceBlindFace2023yue} and PGDiff~\cite{PGDiffGuidingDiffusion2023yang} considerably, as reflected by the (Prox)CMSE. A somewhat surprising side-effect of the improved consistency is an improved distortion performance, as seen by (Prox)MSE and (Prox)LPIPS. The slight decrease in FID compared to DifFace is expected as we add a guidance term that may compete with the unconditional generative prior. As seen in \Cref{fig:elad_consistency,fig:images}, this decrease is not noticeable. In fact, ELAD is highly competitive in terms of perceptual quality compared to end-to-end methods (see~\Cref{fig:images}). All in all, ELAD is the state-of-the-art plug-and-play method for BFR tasks.

\section{Related work}
\label{sec:related-work}

\paragraph{Degradation estimation and its application.}
The idea of identifying the degradations that a given image has gone through was previously explored in the literature in different contexts.
In~\cite{ChainRestorationMultiTaskImage2024cao}, a classifier was trained to determine which degradation is present in a given image, out of 4 possibilities (rain, haze, noise, low-light).
An image captioner was trained in \cite{DaLPSRLeverageDegradationAligned2024jiang} to produce the list of degradations from an image, and a coarse description of their severity (\eg, high amount of noise, low amount of blur).
These methods neglect the \emph{order} and the exact parameters' \emph{values} of each degradation.
Another line of work~\cite{UnpairedRealWorldSuperResolution2022romero,MetricLearningBased2022mou,ExploreImageDeblurring2021tran} learns the latent representation of degradations in an unsupervised manner.
In all of the above, the identified degradation is used as an additional input condition for the restoration algorithm.
This is fundamentally different than our approach, as we do not train an algorithm conditioned on the predicted degradations, but rather use the estimated degradations to approximate the likelihood.
In~\cite{UnprocessingImagesLearned2019brooks}, camera sensor noise was modeled and fitted as a heteroscedastic Gaussian distribution to mimic a real-world noisy dataset, and ~\cite{ZeroShotSuperResolutionUsing2018shochera,Ji_2020_CVPR_Workshops} optimized a down-sampling kernel per-image. Both are not applicable to a more complex chain of degradations.

\paragraph{Consistency.}
Measuring the consistency of the reconstructions with the inputs is common for non-blind image restoration tasks, such as noiseless super-resolution and JPEG decoding~\cite{SRFlowLearningSuperResolution2020lugmayr,NTIRE2021Learning2021lugmayra,NTIRE2022Challenge2022lugmayr,ExplorableSuperResolution2020bahata,WhatsImageExplorable2021bahat,HighPerceptualQualityJPEG2023man}. The work in \cite{HighPerceptualQuality2021ohayona,StochasticImageDenoising2021kawarb} considered the consistency in noisy problems to be testing the residual image between the degraded input and the restored output for normality.~\citet{ReasonsSuperiorityStochastic2023ohayon} generalized this notion as a requirement for the similarity between the conditional distributions $p_{\rvy|\rvx}$ and $p_{\rvy|\hat\rvx}$. However, $p_{\rvy|\rvx}$ is unknown in blind restoration tasks.

\paragraph{Plug \& play image restoration.}
Blind plug-and-play methods typically assume a simple parametric family of degradations (\eg, convolution with a fixed-sized kernel), and attempt to jointly estimate the degradation's parameters (\eg, the kernel's weights) and the clean image ~\cite{ParallelDiffusionModels2023chung,FastDiffusionEM2024laroche,GibbsDDRMPartiallyCollapsed2023murataa,TamingGenerativeDiffusion2024tua}.
Such parametric families are typically too narrow to describe complex degradations that may occur in practice (\eg, compression artifacts).
Another line of work utilizes heuristic guidance terms and initialization schemes to generate images using a generative model, such that the generated outputs share the same features with the given degraded measurement~\cite{DifFaceBlindFace2023yue,PGDiffGuidingDiffusion2023yang}.
Due to the heuristic nature of such methods, the restored images may ``over-fit'' the degraded inputs and produce artifacts, or ``under-fit'' them and become inconsistent.

\paragraph{Blind (no-reference) performance measures.}
Evaluating real-world image restoration methods is commonly done using no-reference quality measures (\eg NIQE~\cite{MakingCompletelyBlind2013mittala}, NIMA~\cite{NIMANeuralImage2018talebi}, ClipIQA~\cite{ExploringCLIPAssessing2023wang}) or statistical divergences (\eg, FID~\cite{AssessingGenerativeModels2018sajjadi}, KID~\cite{GANsTrainedTwo2017heusela}).
All of these measures do not consider the given degraded inputs. In other words, they may even highly reward a restoration algorithm for producing outputs that are entirely inconsistent with the inputs.

\paragraph{Real-world paired datasets.}Another way to evaluate real-world image restoration algorithms is to construct datasets that consist of pairs of low and high-quality images, \eg, where the first is acquired with a mobile phone and the second with a high-quality DSLR camera~\cite{BenchmarkingDenoisingAlgorithms2017plotza,Ji_2020_CVPR_Workshops,LearningRestoreHazy2021zhang,UnderwaterImageRestoration2022han}.
This approach is too costly, and may not faithfully represent the distribution of real-world degraded images~\cite{RealESRGANTrainingRealWorld2021wangb,RealWorldBlindFace2021wang,RobustBlindFace2022zhou}.

\section{Conclusion and limitations}
\label{sec:conclusion}

This work aimed to provide practical tools to help foster progress in a highly challenging task: Real-world image restoration.
We proposed ELA, a new approach to approximate the consistency of a reconstructed candidate with a given degraded image.
Our method relies on a novel degradation estimator, which we train to predict the chain of degradations that a given degraded image has gone through.
Using ELA, we directly define measures of consistency for real-world image restoration algorithms (Prox)CMSE, and develop ELAD: a new plug-and-play real-world image restoration algorithm that beats its predecessors.
Moreover, our proposed ProxMSE and ProxLPIPS measures offer a new way to indicate the \emph{distortion} of real-world image restoration methods, without any access to the ground-truth images.
While our work provides new effective tools for real-world image restoration, it also has several limitations.
First, the parametrization of the degradation family may limit the effectiveness of the degradation estimator.
For example, some images in WebPhoto seem to contain haze (see, \eg,~\Cref{sec:datasets_supp}), which is not accounted in~\cref{eq:bfr_model}.
Second, ELA utilizes a degradation predictor to approximate the mean of $p_{Y|X}$, yet it is possible that many different degradations correspond to a given degraded image.
A more faithful approximation could be one that averages over such space of possible degradations.
Third, our proxy measures assumes access to accurate approximations of MMSE estimators.
While in \Cref{sec:proxmse_bound} we provide an upper bound for the error of such a measure, the tightness of this bound remains unclear in practice.
Apart from addressing the above limitations, follow-up work could explore a multitude of different directions, such as
\emph{(i)} testing our tools on other real-world tasks such as blind super-resolution (BSR), and \emph{(ii)} training and evaluating other methods (\eg, CodeFormer or DiffBIR) on synthetic datasets that mimic real-world ones (as described in~\cref{sec:deg_est}).

\paragraph{Acknowledgment.}
We thank Hila Manor and Noam Elata for proofreading an earlier version of the paper.

\clearpage
{
    \small
    \bibliographystyle{ieeenat_fullname}

}

\newpage
\setcounter{page}{1}
\onecolumn
\begin{center}
\Large
\textbf{\thetitle}\\
\vspace{0.5em}Supplementary Material \\
\vspace{1.0em}
\end{center}
\appendix

\section{Supplementary results}
\label{sec:eval_supp}

\subsection{Proxy performance of real-world end-to-end methods}

For completeness, we evaluate real-world blind face restoration end-to-end methods in \Cref{tab:bfr_rmse_real} using the proxy measures we defined in this paper. To complement our distortion measures, we measure the perceptual quality of the reconstructions with FID~\cite{GANsTrainedTwo2017heusela}. It is important to note that the performance of ELAD, despite being a plug-and-play method, is not far from that of the leading end-to-end methods, as evident from  \cref{fig:images}.

\begin{table*}[b]
    \centering
    \setlength{\tabcolsep}{3pt}
    \renewcommand\theadalign{bc}
    \small
    \begin{NiceTabular}{@{}lcccccccccccccc@{}}\toprule
    & \multicolumn{4}{c}{\thead{LFW-Test}} && \multicolumn{4}{c}{\thead{WebPhoto-Test}} && \multicolumn{4}{c}{\thead{WIDER-Test}} \\
\cmidrule{2-5} \cmidrule{7-10} \cmidrule{12-15}
\thead{Method} & \thead{FID $\downarrow$} & \thead{Proxy\\MSE} $\downarrow$ & \thead{Proxy\\LPIPS} $\downarrow$ & \thead{Proxy\\CMSE} $\downarrow$ && \thead{FID $\downarrow$} & \thead{Proxy\\MSE} $\downarrow$ & \thead{Proxy\\LPIPS} $\downarrow$ & \thead{Proxy\\CMSE} $\downarrow$ && \thead{FID $\downarrow$} & \thead{Proxy\\MSE} $\downarrow$ & \thead{Proxy\\LPIPS} $\downarrow$ & \thead{Proxy\\CMSE} $\downarrow$ \\
\midrule

VQFR~\cite{VQFRBlindFace2022gu}       & 51.31 & 126.8 & \colorbox{tabthird}{0.221} & 46.8 
          && \colorbox{tabsecond}{75.86} & 156.7 & \colorbox{tabsecond}{0.278} & 45.8 
          && 44.09 & 282.6 & 0.333 & 41.0 \\
CodeFormer~\cite{RobustBlindFace2022zhou} & 52.84 & 90.1 & \colorbox{tabfirst}{0.205} & 38.7 
          && 83.93 & \colorbox{tabsecond}{78.7} & \colorbox{tabfirst}{0.250} & \colorbox{tabthird}{33.6} 
          && \colorbox{tabsecond}{39.22} & \colorbox{tabthird}{141.3} & \colorbox{tabfirst}{0.274} & 31.5 \\
GFPGAN~\cite{RealWorldBlindFace2021wang}     & \colorbox{tabthird}{50.32} & 110.7 & 0.223 & 64.9 
          && 87.29 & 626.9 & 0.316 & 532.1 
          && \colorbox{tabthird}{39.29} & 296.1 & \colorbox{tabthird}{0.319} & 137.0 \\
RF++~\cite{RestoreFormerRealWorldBlind2023wang}       & 50.72 & \colorbox{tabthird}{89.1} & \colorbox{tabsecond}{0.211} & \colorbox{tabthird}{26.0} 
          && \colorbox{tabfirst}{75.50} & 334.3 & \colorbox{tabthird}{0.282} & 247.3 
          && 45.45 & 201.3 & \colorbox{tabsecond}{0.317} & \colorbox{tabthird}{29.8} \\
DiffBIR~\cite{DiffBIRBlindImage2024lin}    & \colorbox{tabfirst}{40.90} & \colorbox{tabsecond}{82.8} & 0.283 & \colorbox{tabsecond}{18.5} 
          && 92.67 & \colorbox{tabthird}{82.9} & 0.357 & \colorbox{tabsecond}{16.9} 
          && \colorbox{tabfirst}{35.83} & \colorbox{tabsecond}{120.0} & 0.345 & \colorbox{tabsecond}{18.6} \\
PMRF~\cite{PosteriorMeanRectifiedFlow2024ohayon}       & \colorbox{tabsecond}{49.49} & \colorbox{tabfirst}{40.1} & 0.251 & \colorbox{tabfirst}{13.5} 
          && \colorbox{tabthird}{81.03} & \colorbox{tabfirst}{41.9} & 0.308 & \colorbox{tabfirst}{12.9} 
          && 41.19 & \colorbox{tabfirst}{86.3} & 0.328 & \colorbox{tabfirst}{16.1} \\

\bottomrule

    \end{NiceTabular}
    \caption{Quantitative comparison of leading end-to-end methods on real-world blind face restoration datasets. We evaluate perceptual quality via FID, distortion via ProxMSE, and ProxLPIPS, and consistency via ProxCMSE. We highlight the \rectlabel{tabfirst}{first}, \rectlabel{tabsecond}{second}, and \rectlabel{tabthird}{third} best-performing methods in each measure. PMRF stands out in ProxMSE and ProxCMSE, which comes with no surprise as the method is trained to achieve the best MSE possible under a perfect perceptual index constraint. In terms of perceptual quality alone, DiffBIR and RestorFormer++ generally outperforms the other methods.}
    \label{tab:bfr_rmse_real}
\end{table*}

\subsection{CelebA-Test synthetic results}

In \Cref{fig:bfr_crmse,fig:bfr_rmse_id} in the main paper, we demonstrate the alignment between each proxy measure and its true counterpart measure on synthetic CelebA-Test datasets, using various plug-and-play and end-to-end methods. We also provide the same data in~\Cref{tab:celeba}.
As we noted in~\cref{sec:metric-exp}, a key takeaway from these results is the dominance of ELAD over the other plug-and-play methods in all measures. This is consistent with the results in \Cref{tab:elad}. 

\subsection{Additional visual results}

In \Cref{fig:images_celeba,fig:images_lfw,fig:images_webphoto,fig:images_wider}, we present additional restoration examples produced by various methods, including ELAD, on CelebA-Test, LFW-Test, WebPhoto-Test, and WIDER-test.

\section{Proof of proposition 1}
\label{sec:proofs}
\rankorder*
\begin{proof}

The proof follows directly from Lemma 2 in~\cite{TheoryDistortionPerceptionTradeoff2021freirich} (Appendix B.1).
Namely, for any estimator we can write
\begin{align}
    \mathbb{E}[\norm{X-\hat{X}}^{2}]&=\mathbb{E}[\norm{(X-X^{*})-(\hat{X}-X^{*})}^{2}]\\&=\mathbb{E}[\norm{X-X^{*}}^{2}]+\mathbb{E}[\norm{\hat{X}-X^{*}}^{2}]-2\mathbb{E}[(X-X^{*})^{\top}(\hat{X}-X^{*})].\label{eq:dror1}
\end{align}
By the law of total expectation, it follows that
\begin{align}
    \mathbb{E}[(X-X^{*})^{\top}(\hat{X}-X^{*})]=\mathbb{E}\left[\mathbb{E}\left[(X-X^{*})^{\top}(\hat{X}-X^{*}) \lvert Y\right]\right].
\end{align}
Since $\hat{X}$ and $X$ are independent given $Y$, and since $X^{*}$ is a function of $Y$, we have
\begin{align}
    \mathbb{E}[(X-X^{*})^{\top}(\hat{X}-X^{*})]=\mathbb{E}\left[\mathbb{E}\left[(X-X^{*})^{\top}|Y\right]\mathbb{E}\left[(\hat{X}-X^{*})|Y\right]\right].\label{eq:dror2}
\end{align}
Now, it holds that
\begin{align}
    \mathbb{E}\left[(X-X^{*})^{\top}|Y\right]&=\mathbb{E}\left[X|Y\right]-\mathbb{E}\left[X^{*}|Y\right]\\
    &=\mathbb{E}\left[X|Y\right]-\mathbb{E}\left[X|Y\right]=0.
\end{align}
Plugging this back to~\cref{eq:dror2}, we get $\smash{\mathbb{E}[(X-X^{*})^{\top}(\hat{X}-X^{*})]=0}$, and therefore~\cref{eq:dror1} becomes
\begin{align}
    \text{MSE}(X,\hat{X})&=\E [ \norm{\rvx - \hat{\rvx}}_2^2 ] \\
    &=\E \lft[ \norm{\rvx - \rvx^{*}}_2^2 \rgt]+\E [ \norm{\hat{\rvx} - \rvx^{*}}_2^2 ]\\
    &=d^{*}+\E [ \norm{\hat{\rvx} - \rvx^{*}}_2^2]\label{eq:dror}
\end{align}
where $\smash{d^*=\E \lft[ \norm{\rvx - \rvx^*}_2^2 \rgt]}$ (the MSE of the MMSE estimator) does not change with $\hat{\rvx}$, and
\begin{align}
    \E [ \norm{\hat{\rvx} - \rvx^{*}}_2^2 ]=\text{ProxMSE}(\hat{X}).
\end{align}
Subtracting $d^{*}$ from both sides in~\cref{eq:dror} leads to the desired result.

\end{proof}

\section{MMSE estimator approximation effect}
\label{sec:proxmse_bound}
In practice, we only have an approximation of the true MMSE estimator $\rvx^{*}$ (\eg, a neural network trained to minimize the MSE loss).
Denoting by $\tilde{\rvx}^{*}$ such an approximation of $X^{*}$, we are interested in the effect of the approximation error $\rvr=\tilde{\rvx}^{*}-\rvx^*$ on the proposed ProxMSE measure.
To this end, we present the following bound on the absolute error of ProxMSE.
\begin{restatable}{proposition}{rankerror}
    The absolute error of ProxMSE when using $\tilde{\rvx}^{*}$ instead of $\rvx^*$ is bounded by
    \begin{equation}
    \label{eq:prox_bound}
        \abs{\E \lft[ \norm{\hat{\rvx} - \rvx^*}_2^2 \rgt] - \E \lft[ \norm{\hat{\rvx} - \tilde{\rvx}^{*}}_2^2 \rgt]} \leq \E \lft[ \norm{\rvr}_2^2 + 4 \norm{\rvr}_1 \rgt],
    \end{equation}
    where $\rvr=\tilde{\rvx}^{*}-\rvx^{*}$ and we assume that image pixels are taking values in $[0,1]$.
\end{restatable}

\begin{proof}
We start by decomposing $\smash{\norm{\hat{\rvx} - \rvx^*}_2^2}$ and $\smash{\norm{\hat{\rvx} - \tilde{\rvx}^{*}}_2^2}$. Since $\rvx^*=\tilde{\rvx}^{*}-\rvr$, it holds that
\begin{align}
    \norm{\hat{\rvx} - \rvx^*}_2^2 &= \norm{\hat{\rvx}}_2^2 + \norm{\rvx^*}_2^2 -2 \inner{\hat{\rvx}}{\rvx^*} \label{eq:decomp_xstar_norm}
\\ 
    &= \norm{\hat{\rvx}}_2^2 + \norm{\tilde{\rvx}^{*}}_2^2 + \norm{\rvr}_2^2 -2 \inner{\tilde{\rvx}^{*}}{\rvr} -2 \inner{\hat{\rvx}}{\rvx^*} \label{eq:decomp_xstar_norm_1}\\ 
    &= \norm{\hat{\rvx}}_2^2 + \norm{\tilde{\rvx}^{*}}_2^2 + \norm{\rvr}_2^2 -2 \inner{\tilde{\rvx}^{*}}{\rvr} -2 \inner{\hat{\rvx}}{\tilde{\rvx}^{*}} +2 \inner{\hat{\rvx}}{\rvr}, \label{eq:decomp_xstar_norm_2}
\end{align}
where we expanded $\norm{\rvx^*}_2^2$ in \cref{eq:decomp_xstar_norm_1} and $\inner{\hat{\rvx}}{\rvx^*}$ in \cref{eq:decomp_xstar_norm_2}. Moreover, we have
\begin{align}
\label{eq:decomp_xbar_norm}
    \norm{\hat{\rvx} - \tilde{\rvx}^{*}}_2^2 &= \norm{\hat{\rvx}}_2^2 + \norm{\tilde{\rvx}^{*}}_2^2 -2 \inner{\hat{\rvx}}{\tilde{\rvx}^{*}}.
\end{align}
Taking the absolute difference between the expected values in \cref{eq:decomp_xstar_norm,eq:decomp_xbar_norm}, we get
\begin{align}
    \abs{\E \lft[ \norm{\hat{\rvx} - \rvx^*}_2^2 \rgt] - \E \lft[ \norm{\hat{\rvx} - \tilde{\rvx}^{*}}_2^2 \rgt]} &= \abs{\E \lft[ \norm{\rvr}_2^2 +2 \inner{\hat{\rvx}}{\rvr} -2 \inner{\tilde{\rvx}^{*}}{\rvr} \rgt]} \\
    &\leq \E \lft[ \norm{\rvr}_2^2 \rgt] +2 \abs{\E \lft[ \inner{\hat{\rvx}}{\rvr} \rgt]} +2 \abs{\E \lft[ \inner{\tilde{\rvx}^{*}}{\rvr} \rgt]},\label{eq:triang}
\end{align}
where~\cref{eq:triang} holds due to the triangle inequality.
Moreover,
\begin{align}
\label{eq:jensen}
    & \abs{\E \lft[ \norm{\hat{\rvx} - \rvx^*}_2^2 \rgt] - \E \lft[ \norm{\hat{\rvx} - \tilde{\rvx}^{*}}_2^2 \rgt]} \leq \E \lft[ \norm{\rvr}_2^2 \rgt] +2 \E \lft[ \abs{ \inner{\hat{\rvx}}{\rvr} } \rgt] +2 \E \lft[ \abs{ \inner{\tilde{\rvx}^{*}}{\rvr} } \rgt],
\end{align}
where we used Jensen's inequality on the convex function $f(x)=|x|$.

Now, let us focus on $\abs{ \inner{\hat{\rvx}}{\rvr} }$. By applying the triangle inequality and assuming that each entry in $\hat\rvx$ takes values in $[0,k]$, we get
\begin{align}
    \abs{ \inner{\hat{\rvx}}{\rvr} } = \abs{\sum_i \hat\rvx_i \rvr_i} \leq \sum_i \abs{\hat\rvx_i \rvr_i} \leq \sum_i k \abs{\rvr_i} = k \norm{\rvr}_1.
\end{align}
Similarly, we have $\abs{ \inner{\tilde{\rvx}^{*}}{\rvr} } \leq k \norm{\rvr}_1$.
Substituting these bounds back into \cref{eq:jensen} and assuming $k=1$ (as typically used in practice), we get
\begin{align}
    \abs{\E \lft[ \norm{\hat{\rvx} - \rvx^*}_2^2 \rgt] - \E \lft[ \norm{\hat{\rvx} - \tilde{\rvx}^{*}}_2^2 \rgt]} \leq \E \lft[ \norm{\rvr}_2^2 \rgt] +4k \E \lft[ \abs{ \norm{\rvr}_1 } \rgt] = \E \lft[ \norm{\rvr}_2^2 \rgt] +4 \E \lft[ \abs{ \norm{\rvr}_1 } \rgt]
\end{align}
\end{proof}
In practice, computing $\rvr$ is impossible since $\rvx^*$ is unavailable. Hence, we showcase in controlled experiments (\Cref{sec:metric-exp}) that using an approximation of $\tilde{\rvx}^{*}$ still yields rankings that are consistent with the true MSE.

\section{LPIPS is MSE in latent space}
\label{sec:lpips_supp}

Denote by $\vf_l\in\R^{H_l\times W_l\times C_l}$ the channel-wise normalized feature of $\vx$ from the $l$'th layer and by $\vw_l\in\R^{C_l}$ a per-channel weight vector.
To compute LPIPS~\cite{UnreasonableEffectivenessDeep2018zhanga}, we average spatially and sum channel-wise a weighted $\ell_2$ distance per element,
\begin{align}
    \Delta_\text{LPIPS}(\vx,\hat{\vx}) = \sum_l \tfrac{1}{H_l W_l} \sum_{h,w} \norm{\vw_l \odot (\vf_{h,w}^l-\hat{\vf}^l_{h,w})}_2^2.
\end{align}
This is equivalent to an MSE between flattened feature vectors. Denote by $\vz=[\texttt{vec}(\vz_1),\ldots,\texttt{vec}(\vz_L)]^\top$, where $\vz_l = \tfrac{1}{\sqrt{H_l W_l}} \vw_l \odot \vf_l$, then
\begin{align}
    \Delta_\text{LPIPS}(\vx,\hat{\vx}) = \norm{\vz-\hat{\vz}}_2^2 = \Delta_\text{SE}(\vz,\hat{\vz}).
\end{align}

\begin{table*}[tb]
    \centering
    \setlength{\tabcolsep}{3pt}
    \renewcommand\theadalign{bc}
    \small
    
    \begin{NiceTabular}{@{}lccccccccccccc@{}}\toprule
    & \multicolumn{6}{c}{\thead{CelebA-Test (Uniform Dist. \#1)}} && \multicolumn{6}{c}{\thead{CelebA-Test (Uniform Dist. \#2)}} \\
\cmidrule{2-7} \cmidrule{9-14}
\thead{Method} & \thead{MSE} $\downarrow$ & \thead{Proxy\\MSE} $\downarrow$ & \thead{LPIPS} $\downarrow$ & \thead{Proxy\\LPIPS} $\downarrow$ & \thead{CMSE} $\downarrow$ & \thead{Proxy\\CMSE} $\downarrow$ && \thead{MSE} $\downarrow$ & \thead{Proxy\\MSE} $\downarrow$ & \thead{LPIPS} $\downarrow$ & \thead{Proxy\\LPIPS} $\downarrow$ & \thead{CMSE} $\downarrow$ & \thead{Proxy\\CMSE} $\downarrow$ \\
\midrule

\quad VQFR~\cite{VQFRBlindFace2022gu} &                      277.2 &                      164.2 & \colorbox{tabsecond}{0.35} &  \colorbox{tabthird}{0.226} & -- &                      48.1 &&                      637.4 &                      364.7 &                      0.45 &                      0.355 &                      46.31 &                      50.5 \\
\quad CodeFormer~\cite{RobustBlindFace2022zhou} &                      217.0 &                      103.8 &  \colorbox{tabfirst}{0.34} &  \colorbox{tabfirst}{0.209} & -- &  \colorbox{tabthird}{41.7} &&                      429.7 &                      151.1 &  \colorbox{tabfirst}{0.40} &  \colorbox{tabfirst}{0.245} &                      54.38 &                      59.9 \\
\quad GFPGAN~\cite{RealWorldBlindFace2021wang} &                      245.5 &                      129.5 &  \colorbox{tabthird}{0.36} &                      0.232 & -- &                      67.8 &&                      636.9 &                      363.6 &                      0.47 &                      0.369 &                      117.30 &                      122.4 \\
\quad RF++~\cite{RestoreFormerRealWorldBlind2023wang} & \colorbox{tabsecond}{213.3} &  \colorbox{tabthird}{102.6} &  \colorbox{tabfirst}{0.34} & \colorbox{tabsecond}{0.217} & -- & \colorbox{tabsecond}{30.4} &&                      554.2 &                      282.0 &                      0.45 &                      0.354 & \colorbox{tabsecond}{30.14} &  \colorbox{tabfirst}{35.2} \\
\quad DiffBIR~\cite{DiffBIRBlindImage2024lin} &  \colorbox{tabfirst}{209.4} &  \colorbox{tabfirst}{99.0} &                      0.39 &                      0.297 & -- &  \colorbox{tabfirst}{26.2} &&  \colorbox{tabthird}{384.1} &  \colorbox{tabthird}{100.2} &  \colorbox{tabthird}{0.44} &                      0.326 &  \colorbox{tabfirst}{27.66} & \colorbox{tabsecond}{35.3} \\
\emph{Plug \& play} \\
\quad PGDiff~\cite{PGDiffGuidingDiffusion2023yang} &                      341.6 &                      227.3 &                      0.40 &                      0.313 & -- &                      84.2 &&                      712.4 &                      441.5 &                      0.47 &                      0.379 &                      93.00 &                      93.5 \\
\quad DifFace~\cite{DifFaceBlindFace2023yue} &                      229.6 &                      114.9 &                      0.40 &                      0.306 & -- &                      63.6 && \colorbox{tabsecond}{371.7} & \colorbox{tabsecond}{90.9} &  \colorbox{tabthird}{0.44} &  \colorbox{tabthird}{0.308} &                      50.83 &                      58.0 \\
\quad ELAD (Ours) &  \colorbox{tabthird}{215.1} & \colorbox{tabsecond}{100.8} &                      0.39 &                      0.296 & -- &                      52.1 &&  \colorbox{tabfirst}{363.7} &  \colorbox{tabfirst}{83.4} & \colorbox{tabsecond}{0.43} & \colorbox{tabsecond}{0.306} &  \colorbox{tabthird}{42.41} &  \colorbox{tabthird}{49.2} \\

\bottomrule

    \CodeAfter
    \tikz \draw [dash pattern={on 2pt off 2pt}, shorten > = 2pt] (8.5-|2) -- (8.5-|15) ;
    \end{NiceTabular}

    \begin{NiceTabular}{@{}lccccccccccccc@{}}\toprule
    & \multicolumn{6}{c}{\thead{CelebA-Test (LFW Dist.)}} && \multicolumn{6}{c}{\thead{CelebA-Test (WebPhoto Dist.)}} \\
\cmidrule{2-7} \cmidrule{9-14}
\thead{Method} & \thead{MSE} $\downarrow$ & \thead{Proxy\\MSE} $\downarrow$ & \thead{LPIPS} $\downarrow$ & \thead{Proxy\\LPIPS} $\downarrow$ & \thead{CMSE} $\downarrow$ & \thead{Proxy\\CMSE} $\downarrow$ && \thead{MSE} $\downarrow$ & \thead{Proxy\\MSE} $\downarrow$ & \thead{LPIPS} $\downarrow$ & \thead{Proxy\\LPIPS} $\downarrow$ & \thead{CMSE} $\downarrow$ & \thead{Proxy\\CMSE} $\downarrow$ \\
\midrule

\quad VQFR~\cite{VQFRBlindFace2022gu} &                      199.0 &                      128.8 & \colorbox{tabsecond}{0.32} &  \colorbox{tabthird}{0.210} &                      40.33 &                      44.5 &&                      239.3 &                      147.6 & \colorbox{tabsecond}{0.34} &  \colorbox{tabthird}{0.218} &                      37.22 &                      39.2 \\
\quad CodeFormer~\cite{RobustBlindFace2022zhou} & \colorbox{tabsecond}{158.9} & \colorbox{tabsecond}{88.9} & \colorbox{tabsecond}{0.32} &  \colorbox{tabfirst}{0.199} &  \colorbox{tabthird}{31.58} &  \colorbox{tabthird}{35.0} &&  \colorbox{tabthird}{188.1} & \colorbox{tabsecond}{97.0} &  \colorbox{tabfirst}{0.33} &  \colorbox{tabfirst}{0.204} &  \colorbox{tabthird}{27.37} &  \colorbox{tabthird}{29.7} \\
\quad GFPGAN~\cite{RealWorldBlindFace2021wang} &                      176.7 &                      107.7 &  \colorbox{tabthird}{0.33} &  \colorbox{tabthird}{0.210} &                      54.93 &                      59.0 &&                      214.9 &                      123.2 &  \colorbox{tabthird}{0.35} &                      0.222 &                      55.33 &                      57.6 \\
\quad RF++~\cite{RestoreFormerRealWorldBlind2023wang} &  \colorbox{tabthird}{163.3} &  \colorbox{tabthird}{93.3} &  \colorbox{tabfirst}{0.31} & \colorbox{tabsecond}{0.200} & \colorbox{tabsecond}{22.19} & \colorbox{tabsecond}{26.7} && \colorbox{tabsecond}{187.1} & \colorbox{tabsecond}{97.0} &  \colorbox{tabfirst}{0.33} & \colorbox{tabsecond}{0.208} & \colorbox{tabsecond}{19.34} & \colorbox{tabsecond}{21.6} \\
\quad DiffBIR~\cite{DiffBIRBlindImage2024lin} &  \colorbox{tabfirst}{157.0} &  \colorbox{tabfirst}{88.1} &                      0.35 &                      0.277 &  \colorbox{tabfirst}{10.48} &  \colorbox{tabfirst}{14.6} &&  \colorbox{tabfirst}{184.4} &  \colorbox{tabfirst}{94.8} &                      0.37 &                      0.287 &  \colorbox{tabfirst}{10.65} &  \colorbox{tabfirst}{12.8} \\
\emph{Plug \& play} \\
\quad PGDiff~\cite{PGDiffGuidingDiffusion2023yang} &                      276.4 &                      204.4 &                      0.38 &                      0.309 &                      92.39 &                      95.8 &&                      314.3 &                      220.8 &                      0.39 &                      0.312 &                      77.37 &                      79.1 \\
\quad DifFace~\cite{DifFaceBlindFace2023yue} &                      208.8 &                      134.3 &                      0.39 &                      0.312 &                      71.35 &                      73.9 &&                      219.0 &                      124.4 &                      0.39 &                      0.309 &                      56.50 &                      58.2 \\
\quad ELAD (Ours) &                      191.4 &                      117.3 &                      0.38 &                      0.299 &                      56.91 &                      59.6 &&                      202.8 &  \colorbox{tabthird}{108.6} &                      0.39 &                      0.298 &                      44.23 &                      45.9 \\

\bottomrule

    \CodeAfter
    \tikz \draw [dash pattern={on 2pt off 2pt}, shorten > = 2pt] (8.5-|2) -- (8.5-|15) ;
    \end{NiceTabular}

    \begin{NiceTabular}{@{}lcccccc@{}}\toprule
    & \multicolumn{6}{c}{\thead{CelebA-Test (WIDER Dist.)}} \\
\cmidrule{2-7}
\thead{Method} & \thead{MSE} $\downarrow$ & \thead{Proxy\\MSE} $\downarrow$ & \thead{LPIPS} $\downarrow$ & \thead{Proxy\\LPIPS} $\downarrow$ & \thead{CMSE} $\downarrow$ & \thead{Proxy\\CMSE} $\downarrow$ \\
\midrule

\quad VQFR~\cite{VQFRBlindFace2022gu} &                      302.1 &                      176.1 & \colorbox{tabsecond}{0.36} &  \colorbox{tabthird}{0.231} &                      27.40 &                      28.7 \\
\quad CodeFormer~\cite{RobustBlindFace2022zhou} &                      237.7 &                      110.3 &  \colorbox{tabfirst}{0.35} &  \colorbox{tabfirst}{0.213} &  \colorbox{tabthird}{19.08} &  \colorbox{tabthird}{20.5} \\
\quad GFPGAN~\cite{RealWorldBlindFace2021wang} &                      281.5 &                      152.6 &  \colorbox{tabthird}{0.37} &                      0.242 &                      53.47 &                      54.7 \\
\quad RF++~\cite{RestoreFormerRealWorldBlind2023wang} &  \colorbox{tabthird}{232.1} &  \colorbox{tabthird}{108.7} &  \colorbox{tabfirst}{0.35} & \colorbox{tabsecond}{0.224} & \colorbox{tabsecond}{10.21} & \colorbox{tabsecond}{11.6} \\
\quad DiffBIR~\cite{DiffBIRBlindImage2024lin} & \colorbox{tabsecond}{225.3} & \colorbox{tabsecond}{102.1} &                      0.39 &                      0.303 &  \colorbox{tabfirst}{4.77} &  \colorbox{tabfirst}{6.2} \\
\emph{Plug \& play} \\
\quad PGDiff~\cite{PGDiffGuidingDiffusion2023yang} &                      372.5 &                      245.0 &                      0.41 &                      0.325 &                      51.30 &                      52.1 \\
\quad DifFace~\cite{DifFaceBlindFace2023yue} &                      237.3 &                      109.2 &                      0.40 &                      0.306 &                      33.10 &                      34.2 \\
\quad ELAD (Ours) &  \colorbox{tabfirst}{224.7} &  \colorbox{tabfirst}{96.8} &                      0.40 &                      0.298 &                      23.91 &                      25.0 \\

\bottomrule

    \CodeAfter
    \tikz \draw [dash pattern={on 2pt off 2pt}, shorten > = 2pt] (8.5-|2) -- (8.5-|8) ;
    \end{NiceTabular}
    
    \caption{Quantitative comparison between BFR methods on synthetic datasets defined in \Cref{sec:deg_est}.
    We highlight the \rectlabel{tabfirst}{first}, \rectlabel{tabsecond}{second}, and \rectlabel{tabthird}{third} best-performing methods in each measure. We divide between end-to-end methods and plug-and-play ones. ``Uniform Dist. \#1'' is used in \Cref{fig:bfr_rmse_id}, while ``Uniform Dist. \#2'' is used in \Cref{fig:bfr_crmse}. ``Uniform Dist. \#1'' is the  dataset introduced in \cite{RealWorldBlindFace2021wang}, which does not contain the degradations used, hence we cannot compute CMSE values.}
    \label{tab:celeba}
\end{table*}

\begingroup
\newcolumntype{M}[1]{>{\centering\arraybackslash}m{#1}}
\newcommand{\vcentered}[1]{\begin{tabular}{@{}l@{}} #1 \end{tabular}}
\setlength{\tabcolsep}{0pt} %
\renewcommand{\arraystretch}{0} %

\def\columns{5}
\def\totalwidth{0.9}

\FPeval{\colwidth}{clip(\totalwidth/\columns)}
\FPeval{\doublecolwidth}{clip(2*\totalwidth/\columns)}
\FPeval{\imgwidth}{\totalwidth/\columns *\columns}

\newcommand{\centered}[1]{\begin{tabular}{l} #1 \end{tabular}}

\begin{figure*}[tb]
    \centering
    \begin{tabular}{M{\colwidth\linewidth} M{\colwidth\linewidth} M{\colwidth\linewidth} M{\colwidth\linewidth} M{\colwidth\linewidth}}
    
	    \footnotesize{Degraded} &
	    \footnotesize{ELAD (Ours)} &
	    \footnotesize{DifFace~\cite{DifFaceBlindFace2023yue}} &
	    \footnotesize{PGDiff~\cite{PGDiffGuidingDiffusion2023yang}} &
	    \footnotesize{PMRF~\cite{PosteriorMeanRectifiedFlow2024ohayon}}\\
	    
        \rule{0pt}{0.8ex}\\

        \multicolumn{\columns}{c}{\centered{\includegraphics[width=\totalwidth\linewidth]{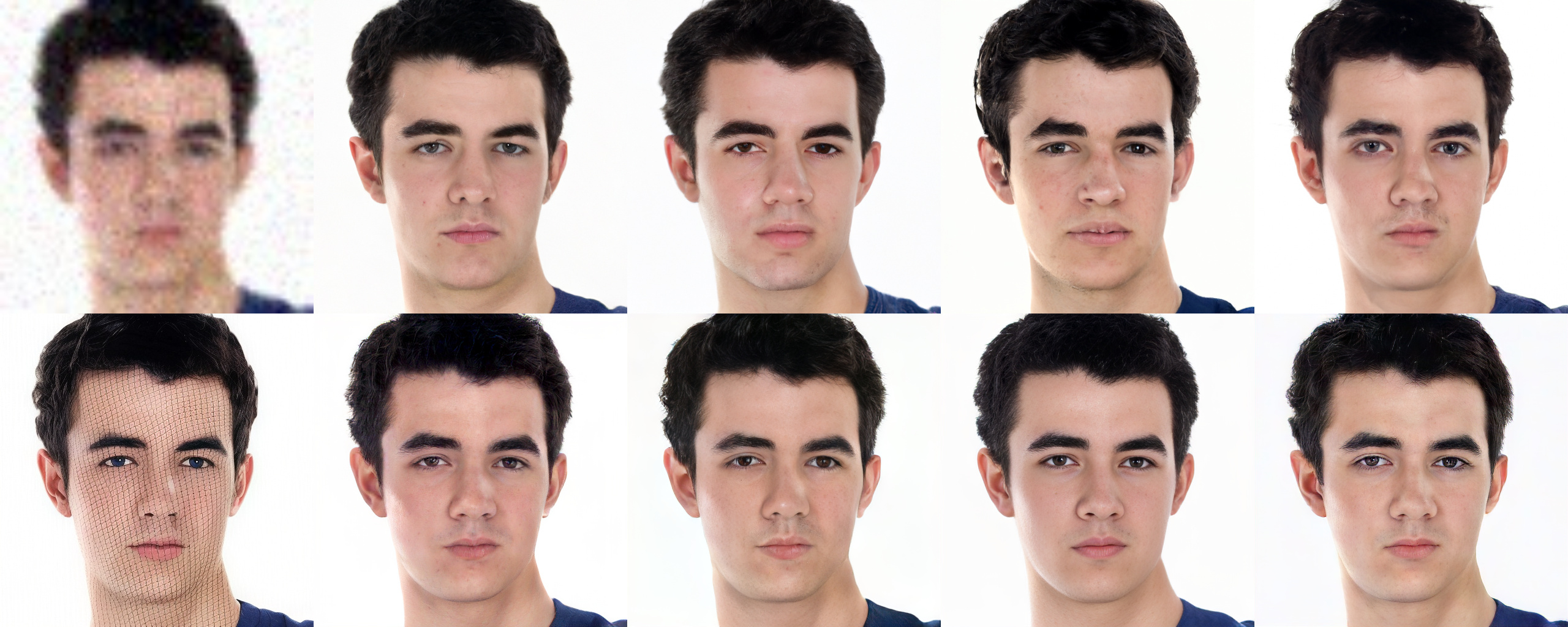}}} \\

        \rule{0pt}{0.8ex}\\

        \footnotesize{DiffBIR~\cite{DiffBIRBlindImage2024lin}} &
	    \footnotesize{RF++~\cite{RestoreFormerRealWorldBlind2023wang}} &
	    \footnotesize{GFPGAN~\cite{RealWorldBlindFace2024chen}} &
	    \footnotesize{CodeFormer~\cite{RobustBlindFace2022zhou}} &
	    \footnotesize{VQFR~\cite{VQFRBlindFace2022gu}}\\

        \rule{0pt}{1.6ex}\\

        \footnotesize{Degraded} &
	    \footnotesize{ELAD (Ours)} &
	    \footnotesize{DifFace~\cite{DifFaceBlindFace2023yue}} &
	    \footnotesize{PGDiff~\cite{PGDiffGuidingDiffusion2023yang}} &
	    \footnotesize{PMRF~\cite{PosteriorMeanRectifiedFlow2024ohayon}}\\
	    
        \rule{0pt}{0.8ex}\\

        \multicolumn{\columns}{c}{\centered{\includegraphics[width=\totalwidth\linewidth]{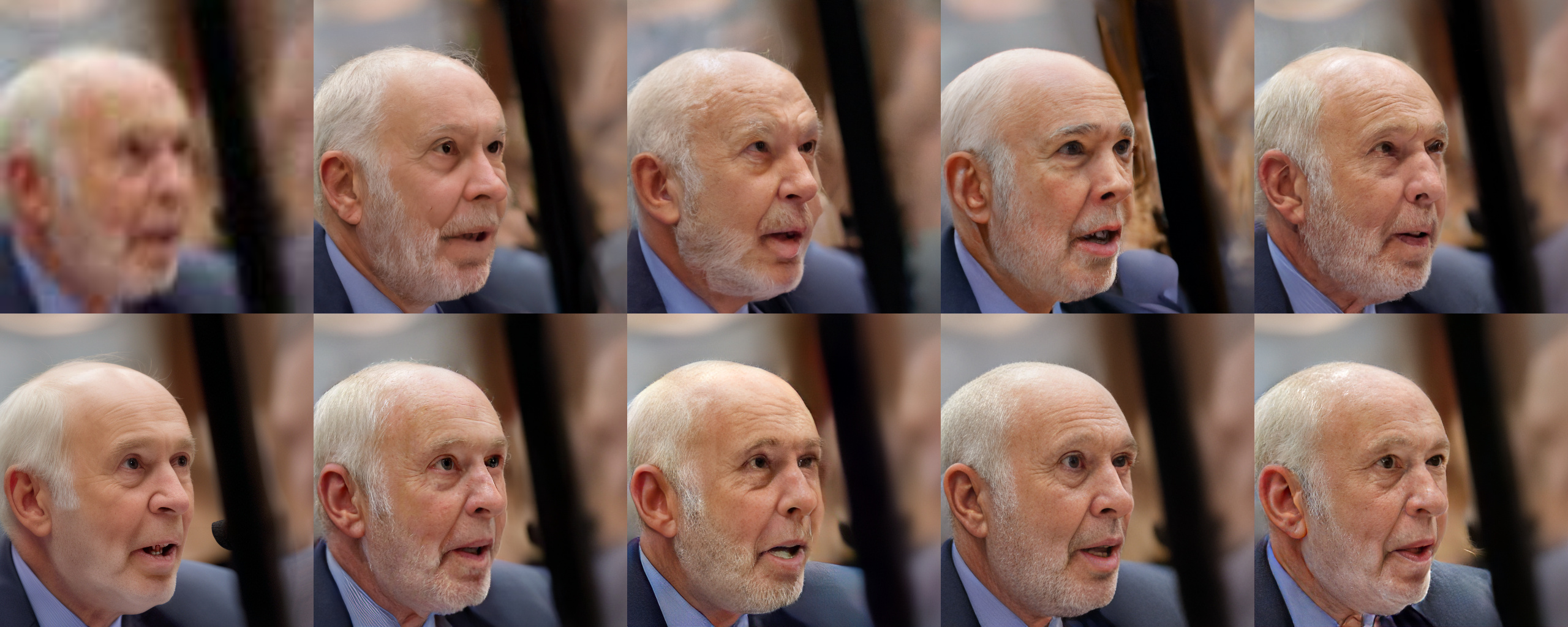}}} \\

        \rule{0pt}{0.8ex}\\

        \footnotesize{DiffBIR~\cite{DiffBIRBlindImage2024lin}} &
	    \footnotesize{RF++~\cite{RestoreFormerRealWorldBlind2023wang}} &
	    \footnotesize{GFPGAN~\cite{RealWorldBlindFace2024chen}} &
	    \footnotesize{CodeFormer~\cite{RobustBlindFace2022zhou}} &
	    \footnotesize{VQFR~\cite{VQFRBlindFace2022gu}}\\

        \rule{0pt}{1.6ex}\\

        \footnotesize{Degraded} &
	    \footnotesize{ELAD (Ours)} &
	    \footnotesize{DifFace~\cite{DifFaceBlindFace2023yue}} &
	    \footnotesize{PGDiff~\cite{PGDiffGuidingDiffusion2023yang}} &
	    \footnotesize{PMRF~\cite{PosteriorMeanRectifiedFlow2024ohayon}}\\
	    
        \rule{0pt}{0.8ex}\\

        \multicolumn{\columns}{c}{\centered{\includegraphics[width=\totalwidth\linewidth]{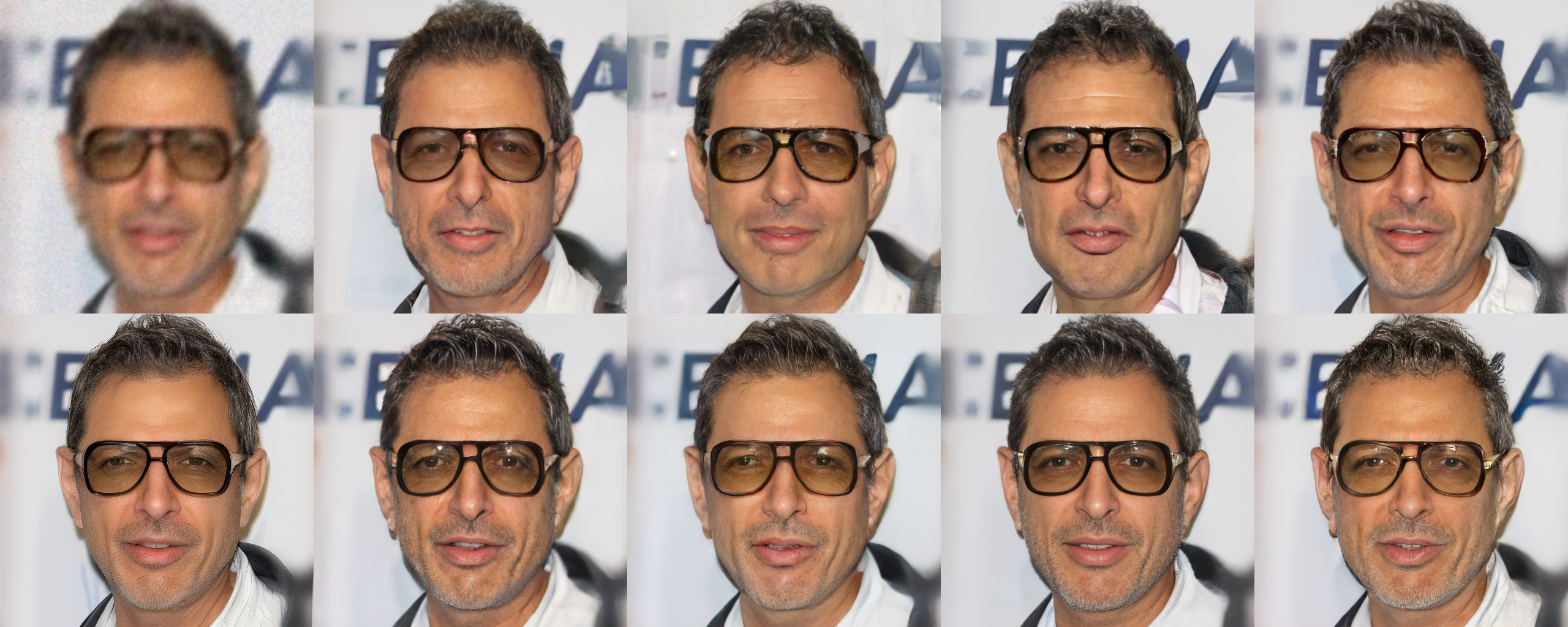}}} \\

        \rule{0pt}{0.8ex}\\

        \footnotesize{DiffBIR~\cite{DiffBIRBlindImage2024lin}} &
	    \footnotesize{RF++~\cite{RestoreFormerRealWorldBlind2023wang}} &
	    \footnotesize{GFPGAN~\cite{RealWorldBlindFace2024chen}} &
	    \footnotesize{CodeFormer~\cite{RobustBlindFace2022zhou}} &
	    \footnotesize{VQFR~\cite{VQFRBlindFace2022gu}}\\
	    
    \end{tabular}
    \caption{Restoration examples on CelebA-Test~\cite{RealWorldBlindFace2021wang}.}
    \label{fig:images_celeba}
\end{figure*}
\endgroup

\begingroup
\newcolumntype{M}[1]{>{\centering\arraybackslash}m{#1}}
\newcommand{\vcentered}[1]{\begin{tabular}{@{}l@{}} #1 \end{tabular}}
\setlength{\tabcolsep}{0pt} %
\renewcommand{\arraystretch}{0} %

\def\columns{5}
\def\totalwidth{0.9}

\FPeval{\colwidth}{clip(\totalwidth/\columns)}
\FPeval{\doublecolwidth}{clip(2*\totalwidth/\columns)}
\FPeval{\imgwidth}{\totalwidth/\columns *\columns}

\newcommand{\centered}[1]{\begin{tabular}{l} #1 \end{tabular}}

\begin{figure*}[tb]
    \centering
    \begin{tabular}{M{\colwidth\linewidth} M{\colwidth\linewidth} M{\colwidth\linewidth} M{\colwidth\linewidth} M{\colwidth\linewidth}}
    
	    \footnotesize{Degraded} &
	    \footnotesize{ELAD (Ours)} &
	    \footnotesize{DifFace~\cite{DifFaceBlindFace2023yue}} &
	    \footnotesize{PGDiff~\cite{PGDiffGuidingDiffusion2023yang}} &
	    \footnotesize{PMRF~\cite{PosteriorMeanRectifiedFlow2024ohayon}}\\
	    
        \rule{0pt}{0.8ex}\\

        \multicolumn{\columns}{c}{\centered{\includegraphics[width=\totalwidth\linewidth]{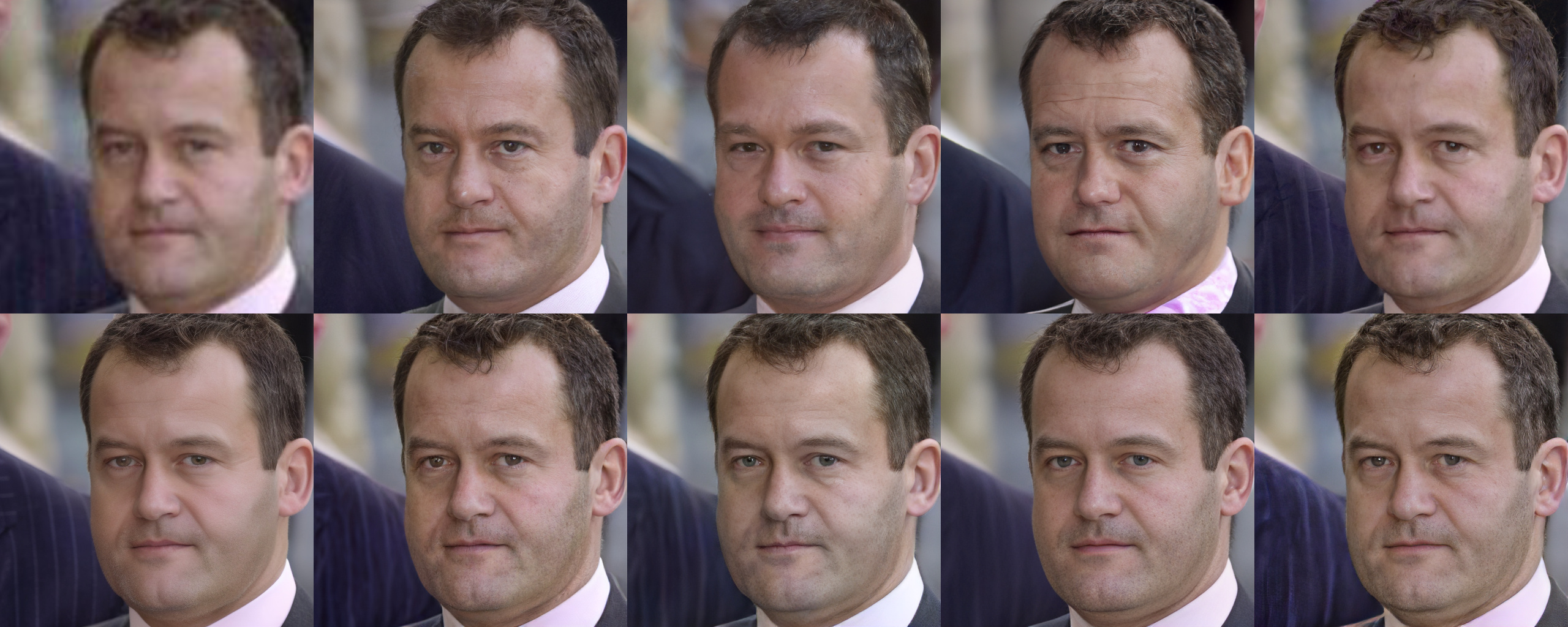}}} \\

        \rule{0pt}{0.8ex}\\

        \footnotesize{DiffBIR~\cite{DiffBIRBlindImage2024lin}} &
	    \footnotesize{RF++~\cite{RestoreFormerRealWorldBlind2023wang}} &
	    \footnotesize{GFPGAN~\cite{RealWorldBlindFace2024chen}} &
	    \footnotesize{CodeFormer~\cite{RobustBlindFace2022zhou}} &
	    \footnotesize{VQFR~\cite{VQFRBlindFace2022gu}}\\

        \rule{0pt}{1.6ex}\\

        \footnotesize{Degraded} &
	    \footnotesize{ELAD (Ours)} &
	    \footnotesize{DifFace~\cite{DifFaceBlindFace2023yue}} &
	    \footnotesize{PGDiff~\cite{PGDiffGuidingDiffusion2023yang}} &
	    \footnotesize{PMRF~\cite{PosteriorMeanRectifiedFlow2024ohayon}}\\
	    
        \rule{0pt}{0.8ex}\\

        \multicolumn{\columns}{c}{\centered{\includegraphics[width=\totalwidth\linewidth]{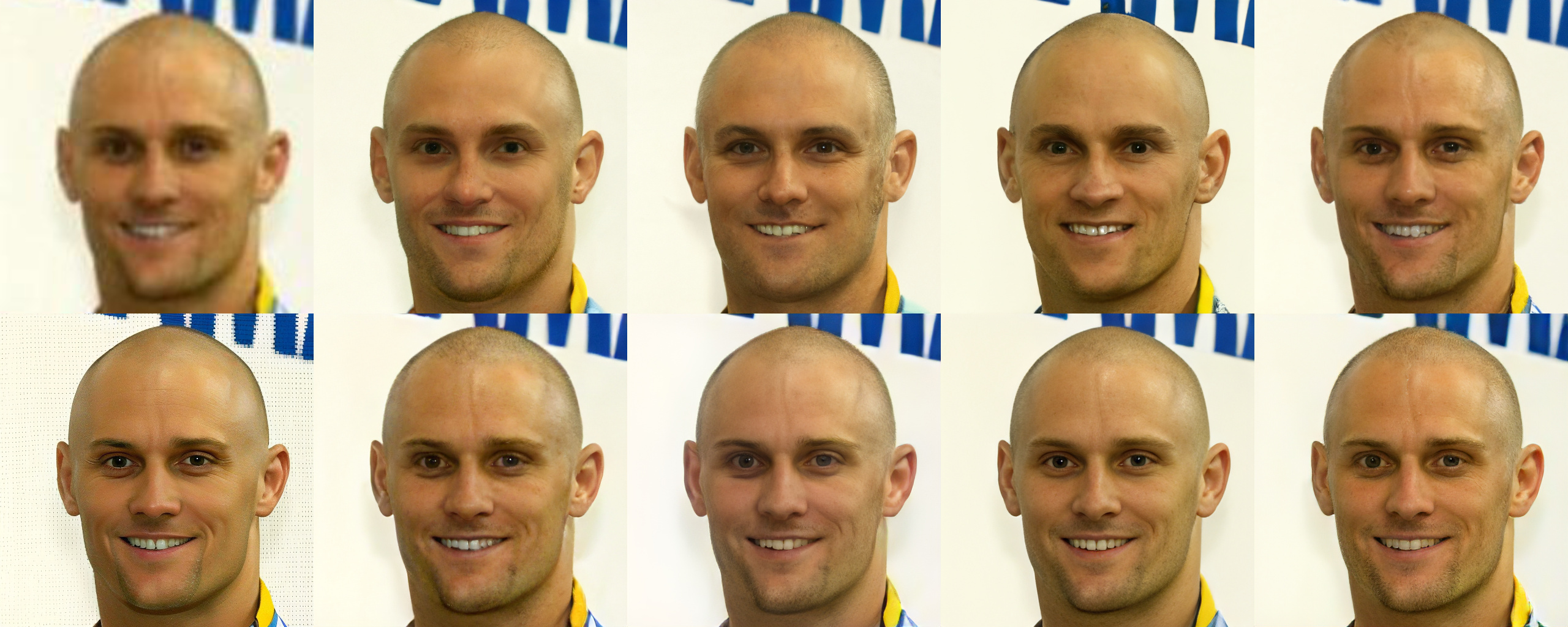}}} \\

        \rule{0pt}{0.8ex}\\

        \footnotesize{DiffBIR~\cite{DiffBIRBlindImage2024lin}} &
	    \footnotesize{RF++~\cite{RestoreFormerRealWorldBlind2023wang}} &
	    \footnotesize{GFPGAN~\cite{RealWorldBlindFace2024chen}} &
	    \footnotesize{CodeFormer~\cite{RobustBlindFace2022zhou}} &
	    \footnotesize{VQFR~\cite{VQFRBlindFace2022gu}}\\

        \rule{0pt}{1.6ex}\\

        \footnotesize{Degraded} &
	    \footnotesize{ELAD (Ours)} &
	    \footnotesize{DifFace~\cite{DifFaceBlindFace2023yue}} &
	    \footnotesize{PGDiff~\cite{PGDiffGuidingDiffusion2023yang}} &
	    \footnotesize{PMRF~\cite{PosteriorMeanRectifiedFlow2024ohayon}}\\
	    
        \rule{0pt}{0.8ex}\\

        \multicolumn{\columns}{c}{\centered{\includegraphics[width=\totalwidth\linewidth]{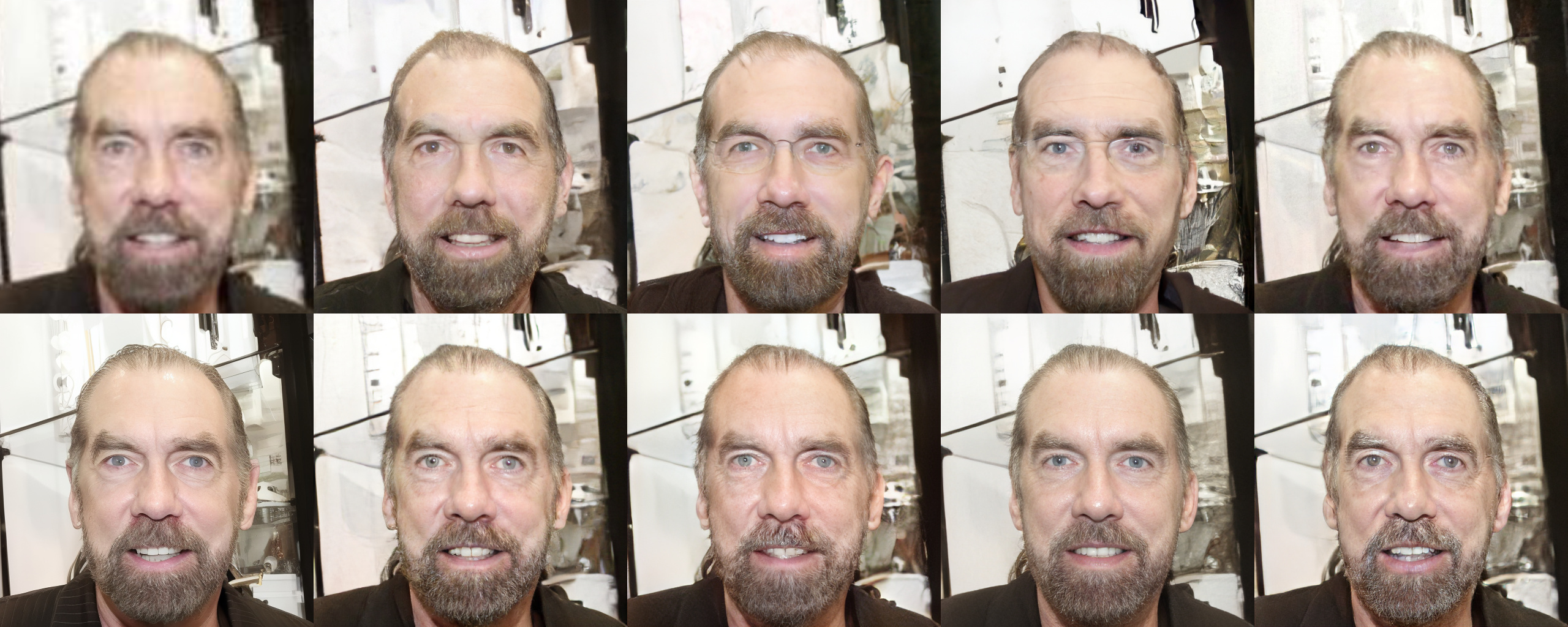}}} \\

        \rule{0pt}{0.8ex}\\

        \footnotesize{DiffBIR~\cite{DiffBIRBlindImage2024lin}} &
	    \footnotesize{RF++~\cite{RestoreFormerRealWorldBlind2023wang}} &
	    \footnotesize{GFPGAN~\cite{RealWorldBlindFace2024chen}} &
	    \footnotesize{CodeFormer~\cite{RobustBlindFace2022zhou}} &
	    \footnotesize{VQFR~\cite{VQFRBlindFace2022gu}}\\
	    
    \end{tabular}
    \caption{Restoration examples on LFW-Test~\cite{LabeledFacesWild2007huang}.}
    \label{fig:images_lfw}
\end{figure*}
\endgroup

\begingroup
\newcolumntype{M}[1]{>{\centering\arraybackslash}m{#1}}
\newcommand{\vcentered}[1]{\begin{tabular}{@{}l@{}} #1 \end{tabular}}
\setlength{\tabcolsep}{0pt} %
\renewcommand{\arraystretch}{0} %

\def\columns{5}
\def\totalwidth{0.9}

\FPeval{\colwidth}{clip(\totalwidth/\columns)}
\FPeval{\doublecolwidth}{clip(2*\totalwidth/\columns)}
\FPeval{\imgwidth}{\totalwidth/\columns *\columns}

\newcommand{\centered}[1]{\begin{tabular}{l} #1 \end{tabular}}

\begin{figure*}[tb]
    \centering
    \begin{tabular}{M{\colwidth\linewidth} M{\colwidth\linewidth} M{\colwidth\linewidth} M{\colwidth\linewidth} M{\colwidth\linewidth}}
    
	    \footnotesize{Degraded} &
	    \footnotesize{ELAD (Ours)} &
	    \footnotesize{DifFace~\cite{DifFaceBlindFace2023yue}} &
	    \footnotesize{PGDiff~\cite{PGDiffGuidingDiffusion2023yang}} &
	    \footnotesize{PMRF~\cite{PosteriorMeanRectifiedFlow2024ohayon}}\\
	    
        \rule{0pt}{0.8ex}\\

        \multicolumn{\columns}{c}{\centered{\includegraphics[width=\totalwidth\linewidth]{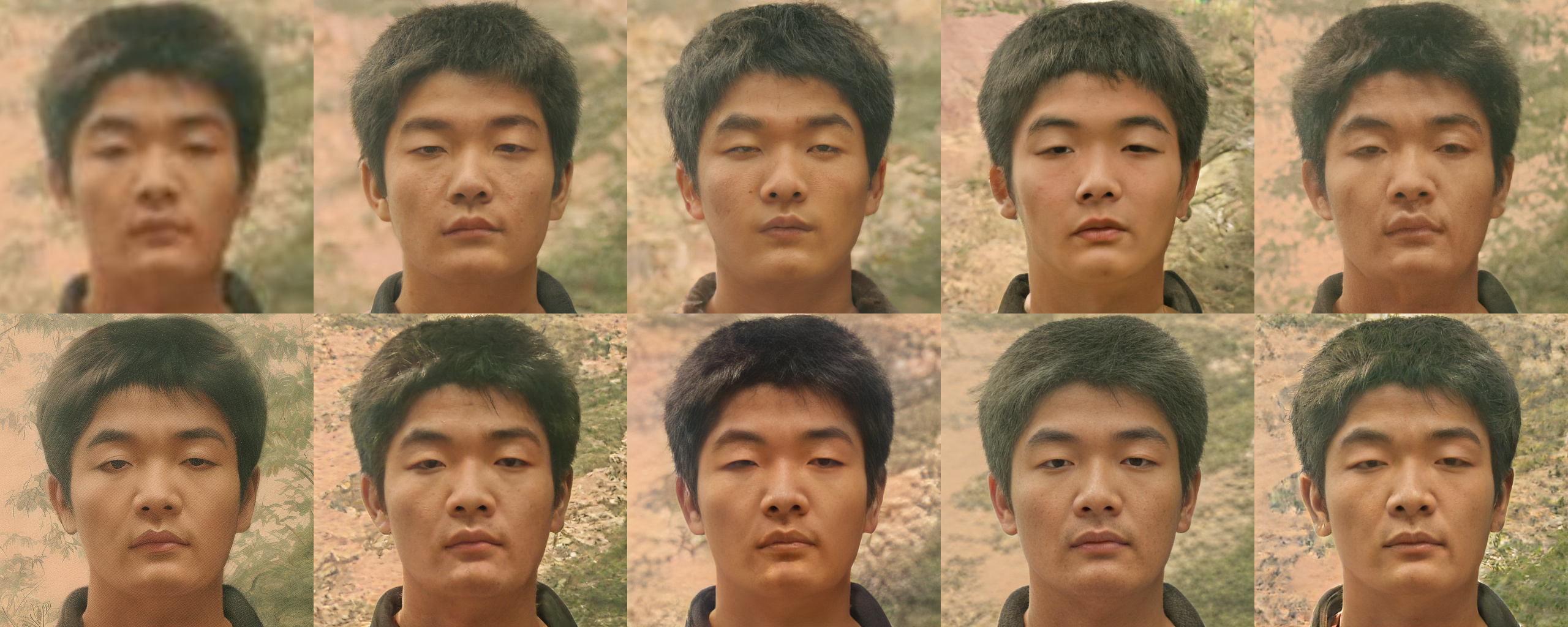}}} \\

        \rule{0pt}{0.8ex}\\

        \footnotesize{DiffBIR~\cite{DiffBIRBlindImage2024lin}} &
	    \footnotesize{RF++~\cite{RestoreFormerRealWorldBlind2023wang}} &
	    \footnotesize{GFPGAN~\cite{RealWorldBlindFace2024chen}} &
	    \footnotesize{CodeFormer~\cite{RobustBlindFace2022zhou}} &
	    \footnotesize{VQFR~\cite{VQFRBlindFace2022gu}}\\

        \rule{0pt}{1.6ex}\\

        \footnotesize{Degraded} &
	    \footnotesize{ELAD (Ours)} &
	    \footnotesize{DifFace~\cite{DifFaceBlindFace2023yue}} &
	    \footnotesize{PGDiff~\cite{PGDiffGuidingDiffusion2023yang}} &
	    \footnotesize{PMRF~\cite{PosteriorMeanRectifiedFlow2024ohayon}}\\
	    
        \rule{0pt}{0.8ex}\\

        \multicolumn{\columns}{c}{\centered{\includegraphics[width=\totalwidth\linewidth]{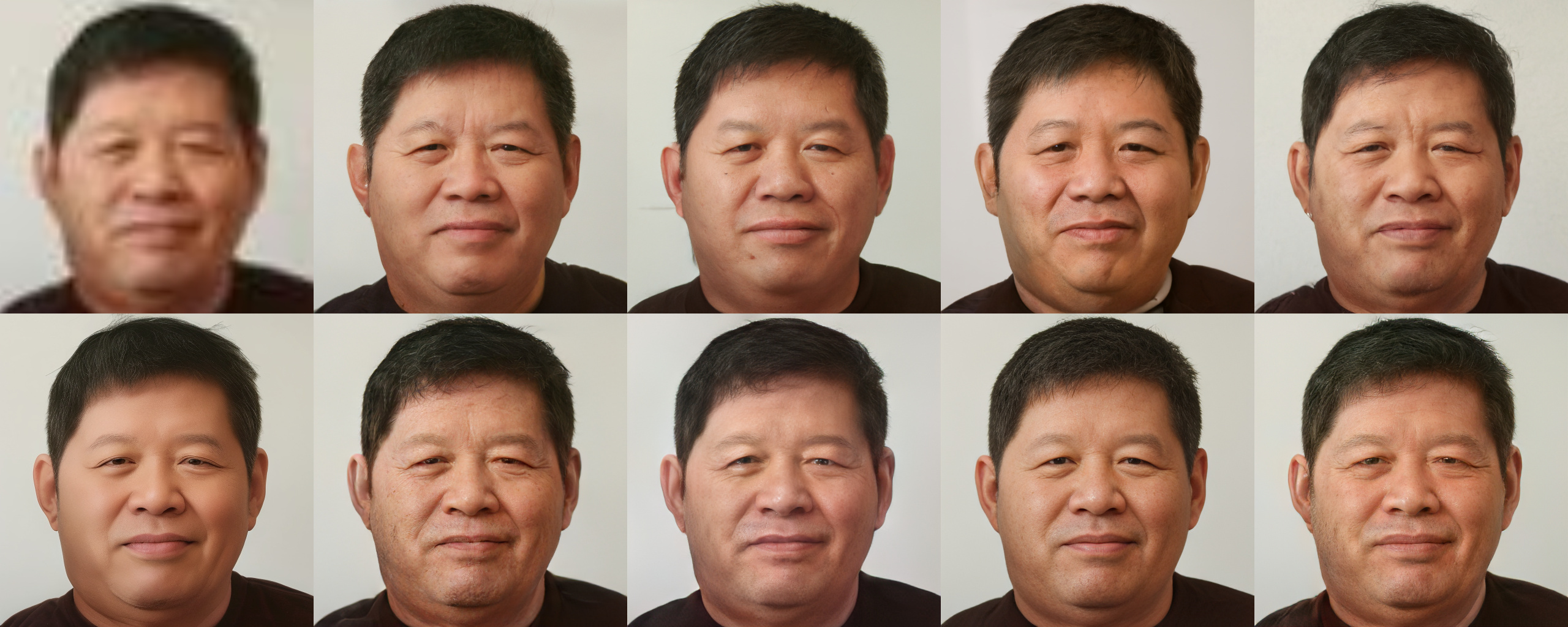}}} \\

        \rule{0pt}{0.8ex}\\

        \footnotesize{DiffBIR~\cite{DiffBIRBlindImage2024lin}} &
	    \footnotesize{RF++~\cite{RestoreFormerRealWorldBlind2023wang}} &
	    \footnotesize{GFPGAN~\cite{RealWorldBlindFace2024chen}} &
	    \footnotesize{CodeFormer~\cite{RobustBlindFace2022zhou}} &
	    \footnotesize{VQFR~\cite{VQFRBlindFace2022gu}}\\

        \rule{0pt}{1.6ex}\\

        \footnotesize{Degraded} &
	    \footnotesize{ELAD (Ours)} &
	    \footnotesize{DifFace~\cite{DifFaceBlindFace2023yue}} &
	    \footnotesize{PGDiff~\cite{PGDiffGuidingDiffusion2023yang}} &
	    \footnotesize{PMRF~\cite{PosteriorMeanRectifiedFlow2024ohayon}}\\
	    
        \rule{0pt}{0.8ex}\\

        \multicolumn{\columns}{c}{\centered{\includegraphics[width=\totalwidth\linewidth]{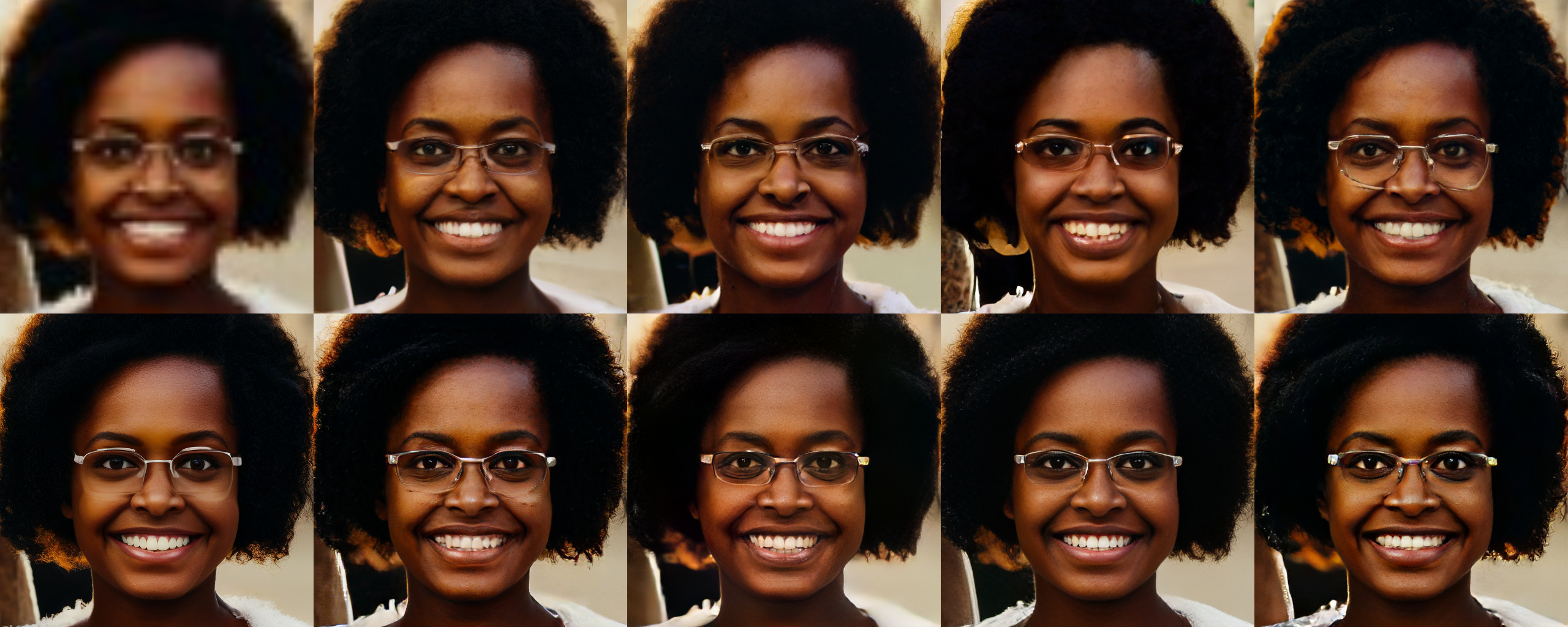}}} \\

        \rule{0pt}{0.8ex}\\

        \footnotesize{DiffBIR~\cite{DiffBIRBlindImage2024lin}} &
	    \footnotesize{RF++~\cite{RestoreFormerRealWorldBlind2023wang}} &
	    \footnotesize{GFPGAN~\cite{RealWorldBlindFace2024chen}} &
	    \footnotesize{CodeFormer~\cite{RobustBlindFace2022zhou}} &
	    \footnotesize{VQFR~\cite{VQFRBlindFace2022gu}}\\
	    
    \end{tabular}
    \caption{Restoration examples on WebPhot-Test~\cite{RealWorldBlindFace2021wang}.}
    \label{fig:images_webphoto}
\end{figure*}
\endgroup

\begingroup
\newcolumntype{M}[1]{>{\centering\arraybackslash}m{#1}}
\newcommand{\vcentered}[1]{\begin{tabular}{@{}l@{}} #1 \end{tabular}}
\setlength{\tabcolsep}{0pt} %
\renewcommand{\arraystretch}{0} %

\def\columns{5}
\def\totalwidth{0.9}

\FPeval{\colwidth}{clip(\totalwidth/\columns)}
\FPeval{\doublecolwidth}{clip(2*\totalwidth/\columns)}
\FPeval{\imgwidth}{\totalwidth/\columns *\columns}

\newcommand{\centered}[1]{\begin{tabular}{l} #1 \end{tabular}}

\begin{figure*}[tb]
    \centering
    \begin{tabular}{M{\colwidth\linewidth} M{\colwidth\linewidth} M{\colwidth\linewidth} M{\colwidth\linewidth} M{\colwidth\linewidth}}
    
	    \footnotesize{Degraded} &
	    \footnotesize{ELAD (Ours)} &
	    \footnotesize{DifFace~\cite{DifFaceBlindFace2023yue}} &
	    \footnotesize{PGDiff~\cite{PGDiffGuidingDiffusion2023yang}} &
	    \footnotesize{PMRF~\cite{PosteriorMeanRectifiedFlow2024ohayon}}\\
	    
        \rule{0pt}{0.8ex}\\

        \multicolumn{\columns}{c}{\centered{\includegraphics[width=\totalwidth\linewidth]{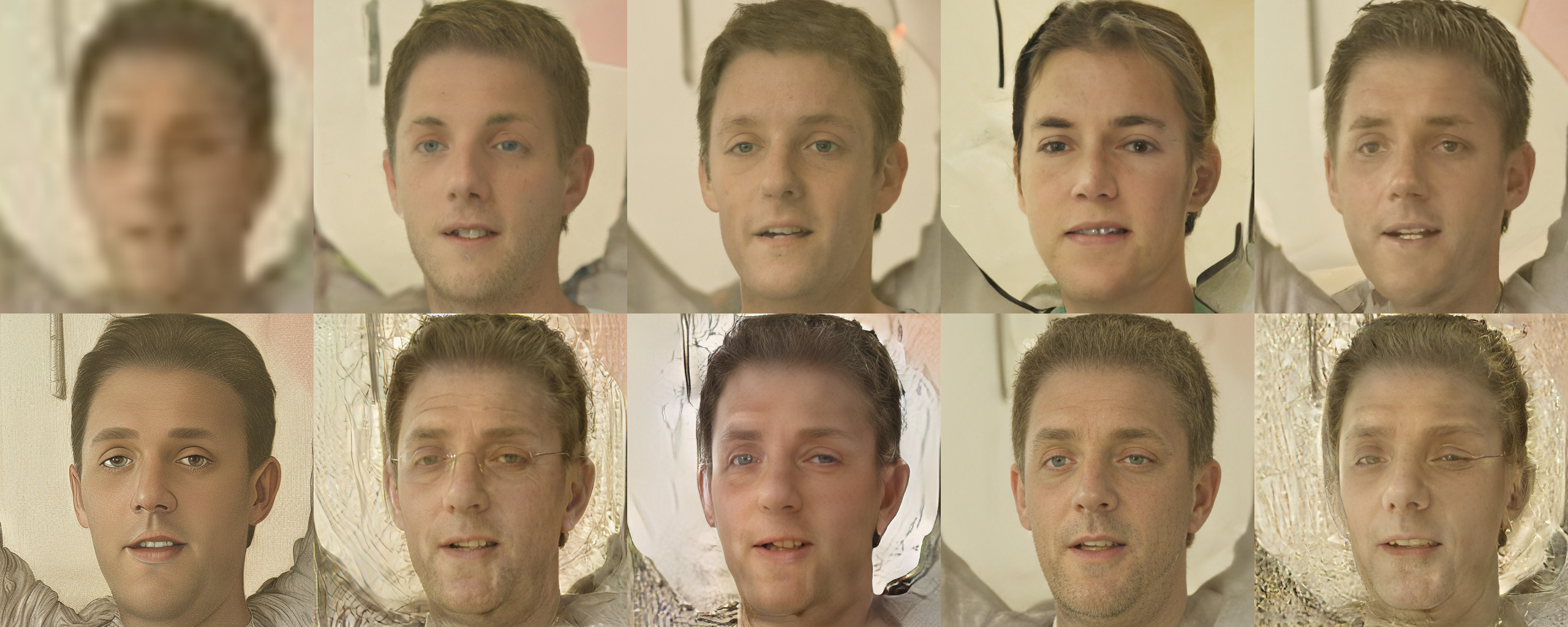}}} \\

        \rule{0pt}{0.8ex}\\

        \footnotesize{DiffBIR~\cite{DiffBIRBlindImage2024lin}} &
	    \footnotesize{RF++~\cite{RestoreFormerRealWorldBlind2023wang}} &
	    \footnotesize{GFPGAN~\cite{RealWorldBlindFace2024chen}} &
	    \footnotesize{CodeFormer~\cite{RobustBlindFace2022zhou}} &
	    \footnotesize{VQFR~\cite{VQFRBlindFace2022gu}}\\

        \rule{0pt}{1.6ex}\\

        \footnotesize{Degraded} &
	    \footnotesize{ELAD (Ours)} &
	    \footnotesize{DifFace~\cite{DifFaceBlindFace2023yue}} &
	    \footnotesize{PGDiff~\cite{PGDiffGuidingDiffusion2023yang}} &
	    \footnotesize{PMRF~\cite{PosteriorMeanRectifiedFlow2024ohayon}}\\
	    
        \rule{0pt}{0.8ex}\\

        \multicolumn{\columns}{c}{\centered{\includegraphics[width=\totalwidth\linewidth]{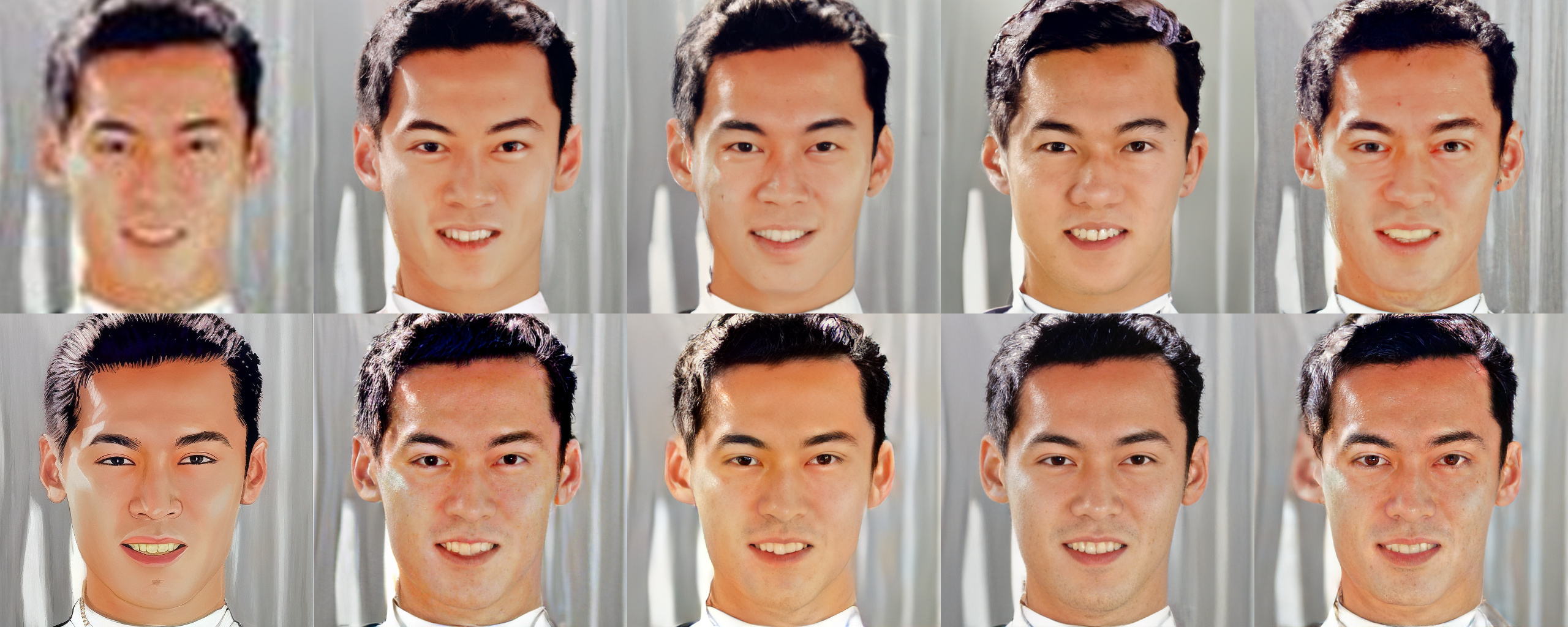}}} \\

        \rule{0pt}{0.8ex}\\

        \footnotesize{DiffBIR~\cite{DiffBIRBlindImage2024lin}} &
	    \footnotesize{RF++~\cite{RestoreFormerRealWorldBlind2023wang}} &
	    \footnotesize{GFPGAN~\cite{RealWorldBlindFace2024chen}} &
	    \footnotesize{CodeFormer~\cite{RobustBlindFace2022zhou}} &
	    \footnotesize{VQFR~\cite{VQFRBlindFace2022gu}}\\

        \rule{0pt}{1.6ex}\\

        \footnotesize{Degraded} &
	    \footnotesize{ELAD (Ours)} &
	    \footnotesize{DifFace~\cite{DifFaceBlindFace2023yue}} &
	    \footnotesize{PGDiff~\cite{PGDiffGuidingDiffusion2023yang}} &
	    \footnotesize{PMRF~\cite{PosteriorMeanRectifiedFlow2024ohayon}}\\
	    
        \rule{0pt}{0.8ex}\\

        \multicolumn{\columns}{c}{\centered{\includegraphics[width=\totalwidth\linewidth]{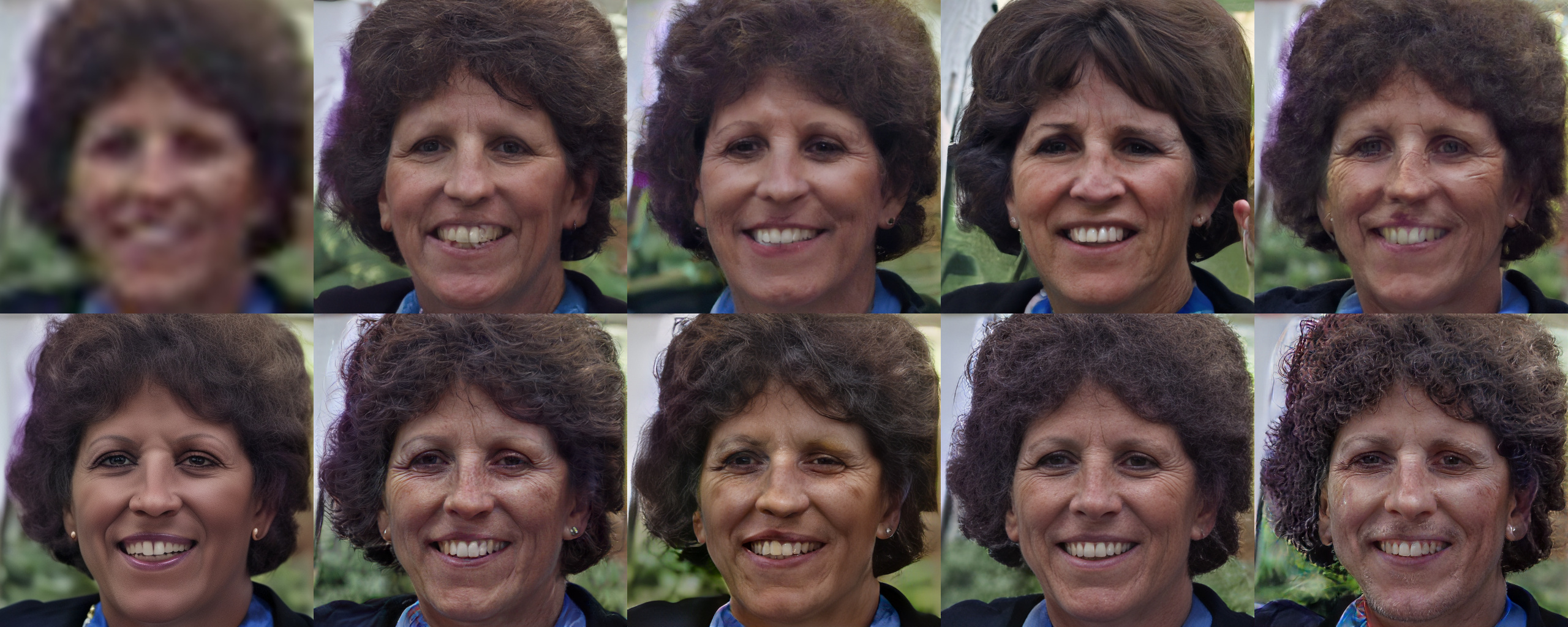}}} \\

        \rule{0pt}{0.8ex}\\

        \footnotesize{DiffBIR~\cite{DiffBIRBlindImage2024lin}} &
	    \footnotesize{RF++~\cite{RestoreFormerRealWorldBlind2023wang}} &
	    \footnotesize{GFPGAN~\cite{RealWorldBlindFace2024chen}} &
	    \footnotesize{CodeFormer~\cite{RobustBlindFace2022zhou}} &
	    \footnotesize{VQFR~\cite{VQFRBlindFace2022gu}}\\
	    
    \end{tabular}
    \caption{Restoration examples on WIDER-Test~\cite{WIDERFACEFace2016yang}.}
    \label{fig:images_wider}
\end{figure*}
\endgroup

\section{Implementation details}
\label{sec:imp_supp}

\subsection{Degradation estimator}
\label{sec:deg_est_supp}

Our degradation estimator consists of a ConvNext-Large~\cite{ConvNet2020s2022liua} for feature extraction, two convolutional layers to sub-sample the extracted features, followed by an MLP head to transform the feature into four parameter values.
The ConvNext is initialized from a robust ImageNet~\cite{ImageNetLargescaleHierarchical2009denga} checkpoint~\cite{ComprehensiveStudyRobustness2024liu}.
The architecture is illustrated in \cref{fig:deg_est_arch}.
We train the estimator on the FFHQ dataset~\cite{StyleBasedGeneratorArchitecture2019karrasb} at $512{\times}512$ resolution for 300K steps and batch size 16 using the AdamW~\cite{DecoupledWeightDecay2018loshchilov} optimizer ($\beta_1{=}0.9$, $\beta_2 {=} 0.999$, $\epsilon {=} 10^{-8}$) with a constant learning rate of $2.5{\times}10^{-5}$, weight decay of $0.05$, linear warmup for 1K steps, and linear cooldown for $20\%$ steps~\cite{ScalingLawsComputeOptimal2024hagele} (for a total of 360K steps).
\begin{figure}[tbh]
    \centering
    \includegraphics[width=0.8\linewidth]{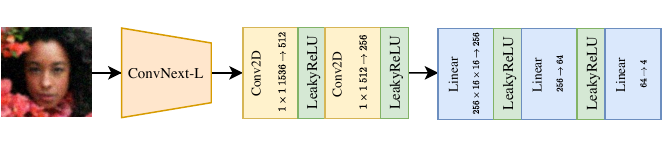}
    \caption{Degradation estimator architecture for BFR.}
    \label{fig:deg_est_arch}
\end{figure}

\subsection{ELAD}
\label{sec:elad_supp}

We detail in \Cref{alg:elad} \emph{Empirical Likelihood Approximation with Diffusion prior} (\textbf{ELAD}).
We start with a concise summary of diffusion model notations used in the algorithm, based on DDIM~\cite{DenoisingDiffusionImplicit2020song}, followed by the design choices taken in ELAD.

\paragraph{Notations.}
Denote the probability density function of clean images by $q(\vx_0)$.
We construct a forward noising process as a Markov chain,
\begin{align}
    q(\vx_t|\vx_{t-1}) = \mathcal{N} (\sqrt{1-\beta_t} \vx_{t-1}, \beta_t \mI),
\end{align}
where $t\in[1,\ldots,T]$ is a timestep and $\{\beta\}_{t=1}^T$ is a noise scheduler and $q(\vx_T)\approx\mathcal{N}(\vzero, \mI)$.
We can sample from $q(\vx_t|\vx_0)$ directly by noting that
\begin{align}
    q(\vx_t|\vx_0) = \mathcal{N} (\sqrt{1-\bar\alpha_t} \vx_0, (1-\bar\alpha_t) \mI),
\end{align}
where $\bar\alpha_t=\prod_{s=1}^t \alpha_s$ and $\alpha_t=(1-\beta_t)$.
Under DDIM, the task of a diffusion model is to construct a reverse denoising process by learning the following densities:
\begin{align}
    p_\theta(\vx_{t-1}|\vx_t) = \mathcal{N} \lft(\sqrt{\bar\alpha_{t-1}}\hat\vx_0^t + \sqrt{1-\bar\alpha_{t-1}-\sigma_t^2}\varepsilon_\theta(\vx_t,t), \sigma_t^2\mI\rgt),
\end{align}
where $\varepsilon_\theta(\vx_t,t)$ is a time-aware MMSE denoiser, $\hat\vx_0^t=(\vx_t-\sqrt{1-\bar\alpha_t}\varepsilon_\theta(\vx_t,t))/\bar\alpha_t$, and $\sigma_t=\eta \beta_t \frac{1-\bar\alpha_{t-1}}{1-\bar\alpha_t}$.
Using $\eta\in[0,1]$, we can control the amount of random noise injected in each step.
Choosing $\eta=0$ results in a deterministic sampler, while $\eta=1$ leads to the DDPM~\cite{DenoisingDiffusionProbabilistic2020hoa} sampler.

\paragraph{Algorithm design.}
Following DPS~\cite{DiffusionPosteriorSampling2022chunga}, we guide the diffusion process by taking steps in the direction of the score-likelihood $p(\vx_t|\vy)$ at time $t$. Similar to prior work~\cite{DiffusionPosteriorSampling2022chunga,PseudoinverseGuidedDiffusionModels2022songa,DenoisingDiffusionModels2023zhua}, we approximate $p(\vx_t|\vy){\approx}p(\hat\vx_0^t|\vy)$ and compute $p(\hat\vx_0^t|\vy)$ using ELA based on the estimated degradation $\va_\vtheta(\vy)$.
Like DifFace~\cite{DifFaceBlindFace2023yue}, we perform $100$ denoising steps, starting from an intermediate timestep ($T_0=400$ out of $T=1,000$ steps) initialized by an MMSE restoration of $\vy$ for accelerated sampling.
We use a DDIM~\cite{DenoisingDiffusionImplicit2020song} sampler with $\eta=0.5$, where we recompute the predicted noise after each score-likelihood step, similar to DiffPIR~\cite{DenoisingDiffusionModels2023zhua} for increased stability.
We adapt the diffusion prior trained by DifFace~\cite{DifFaceBlindFace2023yue}. Still, we emphasize that our method does not depend on a specific diffusion prior, and the results could be improved by utilizing better priors.

As we approximate $p(\vx_t|\vy){\approx}p(\hat\vx_0^t|\vy)$, we employ a dynamic step size.
We achieve good performance using a step size equal to the ratio between the current signal-to-noise ratio (SNR) and the SNR at the initial timestep, multiplied by a constant $\lambda$ (${=}10^{-2}$ in practice).
Moreover, similar to the classifier (free) guidance literature~\cite{PhotorealisticTextImageDiffusion2022saharia}, we find that clipping the score likelihood further stabilizes the process.

To approximate the mean $\mu(\hat\vx_0^t, \va_\theta(\vy))$ we compute an empirical sample mean by sampling $16$ samples from ${p_{\rvy|\rvx,\rva}(\cdot|\rvx=\hat\vx_0^t,\rva=\va_\vtheta(\vy))}$ each denoising step.
Moreover, instead of assuming isotropic likelihood, we approximate the standard deviation empirically using the same samples, resulting in a diagonal covariance.

\RestyleAlgo{ruled}

\begingroup

\SetKwComment{Comment}{ // }{}
\newcommand\mycommfont[1]{\footnotesize\ttfamily\textcolor{cvprblue}{#1}}
\SetCommentSty{mycommfont}

\begin{algorithm*}[hbt!]
\DontPrintSemicolon
\LinesNumbered
\caption{ELAD - Blind Restoration Diffusion Sampler. Full version of \Cref{alg:elad_short}}
\label{alg:elad}
\KwData{measurement $\vy$, degradation estimator $\va_\theta$, MMSE regressor $f$, diffusion noise predictor $\varepsilon_\theta$, diffusion schedule $\{\bar\alpha_t\}_{t=1}^{T}$, start time $T_0 \leq T$, DDIM coefficient $\eta$, ELA coefficient $\lambda$}
\KwResult{a restored image $\vx_0$}
$\vx_{T_0} \sim \mathcal{N} (\sqrt{1-\bar\alpha_{T_0}} f(\vy), (1-\bar\alpha_{T_0}) \mI)$ \tcp*{noise MMSE restoration to step $T_0$}
$\hat\va = \va_\theta(\vy)$ \tcp*{predict degradation}
\For{$t=T_0$ \KwTo $1$}{
    $\hat\vx_0^t=(\vx_t-\sqrt{1-\bar\alpha_t}\varepsilon_\theta(\vx_t,t))/\bar\alpha_t$ \tcp*{perform denoising step}
    $\lambda_t = \tfrac{1}{\bar\alpha_{t-1}}\tfrac{\text{SNR}_t}{\text{SNR}_{T_0}} \lambda$ \tcp*{compute dynamic step size}
    $g = \nabla_{\hat\vx_t} \norm{\vy - \mu(\hat\vx_0^t, \hat\va)}_2^2$ \tcp*{compute score likelihood}
    $\hat\vx_0^t = \hat\vx_0^t - \lambda_t \cdot g.\text{clamp}(-1,1) $ \tcp*{perform likelihood step}
    $\hat\varepsilon = \tfrac{1}{\sqrt{1-\bar\alpha_t}} (\vx_t - \sqrt{\bar\alpha_t}\hat\vx_0^t)$ \tcp*{compute effective noise}
    $\vx_{t-1}=\text{DDIMStep}(\hat\vx_0^t, \hat\varepsilon, \eta)$ \tcp*{perform DDIM step}
}
\end{algorithm*}

\endgroup

\section{Datasets analysis \& synthesis}

\subsection{Real-world datasets analysis}

Prior work~\cite{RobustBlindFace2022zhou} considered the degradations in LFW simpler than those in WebPhoto and WIDER.
However, they could not justify those claims quantitatively as the degradations in those datasets are unknown.
Using our degradation estimator, we presented in \Cref{fig:deg_est_real_world} the estimated distribution of degradation parameters in BFR real-world datasets.
The estimation confirms that images in the LFW dataset have undergone blur and downsampling operations sampled from a narrower distribution centered on lower values than WebPhoto and WIDER distributions, leading to a simpler restoration task.
Yet, we also reveal that LFW images were compressed more aggressively compared to other datasets, an observation that might not have been easily made before.

\begingroup
\newcolumntype{M}[1]{>{\centering\arraybackslash}m{#1}}
\newcommand{\vcentered}[1]{\begin{tabular}{@{}l@{}} #1 \end{tabular}}
\setlength{\tabcolsep}{0pt} %
\renewcommand{\arraystretch}{0} %

\newcommand{\centered}[1]{\begin{tabular}{l} #1 \end{tabular}}

\newcommand{\addimgcol}[8][1]{
	\centered{
	   	\begin{tikzpicture}[
	   		baseline=-2.45,
	   		spy using outlines={magnification=#3, circle, height=#8, width=#8, yellow, every spy on node/.append style={thick}, connect spies},
	   		]
			\node[inner sep=0pt]{\scalebox{#1}[1]{\adjincludegraphics[width=\hero, trim={#4}, clip]{images/ours/#7-bsds/compressed/#2}}};
			\spy on (#5) in node at (#6);
		\end{tikzpicture}}&
		
		\centered{
		\begin{tikzpicture}[
	   		baseline=-2.45,
	   		spy using outlines={magnification=#3, circle, height=#8, width=#8, yellow, every spy on node/.append style={thick}, connect spies},
	   		]
			\node[inner sep=0pt]{\scalebox{#1}[1]{\adjincludegraphics[width=\hero, trim={#4}, clip]{images/qgac/#7-bsds/#2}}};
			\spy on (#5) in node at (#6);
		\end{tikzpicture}}&
		
		\centered{
		\begin{tikzpicture}[
	   		baseline=-2.45,
	   		spy using outlines={magnification=#3, circle, height=#8, width=#8, yellow, every spy on node/.append style={thick}, connect spies},
	   		]
			\node[inner sep=0pt]{\scalebox{#1}[1]{\adjincludegraphics[width=\hero, trim={#4}, clip]{images/qgac-gan/#7-bsds/#2}}};
			\spy on (#5) in node at (#6);
		\end{tikzpicture}}&
		
		\centered{
	   	\begin{tikzpicture}[
	   		baseline=-2.45,
	   		spy using outlines={magnification=#3, circle, height=#8, width=#8, yellow, every spy on node/.append style={thick}, connect spies},
	   		]
			\node[inner sep=0pt]{\scalebox{#1}[1]{\adjincludegraphics[width=\hero, trim={#4}, clip]{images/bahat/#7-bsds/fake_0/#2}}};
			\spy on (#5) in node at (#6);
		\end{tikzpicture}}&

		\centered{
		\begin{tikzpicture}[
	   		baseline=-2.45,
	   		spy using outlines={magnification=#3, circle, height=#8, width=#8, yellow, every spy on node/.append style={thick}, connect spies},
	   		]
			\node[inner sep=0pt]{\scalebox{#1}[1]{\adjincludegraphics[width=\hero, trim={#4}, clip]{images/ours-p/#7-bsds/fake_0/#2}}};
			\spy on (#5) in node at (#6);
		\end{tikzpicture}}&
		
		\centered{
		\begin{tikzpicture}[
	   		baseline=-2.45,
	   		spy using outlines={magnification=#3, circle, height=#8, width=#8, yellow, every spy on node/.append style={thick}, connect spies},
	   		]
			\node[inner sep=0pt]{\scalebox{#1}[1]{\adjincludegraphics[width=\hero, trim={#4}, clip]{images/ours/#7-bsds/real/#2}}};
			\spy on (#5) in node at (#6);
		\end{tikzpicture}}
		
	    \\
}

\begin{figure}[tb]
    \centering
    \begin{tabular}{c c c c c c}
    
	    \multicolumn{2}{c}{\footnotesize{LFW}} &
	    \multicolumn{2}{c}{\footnotesize{WebPhoto}} &
	    \multicolumn{2}{c}{\footnotesize{WIDER}}
	    \\

        \rule{0pt}{0.8ex}\\
        
        \cmidrule(lr){1-2}
        \cmidrule(lr){3-4}
        \cmidrule(lr){5-6}

        \footnotesize{Real} &
        \footnotesize{Synthetic} &
        \footnotesize{Real} &
        \footnotesize{Synthetic} &
        \footnotesize{Real} &
        \footnotesize{Synthetic}
        \\
	    
        \rule{0pt}{0.8ex}\\

        \centered{\includegraphics[width=0.1667\linewidth]{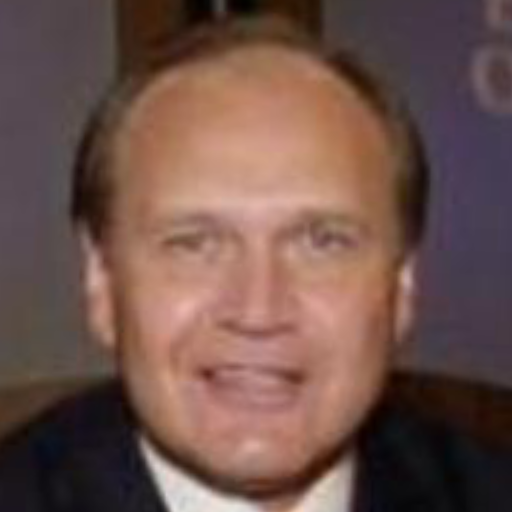}} &
        \centered{\includegraphics[width=0.1667\linewidth]{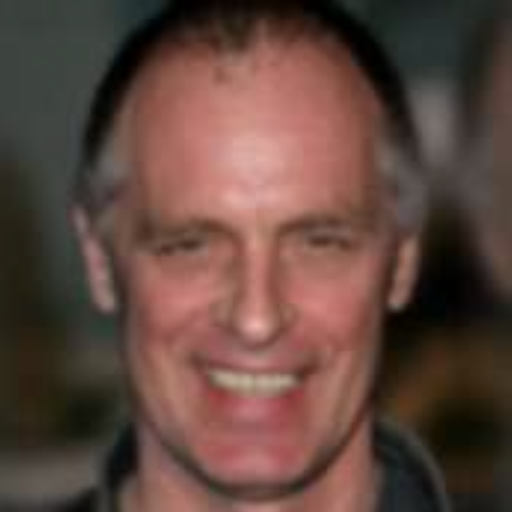}} &
        \centered{\includegraphics[width=0.1667\linewidth]{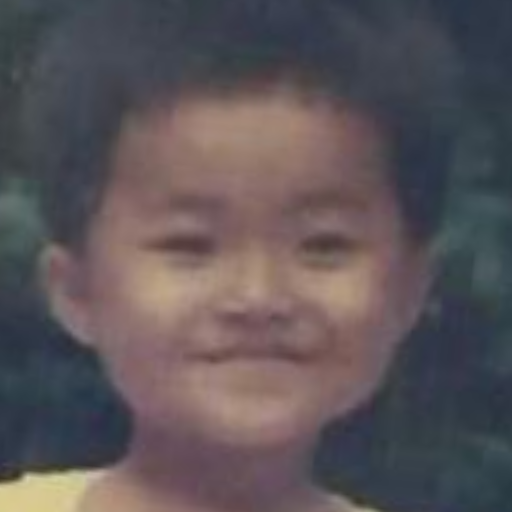}} &
        \centered{\includegraphics[width=0.1667\linewidth]{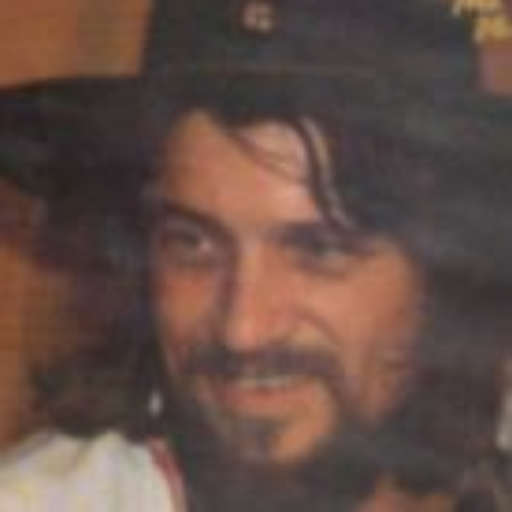}} &
        \centered{\includegraphics[width=0.1667\linewidth]{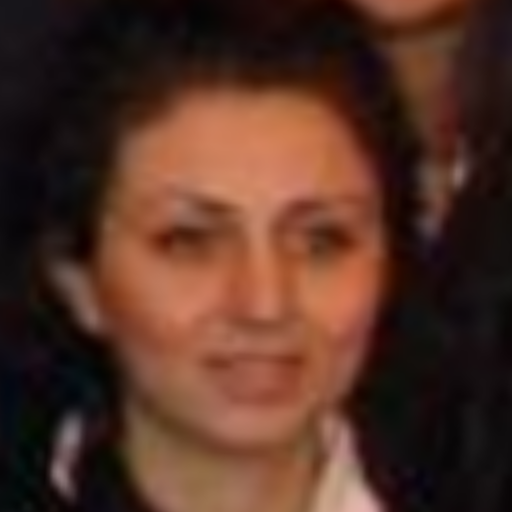}} &
        \centered{\includegraphics[width=0.1667\linewidth]{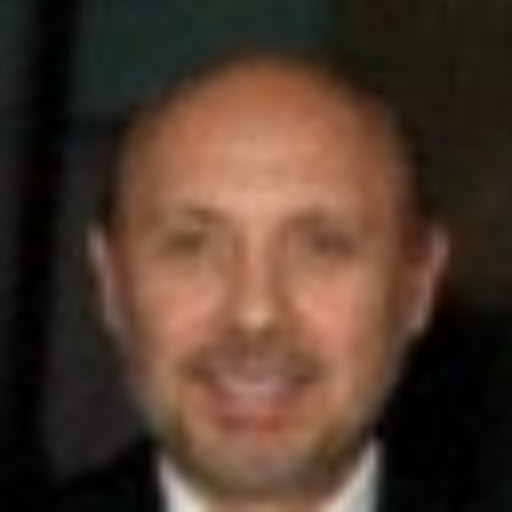}}
	    \\

        \centered{\includegraphics[width=0.1667\linewidth]{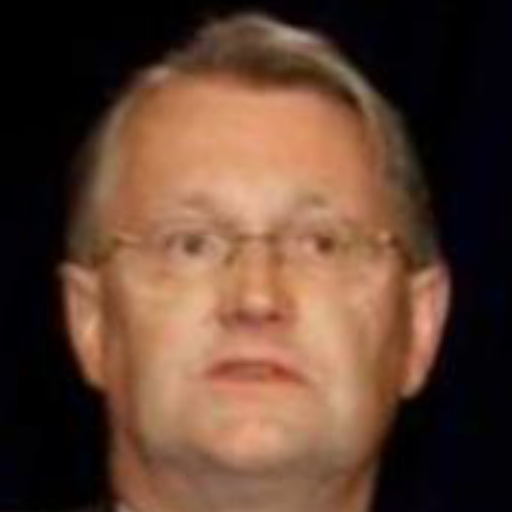}} &
        \centered{\includegraphics[width=0.1667\linewidth]{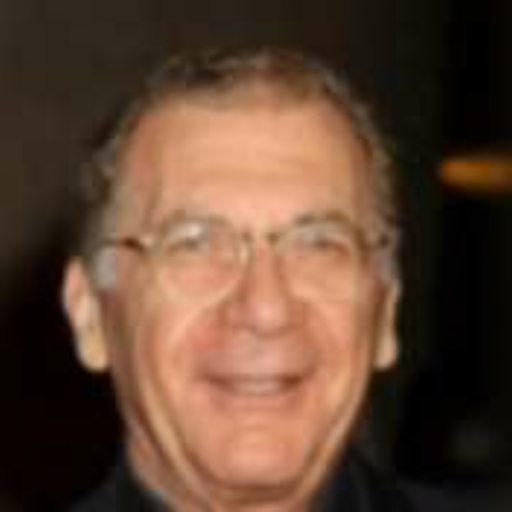}} &
        \centered{\includegraphics[width=0.1667\linewidth]{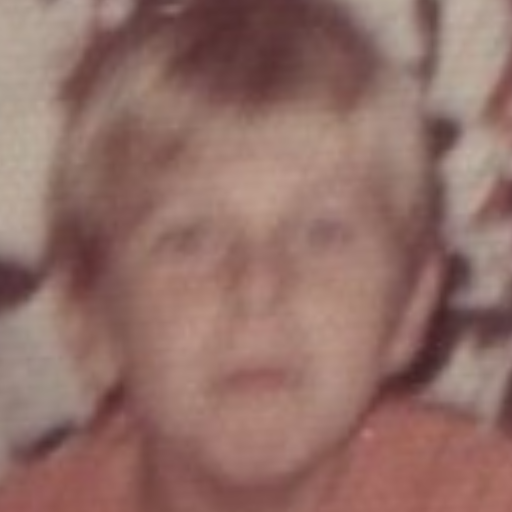}} &
        \centered{\includegraphics[width=0.1667\linewidth]{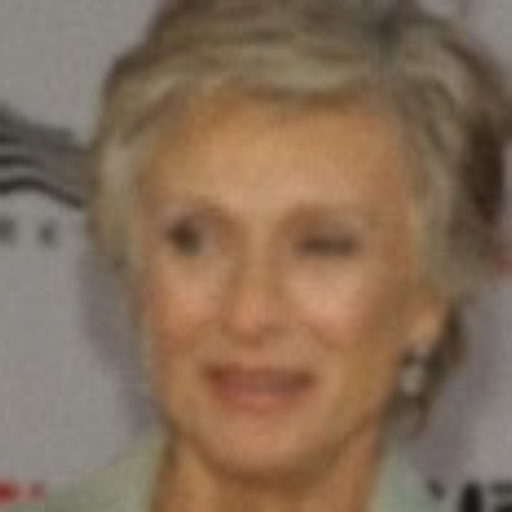}} &
        \centered{\includegraphics[width=0.1667\linewidth]{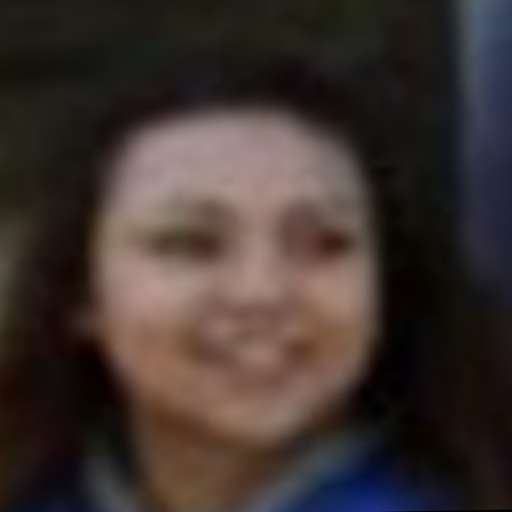}} &
        \centered{\includegraphics[width=0.1667\linewidth]{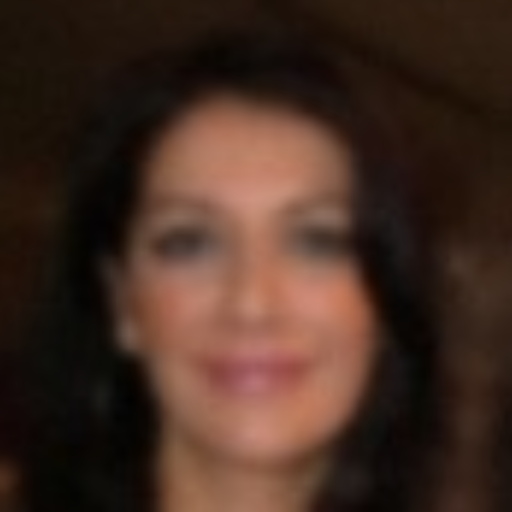}}
	    \\

        \centered{\includegraphics[width=0.1667\linewidth]{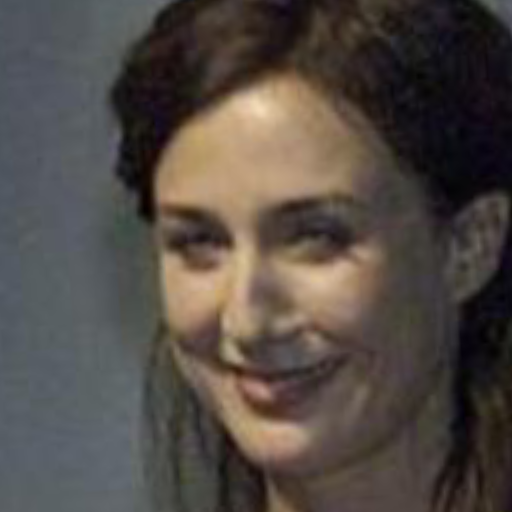}} &
        \centered{\includegraphics[width=0.1667\linewidth]{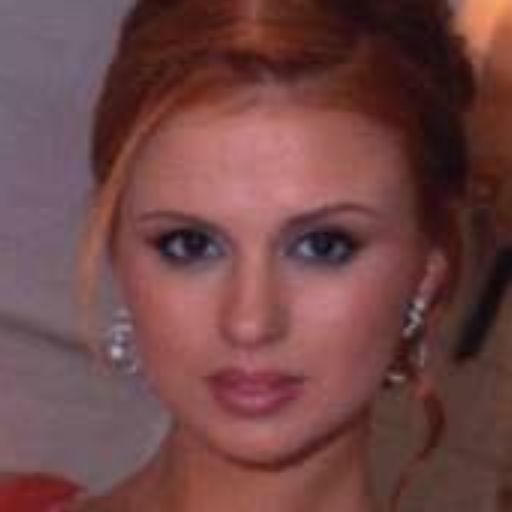}} &
        \centered{\includegraphics[width=0.1667\linewidth]{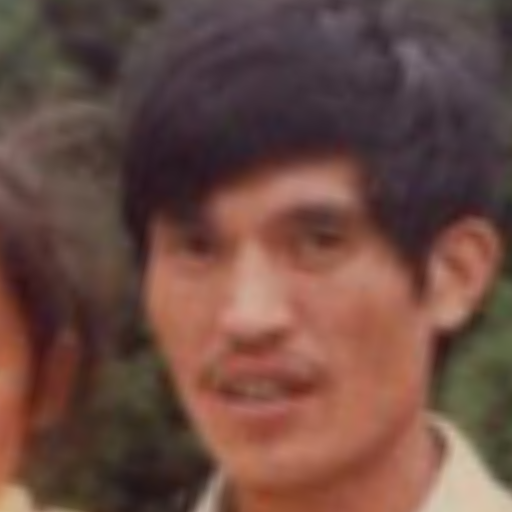}} &
        \centered{\includegraphics[width=0.1667\linewidth]{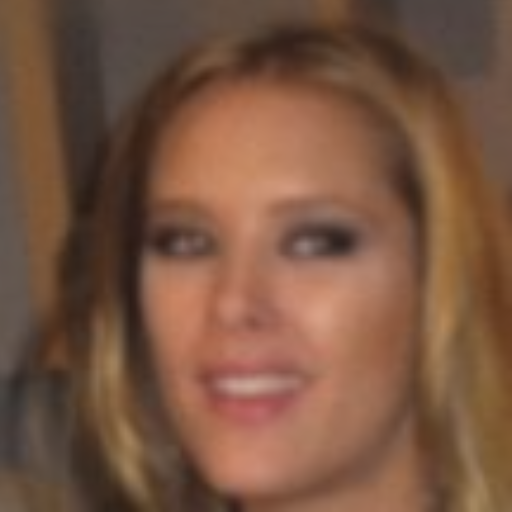}} &
        \centered{\includegraphics[width=0.1667\linewidth]{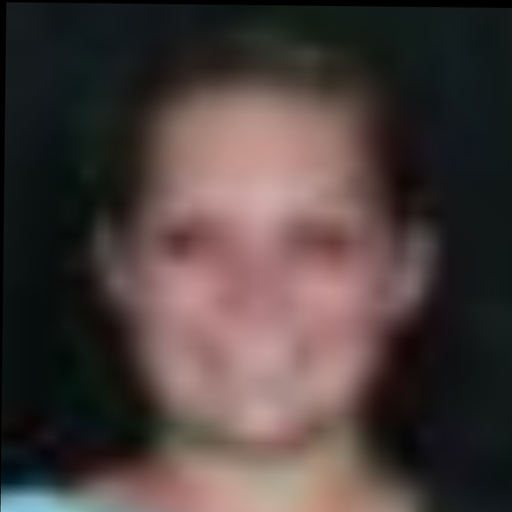}} &
        \centered{\includegraphics[width=0.1667\linewidth]{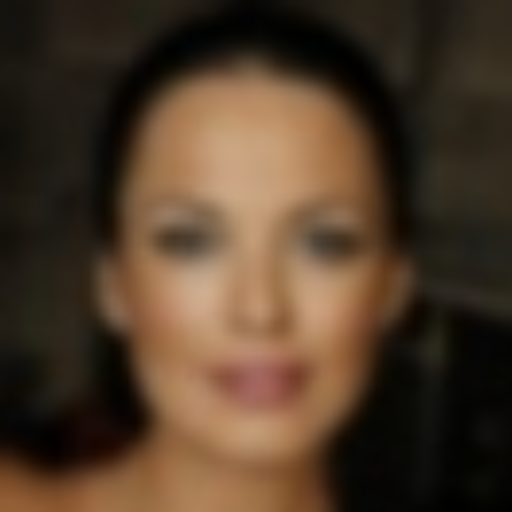}}
	    \\
    \end{tabular}
    \caption{
    Images from real datasets paired with images from CelebA synthetically degraded by the predicted degradation.
    }
    \label{fig:data_example}
\end{figure}
\endgroup

\subsection{Synthetic datasets examples}

\label{sec:datasets_supp}

Throughout the paper we use the real datasets LFW-Test, WebPhotos-Test, and WIDER-Test to the proposed tools.
In \cref{sec:deg_est} we additionally demonstrate how we can leverage the degradation estimator to analyze those datasets.

To demonstrate the effectiveness of our degradation estimator, we present in \Cref{fig:data_example} examples of degraded images from each dataset, and we predict the degradation of each. Using each predicted degradation we synthesize a synthetic degraded image from a clean CelebA-Test image.

As described in \cref{sec:deg_est}, we approximate the distribution of $\rva$ for each dataset by predicting the degradations of each image using our estimator and fitting a KDE over the predictions. In \Cref{sec:datasets_supp} we present examples from the synthetic datasets alongside images from the real datasets. Note that in this figures the image are not paired, as each degraded image is created by sampling randomly from the estimated distribution of $A$.

\begingroup
\newcolumntype{M}[1]{>{\centering\arraybackslash}m{#1}}
\newcommand{\vcentered}[1]{\begin{tabular}{@{}l@{}} #1 \end{tabular}}
\setlength{\tabcolsep}{0pt} %
\renewcommand{\arraystretch}{0} %

\newcommand{\centered}[1]{\begin{tabular}{l} #1 \end{tabular}}

\newcommand{\addimgcol}[8][1]{
	\centered{
	   	\begin{tikzpicture}[
	   		baseline=-2.45,
	   		spy using outlines={magnification=#3, circle, height=#8, width=#8, yellow, every spy on node/.append style={thick}, connect spies},
	   		]
			\node[inner sep=0pt]{\scalebox{#1}[1]{\adjincludegraphics[width=\hero, trim={#4}, clip]{images/ours/#7-bsds/compressed/#2}}};
			\spy on (#5) in node at (#6);
		\end{tikzpicture}}&
		
		\centered{
		\begin{tikzpicture}[
	   		baseline=-2.45,
	   		spy using outlines={magnification=#3, circle, height=#8, width=#8, yellow, every spy on node/.append style={thick}, connect spies},
	   		]
			\node[inner sep=0pt]{\scalebox{#1}[1]{\adjincludegraphics[width=\hero, trim={#4}, clip]{images/qgac/#7-bsds/#2}}};
			\spy on (#5) in node at (#6);
		\end{tikzpicture}}&
		
		\centered{
		\begin{tikzpicture}[
	   		baseline=-2.45,
	   		spy using outlines={magnification=#3, circle, height=#8, width=#8, yellow, every spy on node/.append style={thick}, connect spies},
	   		]
			\node[inner sep=0pt]{\scalebox{#1}[1]{\adjincludegraphics[width=\hero, trim={#4}, clip]{images/qgac-gan/#7-bsds/#2}}};
			\spy on (#5) in node at (#6);
		\end{tikzpicture}}&
		
		\centered{
	   	\begin{tikzpicture}[
	   		baseline=-2.45,
	   		spy using outlines={magnification=#3, circle, height=#8, width=#8, yellow, every spy on node/.append style={thick}, connect spies},
	   		]
			\node[inner sep=0pt]{\scalebox{#1}[1]{\adjincludegraphics[width=\hero, trim={#4}, clip]{images/bahat/#7-bsds/fake_0/#2}}};
			\spy on (#5) in node at (#6);
		\end{tikzpicture}}&

		\centered{
		\begin{tikzpicture}[
	   		baseline=-2.45,
	   		spy using outlines={magnification=#3, circle, height=#8, width=#8, yellow, every spy on node/.append style={thick}, connect spies},
	   		]
			\node[inner sep=0pt]{\scalebox{#1}[1]{\adjincludegraphics[width=\hero, trim={#4}, clip]{images/ours-p/#7-bsds/fake_0/#2}}};
			\spy on (#5) in node at (#6);
		\end{tikzpicture}}&
		
		\centered{
		\begin{tikzpicture}[
	   		baseline=-2.45,
	   		spy using outlines={magnification=#3, circle, height=#8, width=#8, yellow, every spy on node/.append style={thick}, connect spies},
	   		]
			\node[inner sep=0pt]{\scalebox{#1}[1]{\adjincludegraphics[width=\hero, trim={#4}, clip]{images/ours/#7-bsds/real/#2}}};
			\spy on (#5) in node at (#6);
		\end{tikzpicture}}
		
	    \\
}

\begin{figure*}[tb]
    \centering
    \begin{tabular}{M{0.083\linewidth} M{0.083\linewidth} M{0.083\linewidth} M{0.083\linewidth} M{0.083\linewidth} M{0.083\linewidth} M{0.083\linewidth} M{0.083\linewidth} M{0.083\linewidth} M{0.083\linewidth} M{0.083\linewidth} M{0.083\linewidth}}
    
	    \multicolumn{4}{c}{\footnotesize{LFW}} &
	    \multicolumn{4}{c}{\footnotesize{WebPhoto}} &
	    \multicolumn{4}{c}{\footnotesize{WIDER}}
	    \\

        \rule{0pt}{0.8ex}\\
        
        \cmidrule(lr){1-4}
        \cmidrule(lr){5-8}
        \cmidrule(lr){9-12}

        \centered{\footnotesize{Real}} &
        \footnotesize{Synthetic} &
        \footnotesize{Real} &
        \footnotesize{Synthetic} &
        \footnotesize{Real} &
        \footnotesize{Synthetic} &
        \footnotesize{Real} &
        \footnotesize{Synthetic} &
        \footnotesize{Real} &
        \footnotesize{Synthetic} &
        \footnotesize{Real} &
        \footnotesize{Synthetic}
        \\
	    
        \rule{0pt}{0.8ex}\\

        \multicolumn{12}{c}{\centered{\includegraphics[width=\linewidth]{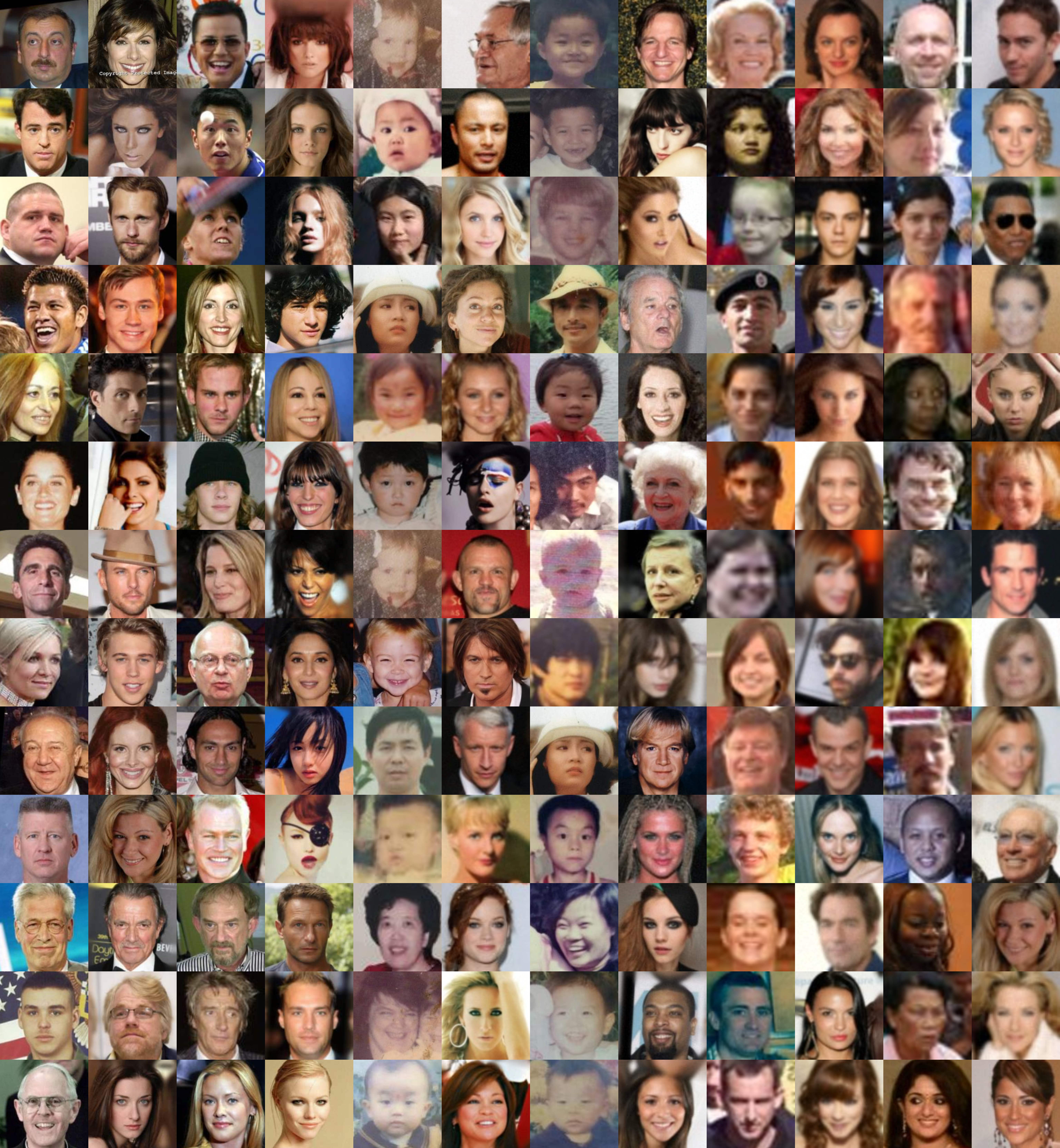}}}
        \\

    \end{tabular}
    \caption{More samples of real and synthetic images. The real images come from the LFW-Test, WebPhoto-Test, and WIDER-Test datasets. The synthetic images mimic those datasets as explained in \Cref{fig:data_example}.}
    \label{fig:data_example_appendix}
\end{figure*}
\endgroup

\end{document}